\documentclass{article}
\PassOptionsToPackage{numbers,compress}{natbib}
\usepackage{iclr2027_conference,times}
\usepackage{amsmath,amssymb,amsthm,bm}
\usepackage{mathtools}
\usepackage{float}
\usepackage{array}
\usepackage{booktabs}
\usepackage{graphicx}
\usepackage{wrapfig}
\usepackage{subcaption}
\usepackage{hyperref}
\usepackage{enumitem}
\usepackage{xcolor}
\usepackage{colortbl}
\definecolor{linkblue}{RGB}{26,77,166}
\definecolor{citegreen}{RGB}{20,110,60}
\definecolor{revblue}{RGB}{0,0,210}
\newif\ifshowrev\showrevfalse
\makeatletter\long\def\rev#1{\ifshowrev{\color{revblue}#1}\else#1\fi}\makeatother
\makeatletter\long\def\old#1{#1}\makeatother % accepted revisions from earlier rounds
\hypersetup{colorlinks=true,linkcolor=linkblue,citecolor=citegreen,urlcolor=linkblue}
\usepackage{pifont}
\newcommand{\cmark}{\textcolor{green!55!black}{\ding{51}}}
\newcommand{\xmark}{\textcolor{red!75!black}{\ding{55}}}

\newcommand{\rih}[1]{\par\smallskip\noindent\textbf{#1}\ }
\newcommand{\SMCAT}{\ensuremath{\mathrm{AMC}}}
\newcommand{\VGAS}{\ensuremath{\text{VGAS}}}
\newcommand{\VGASGR}{\ensuremath{\text{VGAS-GR}}}
\newcommand{\VGASRO}{\ensuremath{\text{VGAS-RO}}}

\newtheorem{proposition}{Proposition}
\newtheorem{lemma}{Lemma}
\newsavebox{\boxA}\newsavebox{\boxB}
\newtheorem{remark}{Remark}

\theoremstyle{definition}

\usepackage[ruled,vlined,linesnumbered]{algorithm2e}

\DeclareMathOperator{\softmax}{softmax}

\DeclareMathOperator{\lse}{logsumexp}
\DeclareMathOperator{\sg}{sg}
\DeclareMathOperator{\std}{std}
\DeclareMathOperator{\mean}{mean}
\DeclareMathOperator*{\argmax}{arg\,max}
\newcommand{\R}{\mathbb{R}}
\newcommand{\E}{\mathbb{E}}

\newcommand{\V}{\mathcal{V}}
\newcommand{\bx}{\bm{x}}

\newcommand{\etab}{\bm{\eta}}

\newcommand{\bu}{\bm{u}}

\title{VGAS: Variance-Reduced Guidance\\ and Adaptive Selection for Training-Free\\ Reward Alignment in Discrete Diffusion}
\author{Kwanyoung Kim\\
Department of AI, GIST\\
\texttt{k0.kim@gist.ac.kr}}
\iclrfinalcopy
\makeatletter\def\lhead#1{}\makeatother
\date{}

\begin{document}
\maketitle

\begin{figure}[h]
\centering
\includegraphics[width=\textwidth]{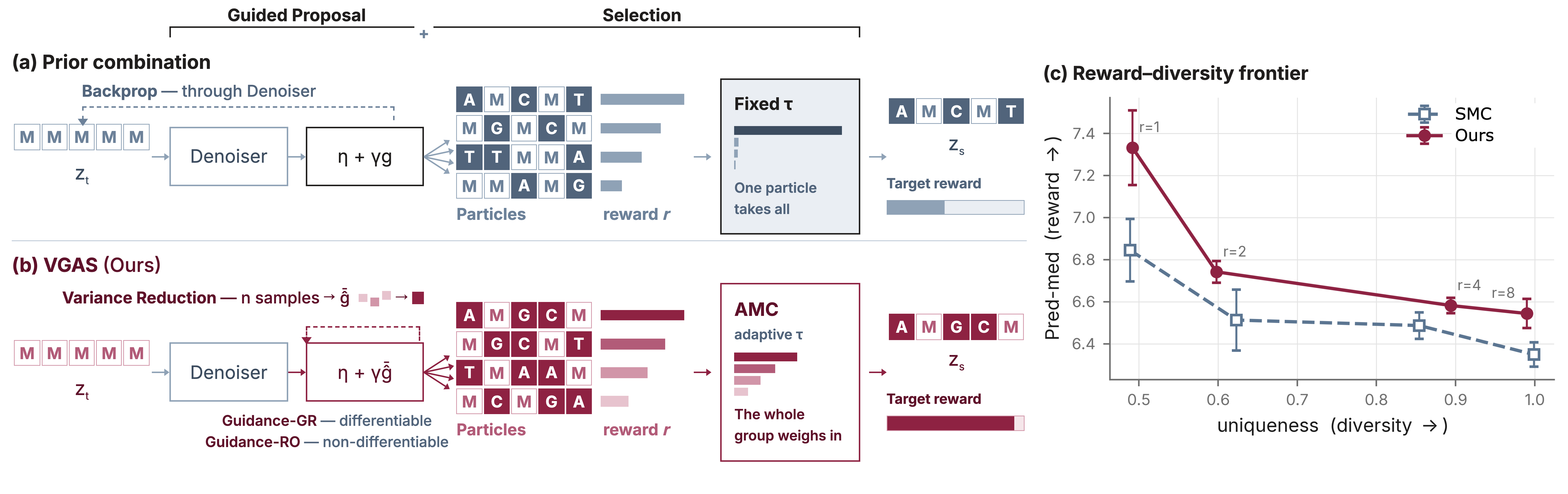}
\caption{\textbf{Overview.} (a) prior combinations pair a high-variance gradient estimate with a fixed
resampling temperature. (b) \VGAS{} reduces that variance and sets the temperature per step. (c) it attains
the higher reward at every resampling frequency.}
\label{fig:teaser}
\end{figure}

\begin{abstract}
Masked discrete diffusion models perform strongly on text, code, and biological sequences, but their
training objective rewards only naturalness, and retraining the generator for every new reward is expensive.
Inference-time steering of a frozen model either guides the sampler by the reward gradient or searches over
several trajectories, and recent samplers combine the two. Such combinations are assembled as pipelines
that leave three choices at their defaults: \old{a guidance estimate resting on one Gumbel draw per
sample}, a reward tilting
placed without reference to the distribution the combination then targets, and a selection temperature held
fixed although the spread of per-step rewards drifts. We identify that distribution and settle the three
choices against it.
We therefore propose \textbf{V}ariance-reduced \textbf{G}uidance and \textbf{A}daptive
\textbf{S}election (\VGAS), a simple yet effective inference-time framework that reduces the variance of
the guidance estimate for both reward types, applies the reward tilting in the clean-token logits, where the
pretrained schedule is preserved, and sets the selection temperature per step. Across regulatory
DNA, protein and small-molecule benchmarks, \old{\VGAS{} attains the best training-free reward and
matches or surpasses a reward-fine-tuned generator.}
\end{abstract}

\section{Introduction}
\label{sec:intro}

Masked discrete diffusion generates a sequence by iteratively unmasking it: the forward process replaces tokens
with a \textsc{mask} symbol and the model learns to reverse it~\citep{sahoo2024mdlm,shi2024md4}. Simple to train and strong on language and code, these models now also address biological and
chemical design problems such as regulatory DNA, proteins, and small
molecules~\citep{wang2025drakes,tfgflow}. Their training objective, however, rewards only naturalness.

In most applications, however, naturalness is not the objective. The goal is samples that also score highly
under a downstream reward, such as an enhancer that drives high expression, a protein that folds stably,
a molecule that meets a target property, or a text completion that a preference model rates well. One
route is post-training: fine-tuning or reinforcement learning of the generator against the
reward~\citep{wang2025drakes,han2026sdpo,rectorbrooks2025ddpp}, but this is expensive and must be
redone for every new reward. The
alternative, now standard for continuous diffusion, is inference-time steering: keep the pretrained model
frozen and bias only the sampling process toward high-reward regions. In this work, we focus on this inference-time,
training-free setting: given a frozen generator $p_\theta$ and a reward $r$, sample sequences that are
simultaneously likely under $p_\theta$ and high under $r$, with no retraining.

Specifically, these steering methods fall into three families (detailed in
\S\ref{sec:background}). \emph{Gradient guidance} biases sampling by the reward
gradient~\citep{chung2023dps,nisonoff2024,tfgflow,entrgi}, and the instantiation we build on is
GILC, which applies the reward tilting to the clean-token logits
directly~\citep{dou2026gilc}. \emph{Search} keeps several trajectories and resamples them by value, as in
SMC~\citep{wu2023smc} and SVDD~\citep{li2024svdd}, and a fast-growing line of Feynman--Kac and sequential
Monte Carlo resampling extends it~\citep{hasan2026dfkc,dang2025pgdlm,yadala2026nestedsmc}; it improves with
more particles but leaves the per-step proposal unchanged. A parallel line steers by gradient-free MCMC
instead~\citep{chu2025sgdd,phunyaphibarn2026csmc}. The third family therefore \emph{combines} them, pairing
a gradient-guided proposal with search, as in the twisted diffusion sampler
TDS~\citep{wu2023smc}, SMC-DDM~\citep{pani2025smcddm} and TreeG's gradient
variant~\citep{guo2025treeg}. Appendix~\ref{app:related} positions
this literature (Table~\ref{tab:positioning}) and anatomizes the closest of these samplers
(Table~\ref{tab:anatomy}).

These combinations, however, are assembled as pipelines, and each half is left at its standard
setting. The guidance gradient
remains a straight-through Monte Carlo estimate whose variance limits how much the guidance helps; the
reward tilting is placed without reference to the distribution the combination then targets; and the
search resamples at a temperature \emph{set in advance}, even though the spread of per-step rewards varies
several-fold along the denoising trajectory, so one temperature is too soft at one end and too sharp at the
other.

As shown in Figure~\ref{fig:teaser}, all three can be settled at no additional denoiser call.
We first establish what the combined sampler targets. For the estimator, we replace the
straight-through estimate with a Rao--Blackwellized one for differentiable rewards and a leave-one-out
baseline for non-differentiable ones. For the placement, we apply the reward tilting in the clean-token logits,
where the undetermined constant in the guidance direction cancels and the pretrained schedule is preserved.
For the temperature, we standardize the per-step values, which reduces to a single adaptive resampling
temperature. Selection also concentrates the
sample set, and \VGAS{} attains the higher reward at every resampling frequency
(Figure~\ref{fig:teaser}c). Our contributions are as follows.
\begin{itemize}[leftmargin=1.4em,itemsep=1pt,topsep=2pt]
\item We propose \VGAS, a training-free sampler that reduces the variance of
the guidance estimate for both differentiable and non-differentiable rewards, applies the reward tilting in
the clean-token logits, and sets the selection temperature per step.
\item We show that a guided proposal paired with selection, run without the proposal correction
\old{and at a fixed selection temperature}, still targets a reward tilt of the pretrained model, and we establish the effective temperature of that
tilt, which is what makes the correction optional.
\item Across regulatory DNA, protein design and small-molecule property targeting, \VGAS{} achieves
state-of-the-art inference-time reward alignment.
\end{itemize}

\section{Background}
\label{sec:background}

\rih{Masked Discrete Diffusion.}
Let $\bx=(x^1,\dots,x^L)$ be a length-$L$ sequence with each token $x^\ell$ in a vocabulary
$\V=\{1,\dots,V\}$, one of which is a special absorbing \textsc{mask} symbol; write
$V_0=V-1$ for the number of ordinary tokens. A masked (absorbing) diffusion model~\citep{sahoo2024mdlm,shi2024md4,wang2025drakes} defines a
forward process that gradually replaces tokens with \textsc{mask}, and learns to reverse
it. Generation therefore proceeds by iterative unmasking: start from the all-\textsc{mask}
sequence and, over $T$ steps, progressively reveal tokens until none are masked.

Concretely, a time variable $t$ runs from $t{=}1$ (all masked) down to $t{=}0$ (fully revealed); write
$\bm z_t$ for the partially masked state and $s=t-\Delta t$ for the next, less-noisy time. A denoiser
network reads $\bm z_t$ and outputs, for each masked position, a categorical distribution over
the clean token,
\begin{equation}
\label{eq:clean}
p_\theta(x_0^\ell = k\mid \bm z_t)=\softmax(\etab^\ell)_k,\qquad \etab^\ell=\etab^\ell(\bm z_t)\in\R^V,
\end{equation}
and these predictions are independent across positions,
$p_\theta(\bx_0\mid \bm z_t)=\prod_\ell p_\theta(x_0^\ell\mid \bm z_t)$. Let $\bar\alpha_t\in[0,1]$ be the
fraction of tokens still present under the schedule ($\bar\alpha_1{=}0,\bar\alpha_0{=}1$); we write it
with a bar to keep it distinct from the selection temperature $\alpha$ of \S\ref{sec:filter}. One reverse step either reveals a masked position or leaves it masked:
\begin{equation}
\label{eq:revkernel}
p_\theta(z_s^\ell\mid \bm z_t)=
\begin{cases}
\delta_{z_s^\ell = z_t^\ell}, & z_t^\ell\neq\textsc{mask}\ \ \text{(already revealed; frozen)},\\[2pt]
(1-a_t)\,\delta_{\textsc{mask}} + a_t\, p_\theta(x_0^\ell\mid \bm z_t), & z_t^\ell=\textsc{mask}.
\end{cases}
\end{equation}
Here $a_t=(\bar\alpha_s-\bar\alpha_t)/(1-\bar\alpha_t)$$\in[0,1]$ is the per-step reveal probability set by the
schedule.

\rih{Reward-Tilted Sampling and the Soft Value Function.}
Given a reward $r:\V^L\to\R$, reward alignment maximizes the expected reward while staying close to
the generator. For an inverse temperature $\beta^{-1}$ (larger $1/\beta$ = stronger guidance), the maximizer
of $\E_{p}[r]-\beta\,\mathrm{KL}(p\,\|\,p_\theta)$ is the reward-tilted
distribution~\citep{wang2025drakes,dou2026gilc}
\begin{equation}
\label{eq:target}
p^\star(\bx_0)\;\propto\; p_\theta(\bx_0)\,\exp\!\big(r(\bx_0)/\beta\big).
\end{equation}
Post-training methods approximate $p^\star$ by fitting a new generator; the inference-time methods
studied here leave $p_\theta$ frozen and sample from $p^\star$ during the reverse process instead. The Doob
$h$-transform of Eq.~\eqref{eq:revkernel} samples Eq.~\eqref{eq:target} exactly, reweighting each reverse
step by the soft value $v_s$~\citep{wu2023smc,dou2026gilc}:
\begin{equation}
\label{eq:doob}
p^\star(\bm z_s\mid\bm z_t)\;\propto\; p_\theta(\bm z_s\mid\bm z_t)\,\exp\!\big(v_s(\bm z_s)/\beta\big),
\qquad
v_s(\bm z_s)=\beta\log\E_{p_\theta(\bx_0\mid\bm z_s)}\!\big[e^{r(\bx_0)/\beta}\big],
\end{equation}
where $\bx_0$ denotes a fully revealed sequence, $p_\theta(\bx_0\mid\bm z_s)$ the clean posterior of
Eq.~\eqref{eq:clean}, and $v_s$ the soft value function: the (log-sum-exp) expected future reward if
denoising continues from $\bm z_s$~\citep{dou2026gilc}. In SMC terms, $v_s$ is the optimal twist, and
every method below approximates it. The difficulty is that $v_s$ depends on the reward
of clean sequences yet has to be evaluated at noisy states. Three families of inference-time methods
approximate it, and the first two can be composed.
\begin{itemize}[leftmargin=1.4em,itemsep=1pt,topsep=2pt]
\item \textbf{Gradient guidance} edits the reverse kernel itself by the reward gradient.
DG~\citep{nisonoff2024} and TFG-Flow~\citep{tfgflow} tilt the transition rates of the discrete
process; for masked diffusion the tilt can instead be applied to the clean-token
logits~\citep{entrgi,dou2026gilc}, and we build on one such logit-space instantiation, GILC.
\item \textbf{Search-based selection} keeps the base reverse kernel and samples several trajectories,
reweighting or selecting by an inexpensive value estimate $\hat V(\bm z)$: SMC resamples a particle population
with weights \old{$w\propto\exp(\Delta\hat V/\alpha)$ set by the per-step change in the estimate}, while
SVDD keeps the best candidate. Both improve with more
particles but leave the per-step proposal unchanged.
\item \textbf{Combining the two} acts on different parts of Eq.~\eqref{eq:doob}: in SMC terms, a
twisted proposal paired with a selection step. This is the regime of TDS, SMC-DDM and TreeG; each half,
however, keeps its default form, a high-variance estimator and a fixed temperature.
\end{itemize}

\rih{Revisiting GILC.}
\label{sub:gs}
Pushing the reveal toward higher reward requires $\nabla_{\etab} r$, but $r$ acts on a discrete sequence and is
piecewise constant in $\etab$. The standard fix is the Gumbel--Softmax straight-through (ST-GS)
estimator~\citep{jang2017gumbel,bengio2013ste}: draw
$\hat\bx_{\mathrm{soft}}=\softmax((\etab+\bm\zeta)/\tau)$ at a relaxation temperature $\tau>0$, feed the hard one-hot to the reward, and
backpropagate through the soft relaxation,
\begin{equation}
\label{eq:st}
\hat\bx=\mathrm{onehot}\big(\argmax_k \hat\bx_{\mathrm{soft},k}\big)-\sg(\hat\bx_{\mathrm{soft}})+\hat\bx_{\mathrm{soft}},
\qquad
\partial\hat\bx/\partial\etab=\partial\hat\bx_{\mathrm{soft}}/\partial\etab .
\end{equation}
GILC~\citep{dou2026gilc} instantiates gradient guidance through this path: it draws the revealed token from a corrected clean
prediction, adding $\bm\delta^\ell=\gamma\bm g^\ell$ to each masked position's logits,
\begin{equation}
\label{eq:direction}
\hat p^{\ell}=\softmax\!\big(\etab^\ell+\bm\delta^\ell\big),
\qquad
p^\star_{\text{guided}}(z_s^\ell\mid \bm z_t)=(1-a_t)\,\delta_{\textsc{mask}}+a_t\,\hat p^{\ell}.
\end{equation}
The logit correction $\bm g^\ell=\big(\partial r(\hat\bx_0)/\partial\hat\bx_0\big)\big(\partial\hat\bx_0/\partial\etab^\ell\big)$
estimates $\nabla_{\etab^\ell}\E[r]$, obtained by approximating the expensive model Jacobian by the
identity, $\partial\etab/\partial\bm z_t\approx\bm I$ (hence ``Jacobian-free''). This form applies when $r$ is
differentiable; when it is not, the same $\bm g^\ell$ is formed by a score-function (policy-gradient)
estimator instead, and \S\ref{sec:proposal} treats the two cases symmetrically. \old{The identity approximation
is inherited from GILC, which states it as bypassing the Jacobian, and it is what removes the denoiser from
the backward pass (Appendix~\ref{app:guidance-math}).} The guidance scale $\gamma>0$
plays the role of the inverse temperature $1/\beta$ in Eq.~\eqref{eq:target}. \old{Both estimators
evaluate their Jacobian at one Gumbel draw per sample}, so $\bm g^\ell$ carries the variance of that draw, and that variance is what caps
the gain from guidance.

\section{Main Contribution: \VGAS}
\label{sec:method}
\old{In this section, we present \VGAS, which acts on the three choices such combinations leave at their
defaults.} \S\ref{sec:proposal} reduces the variance of the guided proposal's logit correction,
\old{\S\ref{sec:filter} establishes what selection over that proposal targets and why the correction
belongs in the clean-token logits}, and \S\ref{sec:search}
makes the resampling temperature adaptive, \old{none of them} at an extra denoiser call.

\subsection{Variance-Reduced Guided Proposals}
\label{sec:proposal}

We build on GILC, whose logit correction $\bm g$ of Eq.~\eqref{eq:direction} is a Monte Carlo
estimate of $\nabla_{\etab}\E[r]$ \old{formed from a single draw. We therefore introduce} one variance
reduction per reward type, each reusing the samples the correction already draws (base estimators in
Appendix~\ref{app:pg}).

\begin{figure}[t]
\begin{minipage}[t]{0.485\textwidth}
\begin{algorithm}[H]
\small
\caption{\textsc{Guidance-GR}}
\label{alg:corr-rao}
\KwIn{logits $\etab$; \old{state $\bm z_t$}; reward $r$; size $n$; Rao $M$}
\For{$i=1,\dots,n$}{
  $\bm\zeta\!\sim\!\mathrm{Gumbel}$,\ $i^\star\!=\!\argmax(\etab{+}\bm\zeta)$\;
  $\bar\bx\!=\!\tfrac1M\sum_{m\le M}\softmax_\tau(\etab{+}\bm\zeta^{(m)}\!\mid\!i^\star)$\;
  $\hat\bx\!\gets\!\mathrm{ST}(i^\star,\bar\bx)$ (Eq.~\ref{eq:st}),\ $R_i\!=\!r(\hat\bx)$\;
}
$\bm g_{\mathrm{GR}} = \tfrac1n\sum_i \partial R_i/\partial\etab$\ \ (Eq.~\ref{eq:gr})\;
\KwOut{correction $\bm g_{\mathrm{GR}}$}
\end{algorithm}
\end{minipage}\hfill
\begin{minipage}[t]{0.485\textwidth}
\begin{algorithm}[H]
\small
\caption{\textsc{Guidance-RO}}
\label{alg:corr-pgro}
\KwIn{logits $\etab$; \old{state $\bm z_t$}; reward $r$; size $n$}
$\bm p = \softmax(\etab)$\;
\For{$i=1,\dots,n$}{
  $\bx\!\sim\!\bm p$,\ $\hat\bx=\mathrm{onehot}(\bx)$,\ $R_i=r(\hat\bx)$\;
}
$A_i = R_i-\tfrac{1}{n-1}\sum_{j\neq i}R_j$\ \ (Eq.~\ref{eq:rloo})\;
$\bm g_{\mathrm{RO}} = \tfrac1n\sum_i A_i(\hat\bx_i-\bm p)$\;
\KwOut{correction $\bm g_{\mathrm{RO}}$}
\end{algorithm}
\end{minipage}
\end{figure}

\noindent\textbf{Differentiable reward: the Gumbel--Rao proposal.}
The pathwise base estimator \textsc{Guidance-DB} (Appendix~\ref{app:pg}) averages straight-through
gradients Eq.~\eqref{eq:st} over $n$ Gumbel--Softmax samples; \old{it is a biased relaxation, evaluated at
the single Gumbel draw that produced each sample}. The reveal already computes the hard sample $i^\star=\argmax_k(\eta_k+\zeta_k)$, a sufficient statistic for the
discrete outcome, so \old{conditioning on it} preserves the mean and cannot increase the
variance~\citep{rao1945,blackwell1947}. We therefore replace the straight-through Jacobian by its conditional
expectation given $i^\star$, the Gumbel--Rao construction of~\citet{paulus2020raoblackwell}
(\textsc{Guidance-GR}, Algorithm~\ref{alg:corr-rao}). That conditional has a closed-form
truncated-Gumbel reparameterization~\citep{maddison2014astar,kool2019stochastic} that reuses $\etab$ at no
forward pass. Writing $v_k=\eta_k+\zeta_k$ for the perturbed logits and $T=v_{i^\star}$
for their maximum, and drawing $U_0,U_1,\dots$ independently from $\mathrm{U}(0,1)$,
\begin{align}
\label{eq:condgumbel}
T &= \log\textstyle\sum_j e^{\eta_j}-\log(-\log U_0),\notag\\
v_k &= -\log\!\big(e^{-T}-e^{-\eta_k}\log U_k\big)\ \ (k\neq i^\star),
\qquad \zeta_k=v_k-\eta_k,
\end{align}
where $-\log U_k\sim\mathrm{Exp}(1)>0$ keeps $v_k\le T$, so the induced argmax is exactly $i^\star$.
Repeating this $M$ times gives $\bm\zeta^{(1)},\dots,\bm\zeta^{(M)}$, and averaging over them gives the
correction we use,
\begin{equation}
\label{eq:gr}
\bm g_{\mathrm{GR}}=\frac{\partial r(\hat\bx_0)}{\partial\hat\bx_0}\,
\E_{\bm\zeta\mid\,\argmax(\etab+\bm\zeta)=i^\star}\!\Big[\tfrac{\partial \softmax_\tau(\etab+\bm\zeta)}{\partial\etab}\Big]
\;\approx\; \frac{\partial r(\hat\bx_0)}{\partial\hat\bx_0}\,\frac1M\sum_{m=1}^M
\frac{\partial\softmax_\tau(\etab+\bm\zeta^{(m)})}{\partial\etab},
\end{equation}
\old{which removes the variance the single Gumbel draw contributes. \textsc{Guidance-DB} is the ``base''
row of Table~\ref{tab:grm}, where the sweep over $M$ is reported.}

\smallskip\noindent\textbf{Non-differentiable reward: the policy-gradient proposal.}
A non-differentiable reward supplies no gradient to relax, so $\bm g$ is obtained from the
score-function identity instead, which needs only reward values. Draw $n$ hard samples
$\bx^{(i)}\sim\bm p=\softmax(\etab)$, write $\hat\bx^{(i)}=\mathrm{onehot}(\bx^{(i)})$ and
$R_i=r(\hat\bx^{(i)})$, and weight the score
$\nabla_{\etab}\log p(\bx^{(i)})=\hat\bx^{(i)}-\bm p$ by the reward. \old{GILC forms that weight as a
group-relative advantage, standardizing $R_i$ by the mean and the standard deviation of the group
(\textsc{Guidance-PG}, Appendix~\ref{app:pg}).} The
price is a variance higher than on the differentiable path. We propose to use a Reinforce leave-one-out
(RLOO) baseline~\citep{williams1992,kool2019rloo,ahmadian2024rloo} \old{in place of that standardization},
which \old{weights by} the leave-one-out advantage $A_i$, \old{the deviation of $R_i$} from the mean of the
other $n-1$ rewards:
\begin{equation}
\label{eq:rloo}
A_i = R_i-\tfrac{1}{n-1}\textstyle\sum_{j\neq i} R_j,
\qquad
\bm g_{\mathrm{RO}} = \tfrac1n\sum_{i=1}^n A_i\,(\hat\bx^{(i)}-\bm p),
\end{equation}
which reuses the same $n$ reward evaluations (\textsc{Guidance-RO}, Algorithm~\ref{alg:corr-pgro}).
The two reductions play the same role, one per estimator, \old{through different arguments, and the
asymmetry is structural: the reveal already computes the sufficient statistic $i^\star$, so conditioning on
it is free on the pathwise path, while the score-function estimator admits no such statistic and is reduced
by a baseline instead}. Gumbel--Rao \old{therefore} carries the variance inequality above\old{; the baseline
is unbiased for every constant and, for a single sample, reduces variance over an interval that
Lemma~\ref{lem:rloo} identifies under a sign condition, for which the leave-one-out mean is the natural
plug-in} (Appendix~\ref{app:proofs}). \old{Table~\ref{tab:rloo} measures its effect on the proposal.}

\noindent\textbf{Preserving the unmasking schedule.} \old{Placing $\bm g$ inside the softmax, rather than
on the kernel after the reveal probability is fixed, leaves the reveal probability at $a_t$, whereas the
post-hoc tilt rescales it to $a_tZ^\ell/(1-a_t+a_tZ^\ell)$ with
$Z^\ell=\E_{p_\theta^\ell}[e^{\gamma g}]$.} The
additive constant in $\bm g$ is not fixed by its definition: the softmax is unchanged by
$\etab\mapsto\etab+c\bm1$, so a correction expressed in logit coordinates is determined only up to such a
constant, and different derivations fix it differently. The logit-space reveal is invariant to it,
whereas \old{the post-hoc one} tends to $1$ or to $0$ as the constant grows, so an undetermined
choice can drive the post-hoc schedule to a full reveal or to none (Proposition~\ref{prop:leak},
proved in Appendix~\ref{app:proofs}). \S\ref{sec:filter} shows that this invariance acts on the target itself
and not only on the quality of the proposal.

\subsection{Selection over a Guided Proposal}
\label{sec:filter}

\noindent\textbf{Selector.} Among SMC, SVDD and Best-of-$N$ we use SMC, the only one that acts throughout
the trajectory without extra denoiser calls: SVDD evaluates $M$ candidates at every step and Best-of-$N$
defers all selection to the end. A selector scores a partial state by the one-step value
$V(\bm z)=r\big(\argmax_{\bx_0} p_\theta(\bx_0\mid\bm z)\big)$, and the previous step's value is cached, so
each step costs one denoiser call.
\textsc{SMC}~\citep{delmoral2004feynman} carries $N$ particles, scores particle $k$ by the increment
$\Delta^k=V(\bm z_s^k)-V(\bm z_t^k)$, and resamples multinomially with $w^k\propto\exp(\Delta^k/\alpha)$,
where $\alpha$ is the selection temperature, \old{distinct from} the KL weight $\beta$ of
Eq.~\eqref{eq:target}.

\noindent\textbf{Target of the combination.} \old{The reward enters through the proposal of
\S\ref{sec:proposal} and again through the weights, so neither fixes the target alone. Three measures are
in play: the base model $p_\theta$,
the law of the surviving particles \old{that Lemma~\ref{lem:target} identifies}, and the twisted proposal $q_\gamma$, the path measure of the guided
reveal of \S\ref{sec:proposal} run to the end, with terminal marginal $q_\gamma(\bx_0)$. The log-ratio of
that marginal to the base model,}
\begin{equation}
\label{eq:phi}
\Phi_\gamma(\bx_0)\;=\;\log\frac{q_\gamma(\bx_0)}{p_\theta(\bx_0)}
\end{equation}
\old{is} the accumulated log-twist, \old{whose trajectory counterpart is available in closed form}
(Appendix~\ref{app:correction}). The combination then
targets
\begin{equation}
\label{eq:target-gamma}
\pi_\gamma(\bx_0)\;\propto\;q_\gamma(\bx_0)\,e^{r(\bx_0)/\alpha}
\;=\;p_\theta(\bx_0)\,e^{\Phi_\gamma(\bx_0)+r(\bx_0)/\alpha},
\end{equation}
\old{a reward tilt of the twisted proposal at every $\gamma$}. Equivalently
$\pi_\gamma=\argmax_{p}\,\E_{p}[r]-\alpha\,\mathrm{KL}(p\,\|\,q_\gamma)$, which is Eq.~\eqref{eq:target} with the twisted base measure. \old{Eq.~\eqref{eq:target-gamma} is exact because successive
increments cancel along a trajectory, leaving only the terminal reward in the weight.}

\noindent \old{The alternative is to keep the proposal correction.} Adding
$\log(p_\theta/q_\gamma)$ to the log-weights returns the target to $p^\star$ for every $\gamma$, and the
established particle methods for reward alignment retain
it~\citep{wu2023smc,pani2025smcddm,yadala2026nestedsmc}. \old{We omit it instead, and doing so does not make
the sampler approximate: both weightings are Feynman--Kac changes of
measure~\citep{delmoral2004feynman,chopin2020introduction}, stated in general form by \citet{wu2023smc}, and
Lemma~\ref{lem:target} identifies the distribution each one targets exactly at finite $\gamma$, ours
$q_\gamma$ tilted by the reward and theirs $p^\star$; Proposition~\ref{prop:efftemp} then returns ours to
the reward-tilted family of Eq.~\eqref{eq:target}.}

\begin{lemma}[Targets with and without the proposal correction]
\label{lem:target}
\old{Under boundedness of $r$ and $\gamma\bm g$, unbiased multinomial resampling, the deterministic
all-\textsc{mask} start and $V(\bm z_0)=r(\bx_0)$, running} the selector over $q_\gamma$ at fixed $\alpha$, (i) \old{$\tfrac1N\sum_{k}\varphi(\bx_0^k)\to\E_{\pi_\gamma}[\varphi]$
almost surely as $N\to\infty$ for every bounded measurable $\varphi$};
and (ii) adding $\log(p_\theta/q_\gamma)$ to the log-weights replaces $\pi_\gamma$ by $p^\star$ of
Eq.~\eqref{eq:target} with $\beta=\alpha$, for every $\gamma$.
\end{lemma}

\begin{proposition}[\old{Effective temperature of the guided proposal}]
\label{prop:efftemp}
\old{$\Phi_\gamma(\bx_0)=\kappa\gamma\,r(\bx_0)+c+\varepsilon(\bx_0)$, where $c$ is independent of
$\bx_0$, $\kappa>0$ is a tangential scale for the estimator that carries the reward gradient into the
clean-token logits, and
$|\varepsilon|\le E(\gamma)=O(\gamma)$. Omitting the correction therefore leaves the target in the
reward-tilted family of Eq.~\eqref{eq:target} at inverse temperature $\kappa\gamma+1/\alpha$, up to the
factor $e^{\varepsilon}$; at $\gamma=0$ the bound vanishes and $\beta=\alpha$ exactly.}
\end{proposition}

\noindent \old{Appendix~\ref{app:proofs} gives the remaining definitions, including $\kappa$ and
$E(\gamma)$, the remaining hypotheses of both results, and their proofs. \old{There $E(\gamma)$ is controlled by four quantities, one of
which is the estimator deviation that the reductions of \S\ref{sec:proposal} act on, and it grows with the
sequence length.}}

\noindent \old{That the effective temperature differs from the nominal guidance scale is known for
classifier-free guidance in the continuous setting~\citep{bradley2025cfgpc}, and guidance is known to
reshape the sampled distribution of class-conditional masked diffusion in ways the nominal scale does not
predict~\citep{he2026exactly}; Proposition~\ref{prop:efftemp} is the counterpart for a reward-gradient
tilt.}

\noindent \old{To reach the inverse temperature $\kappa\gamma+1/\alpha$ through the weights alone,
$\alpha$ must be smaller}, which concentrates them on a few particles, whereas twisting the proposal
relocates particles into the reward-relevant region first (\S\ref{sec:analysis}). \old{The omission also
fixes where the tilt is applied: with the correction retained, $q_\gamma$ is only a proposal and the weights
absorb its distortions; without it, $q_\gamma$ enters the target, so the undetermined constant in $\bm g$
reaches $\pi_\gamma$ through the reveal schedule (Proposition~\ref{prop:leak}(c)), which logit-space
guidance does not disturb.}

\subsection{Adaptive-Temperature Selection}
\label{sec:search}
\old{The last of the three defaults is the selection temperature, which} in \textsc{SMC}, TDS and SMC-DDM
alike \old{is} prescribed before sampling rather than measured from the particles. \old{Recorded while
those samplers run, the per-step spread $\sigma_t$ of the increments $\{\Delta_t^k\}$ falls by $3.7\times$
to $27.7\times$ along denoising on the three benchmarks, and with it the temperature matched to it
(Figure~\ref{fig:mechanism}a), so a temperature prescribed once exceeds the matched value over most of the
trajectory} \old{(Appendix~\ref{app:adaptive-smc} separates this from classical adaptive tempering)}.
\begin{wrapfigure}{r}{0.37\textwidth}
\vspace{0.2\baselineskip}
\centering
\includegraphics[width=0.355\textwidth]{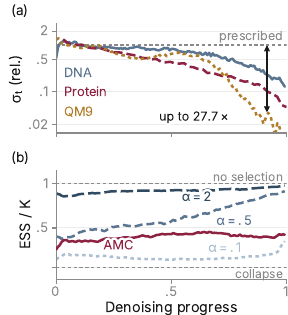}
\caption{\old{(a) Spread against the prescribed level. (b) ESS on DNA.}}
\label{fig:mechanism}
\vspace{-1.9\baselineskip}
\end{wrapfigure}
We therefore standardize the increment across the particle group, which applies a shift and a scale, and
\old{Lemma~\ref{lem:baseline} shows that only one of the two survives the resampling softmax}.

\begin{lemma}[Baseline invariance of the resampling law]
\label{lem:baseline}
Fix $c>0$. For increments $\{\Delta_t^k\}_{k=1}^{N}\subset\R$ and any constant $b$, the resampling law
$\hat w_t^k\propto\exp\!\big((\Delta_t^k-b)/c\big)$ equals $\mathrm{softmax}_k(\Delta_t^k/c)$ and is
independent of $b$.
\end{lemma}

\noindent With $\mu_t,\sigma_t$ the per-step mean and standard deviation of $\{\Delta_t^k\}$, the shift
therefore cancels and only the scale acts,
\begin{equation}
\label{eq:std}
\hat w_t^k\;\propto\;\exp\!\Big(\frac{\Delta_t^k-\mu_t}{\alpha\sigma_t}\Big)\;\propto\;\exp\!\Big(\frac{\Delta_t^k}{\alpha_t}\Big),
\qquad \alpha_t=\alpha\,\sigma_t,
\end{equation}
so the fixed temperature becomes a per-step one that tracks the spread, \old{so standardization's only
operative effect is $\alpha\mapsto\alpha\sigma_t$, whichever estimate of $\mu_t$ is used} ($\sigma_t$ is
floored at $\epsilon{=}10^{-6}$). \old{The proposal's contribution $\kappa\gamma$ is unchanged; only the selection term $1/\alpha_t$ now
varies along the trajectory.} \old{Measured as the per-step effective sample size (ESS), a fixed $\alpha$ either collapses the particle
population from the first steps or, at the value that avoids that, leaves the weights nearly uniform by the
last ones, so selection is effectively off where the value estimate is most reliable, whereas $\SMCAT$ holds
it in between throughout (Figure~\ref{fig:mechanism}b, Appendix~\ref{app:h-ess}).} We call the selector $\SMCAT$ (Adaptive Monte Carlo), and it changes one line
of SMC (Appendix~\ref{app:selectors}). \old{Proposition~\ref{prop:at} identifies the target that a per-step
temperature induces.}

\begin{proposition}[Target of $\SMCAT$]
\label{prop:at}
\old{The accumulated weight of a particle is proportional to $e^{R_\sigma}$ with
$R_\sigma=\sum_t\Delta_t/\alpha_t$, a functional of the whole trajectory, and the self-normalized estimator
converges in probability under the path measure $\pi^\star_{\SMCAT}\propto q_\gamma e^{R_{\sigma^\star}}$,
whose $\bx_0$-marginal is the law of the samples. If $\sigma^\star_t\equiv\sigma$, that marginal is
$\pi_\gamma$ with $\alpha$ replaced by $\alpha\sigma$.}
\end{proposition}

\noindent \old{Adapting the temperature changes what the selector targets. With $\alpha_t$ varying the
increments no longer cancel, so $R_\sigma$ retains the whole trajectory rather than its endpoint, and the
reweighting is by the greedy value $V$ of \S\ref{sec:filter} rather than the soft value $v_s$ of
\S\ref{sec:background}. \old{In return the selector gains a pressure that tracks the spread at every step,}
which no single $\alpha$ does (\S\ref{sec:analysis}). When the spread does not drift, $\sigma^\star_t$
constant, the selector reduces to the fixed-temperature one at $\alpha\sigma$, so the adaptation is active
exactly where the drift is. Appendix~\ref{app:proofs} carries the hypotheses and the proof, and
Appendix~\ref{app:anatomy} places $\SMCAT$ against the selectors of TDS and SMC-DDM.}

\noindent \old{Standardizing has two further consequences.} Dividing by $\sigma_t$ makes
the selector invariant to reward rescaling $r\mapsto ar$, $a>0$, which scales $\Delta_t^k$ and $\sigma_t$
alike, so $\alpha$ needs no retuning (Table~\ref{tab:scale}). \old{The invariance is a property of
$\alpha$ alone; $\gamma$ multiplies a gradient that scales with the reward.} \old{A} leave-one-out baseline in place of
$\mu_t$ is inert to $O(1/(N{-}1))$ (Remark~\ref{rem:loo}), which Table~\ref{tab:rloo} tests.

\noindent\textbf{Our final method: \VGAS.}
\VGAS{} combines the two: a variance-reduced guided proposal searched under the adaptive selector
$\SMCAT$, as $\VGASGR$ for a differentiable reward and $\VGASRO$ (RLOO baseline) for a non-differentiable one.

\section{Experiments}
\label{sec:experiments}

We evaluate \VGAS{} on three domains: regulatory DNA sequence design (\S\ref{sec:dna}), protein
sequence design (\S\ref{sec:protein}), and small-molecule property targeting (\S\ref{sec:qm9}).

\noindent\textbf{Baselines.} To rigorously evaluate our proposed method, we compare against the
strongest training-free and fine-tuned steering methods for discrete diffusion, reproduced from their
official repositories under a common setup. They fall into four groups. Fine-tuning:
DRAKES, the standard fine-tuned reference on these benchmarks. Gradient guidance: DG and GILC.
Search: Best-of-$N$, SMC, and SVDD. Combination (guidance $+$ search): TDS, SMC-DDM, and TreeG,
of which the last two are closest to ours. Implementation details are in Appendix~\ref{app:impl}.

\noindent\textbf{Synthetic verification.} Before the three benchmarks, we check the sampler where the
answer is known. On the $64{\times}64$ grid of \citet{pani2025smcddm} the reward-tilted target is available
in closed form, and the components reduce the earth-mover distance to it one at a
time~(Appendix~\ref{app:more-analysis}).

\begin{table}[b]
\centering
\small
\renewcommand{\arraystretch}{0.88}
\setlength{\tabcolsep}{5pt}
\caption{\textbf{Quantitative comparison on DNA enhancer design}~\citep{wang2025drakes} (mean${\pm}$std,
$3$ seeds, \old{$N{=}20$}). Bold indicates the best performance.}
\label{tab:dna}
\resizebox{0.88\textwidth}{!}{%
\begin{tabular}{@{}lrrrrr@{}}
\toprule
Method & Pred-med\,$\uparrow$ & ATAC-Acc (\%)\,$\uparrow$ & 3-mer Corr\,$\uparrow$ & JASPAR\,$\uparrow$ & App-LL\,$\uparrow$\\
\cmidrule(r){1-1}\cmidrule(l){2-6}
Pretrained (MDLM)~\citep{sahoo2024mdlm} & $0.18_{\pm0.02}$ & $2.0_{\pm0.4}$ & $-0.04_{\pm0.01}$ & $0.29_{\pm0.01}$ & $-261.9_{\pm0.6}$\\
DG~\citep{nisonoff2024} & $1.10_{\pm0.01}$ & $0.0_{\pm0.0}$ & $-0.09_{\pm0.01}$ & $-0.02_{\pm0.01}$ & $-267.8_{\pm0.2}$\\
Best-of-$N$~\citep{beirami2025bon} & $1.81_{\pm0.06}$ & $6.8_{\pm1.9}$ & $0.35_{\pm0.04}$ & $0.76_{\pm0.01}$ & $-260.2_{\pm0.1}$\\
SMC~\citep{delmoral2004feynman} & $4.26_{\pm0.13}$ & $30.9_{\pm5.9}$ & $0.83_{\pm0.03}$ & $0.76_{\pm0.02}$ & $-257.8_{\pm2.6}$\\
TDS~\citep{wu2023smc} & $4.48_{\pm0.80}$ & $36.6_{\pm9.1}$ & $0.76_{\pm0.17}$ & $0.81_{\pm0.14}$ & $\mathbf{-257.4}_{\pm4.5}$\\
SVDD~\citep{li2024svdd} & $5.30_{\pm0.02}$ & $29.3_{\pm0.2}$ & $0.73_{\pm0.01}$ & $0.83_{\pm0.05}$ & $-259.1_{\pm0.3}$\\
DRAKES~\citep{wang2025drakes} & $5.61_{\pm0.08}$ & $92.4_{\pm0.7}$ & $0.89_{\pm0.01}$ & $\mathbf{0.91}_{\pm0.01}$ & $-264.2_{\pm0.3}$\\
\midrule
GILC-DB~\citep{dou2026gilc} & $6.21_{\pm0.05}$ & $83.8_{\pm1.7}$ & $0.79_{\pm0.02}$ & $0.90_{\pm0.01}$ & $-279.1_{\pm0.5}$\\
GILC-PG~\citep{dou2026gilc} & $4.86_{\pm0.07}$ & $47.2_{\pm0.7}$ & $0.37_{\pm0.02}$ & $0.87_{\pm0.01}$ & $-276.3_{\pm0.4}$\\
SMC-DDM~\citep{pani2025smcddm} & \old{$5.71_{\pm0.04}$} & \old{$41.8_{\pm0.5}$} & \old{$0.53_{\pm0.02}$} & \old{$0.77_{\pm0.01}$} & \old{$-278.3_{\pm0.2}$}\\
TreeG-G~\citep{guo2025treeg} & $5.96_{\pm0.02}$ & $95.5_{\pm0.8}$ & $0.78_{\pm0.02}$ & $0.89_{\pm0.01}$ & $-286.4_{\pm0.2}$\\
\midrule
\rowcolor{green!10} $\VGASGR$ & $\mathbf{7.33}_{\pm0.18}$ & $\mathbf{99.1}_{\pm1.6}$ & $0.88_{\pm0.03}$ & $0.88_{\pm0.01}$ & $-276.3_{\pm0.7}$\\
\rowcolor{green!10} $\VGASRO$ & $6.96_{\pm0.19}$ & $93.9_{\pm2.4}$ & $\mathbf{0.94}_{\pm0.01}$ & $0.90_{\pm0.01}$ & $-275.0_{\pm0.3}$\\
\bottomrule
\end{tabular}}
\end{table}

\subsection{DNA Enhancer Sequence Design}
\label{sec:dna}
\noindent\textbf{Setup.} We evaluate on the genomic enhancer benchmark of \citet{gosai2023}, as adopted by
\citet{wang2025drakes}. The backbone is a masked diffusion language model~\citep{sahoo2024mdlm} pretrained on
this data with $T{=}128$ reverse steps. We steer toward high predicted HepG2 enhancer activity, with a
held-out chromatin-accessibility (ATAC) oracle reserved for evaluation.

\smallskip\noindent\textbf{Metrics.} Reward is the median predicted HepG2 activity (Pred-med); fidelity is
the held-out ATAC pass rate, the $3$-mer and JASPAR-motif correlations to natural
enhancers~\citep{jaspar2022}, and the approximate log-likelihood under the frozen generator (App-LL;
Appendix~\ref{app:metrics}).

\smallskip\noindent\textbf{Results.} As shown in Table~\ref{tab:dna}, \VGAS{} generates sequences
with the highest predicted HepG2 activity among all methods, and does so on both paths: the differentiable
$\VGASGR$ and the non-differentiable $\VGASRO$ each exceed every training-free baseline, the closest
combination methods SMC-DDM and TreeG-G, and the reward-fine-tuned DRAKES. Fidelity does not degrade
alongside it. The held-out accessibility oracle, which never guides sampling, reaches $99.1\%$ against
DRAKES's $92.4\%$, the $3$-mer correlation to natural enhancers is highest under $\VGASRO$, and the
JASPAR-motif correlation is level with DRAKES. Approximate likelihood is the one metric on which guided
samplers trade against the pretrained model, and \old{\VGAS{} concedes more of it than the samplers that
leave the reverse kernel intact and no more than the logit-space guidance it builds on}.

\begin{table}[t]
\centering
\small
\renewcommand{\arraystretch}{0.88}
\setlength{\tabcolsep}{6pt}
\caption{\textbf{Quantitative comparison on protein design} (inverse folding)~\citep{wang2025drakes}
(mean${\pm}$std, $3$ seeds, \old{$N{=}20$}).}
\label{tab:protein}
\resizebox{0.88\textwidth}{!}{%
\begin{tabular}{@{}lrrrrr@{}}
\toprule
Method & Pred-$\mathrm{ddG}$\,$\uparrow$ & \%($\mathrm{ddG}{>}0$) (\%)\,$\uparrow$ & scRMSD\,$\downarrow$ & \old{\%(scRMSD${<}2$\,\AA)}\,$\uparrow$ & Success Rate (\%)\,$\uparrow$\\
\cmidrule(r){1-1}\cmidrule(l){2-6}
Pretrained (MDLM)~\citep{sahoo2024mdlm} & $-0.55_{\pm0.05}$ & $36.1_{\pm0.6}$ & $0.85_{\pm0.01}$ & $90.2_{\pm1.0}$ & $33.9_{\pm0.9}$\\
DG~\citep{nisonoff2024} & $-0.54_{\pm0.05}$ & $36.2_{\pm0.6}$ & $0.85_{\pm0.02}$ & $90.3_{\pm1.0}$ & $33.9_{\pm0.6}$\\
SMC~\citep{delmoral2004feynman} & $0.37_{\pm0.07}$ & $59.9_{\pm2.6}$ & $0.86_{\pm0.01}$ & $\mathbf{93.7}_{\pm2.5}$ & $55.8_{\pm1.0}$\\
TDS~\citep{wu2023smc} & $0.41_{\pm0.14}$ & $62.8_{\pm3.7}$ & $0.86_{\pm0.02}$ & $92.4_{\pm2.2}$ & $57.1_{\pm2.3}$\\
Best-of-$N$~\citep{beirami2025bon} & $0.51_{\pm0.06}$ & $61.7_{\pm1.6}$ & $0.86_{\pm0.02}$ & $92.6_{\pm0.5}$ & $57.2_{\pm1.3}$\\
SVDD~\citep{li2024svdd} & $0.69_{\pm0.08}$ & $69.3_{\pm2.1}$ & $0.85_{\pm0.03}$ & $89.7_{\pm0.9}$ & $65.0_{\pm3.3}$\\
DRAKES~\citep{wang2025drakes} & $1.08_{\pm0.02}$ & $86.1_{\pm0.4}$ & $0.91_{\pm0.01}$ & $92.2_{\pm0.9}$ & $78.6_{\pm1.3}$\\
\midrule
GILC-DB~\citep{dou2026gilc} & $0.64_{\pm0.02}$ & $77.0_{\pm1.3}$ & $0.89_{\pm0.01}$ & $89.2_{\pm0.4}$ & $67.5_{\pm1.4}$\\
GILC-PG~\citep{dou2026gilc} & $0.16_{\pm0.02}$ & $54.3_{\pm0.4}$ & $0.85_{\pm0.01}$ & $90.8_{\pm0.5}$ & $49.7_{\pm0.6}$\\
SMC-DDM~\citep{pani2025smcddm} & $0.62_{\pm0.05}$ & $65.3_{\pm3.0}$ & $0.86_{\pm0.03}$ & $92.5_{\pm0.5}$ & $59.4_{\pm3.3}$\\
\midrule
\rowcolor{green!10} $\VGASGR$ & $\mathbf{1.29}_{\pm0.16}$ & $\mathbf{90.2}_{\pm4.3}$ & $0.88_{\pm0.04}$ & $91.6_{\pm0.9}$ & $\mathbf{82.1}_{\pm3.2}$\\
\rowcolor{green!10} $\VGASRO$ & $1.12_{\pm0.09}$ & $76.0_{\pm1.0}$ & $0.85_{\pm0.02}$ & $92.3_{\pm1.7}$ & $69.7_{\pm0.3}$\\
\bottomrule
\end{tabular}}
\end{table}

\begin{table}[t]
\centering
\footnotesize
\renewcommand{\arraystretch}{0.82}
\setlength{\tabcolsep}{5pt}
\caption{\textbf{Quantitative comparison on QM9 conditional generation}~\citep{ramakrishnan2014qm9}
(mean${\pm}$std, $3$ seeds); \old{metrics and units in Appendix~\ref{app:metrics}.}}
\label{tab:qm9}
\resizebox{0.88\textwidth}{!}{%
\begin{tabular}{@{}lrrrrrr@{}}
\toprule
Method & $C_v$\,$\downarrow$ & $\alpha$\,$\downarrow$ & $\mu$\,$\downarrow$ & gap\,$\downarrow$ & HOMO\,$\downarrow$ & LUMO\,$\downarrow$\\
\cmidrule(r){1-1}\cmidrule(l){2-7}
Pretrained & $4.17_{\pm0.04}$ & $7.77_{\pm0.05}$ & $1.58_{\pm0.04}$ & $1450_{\pm16}$ & $635_{\pm20}$ & $1422_{\pm6}$\\
TFG-Flow~\citep{tfgflow} & $2.20_{\pm0.03}$ & $2.95_{\pm0.15}$ & $0.946_{\pm0.025}$ & $926_{\pm7}$ & $551_{\pm11}$ & $761_{\pm14}$\\
SMC~\citep{delmoral2004feynman} & $3.48_{\pm0.03}$ & $3.95_{\pm0.23}$ & $1.27_{\pm0.03}$ & $1445_{\pm13}$ & $639_{\pm10}$ & $1416_{\pm18}$\\
SVDD~\citep{li2024svdd} & $3.54_{\pm0.06}$ & $3.79_{\pm0.08}$ & $1.32_{\pm0.02}$ & $1416_{\pm22}$ & $629_{\pm19}$ & $1421_{\pm40}$\\
\midrule
GILC-DB~\citep{dou2026gilc} & $2.49_{\pm0.08}$ & $2.68_{\pm0.09}$ & $1.19_{\pm0.06}$ & $913_{\pm15}$ & $500_{\pm2}$ & $892_{\pm8}$\\
GILC-PG~\citep{dou2026gilc} & $2.10_{\pm0.08}$ & $3.33_{\pm0.24}$ & $1.09_{\pm0.05}$ & $859_{\pm41}$ & $420_{\pm9}$ & $665_{\pm10}$\\
SMC-DDM~\citep{pani2025smcddm} & $2.14_{\pm0.06}$ & $3.61_{\pm0.10}$ & $0.958_{\pm0.025}$ & $879_{\pm16}$ & $520_{\pm12}$ & $641_{\pm12}$\\
TreeG-G~\citep{guo2025treeg} & $3.54_{\pm0.13}$ & $5.63_{\pm0.19}$ & $1.77_{\pm0.19}$ & $1380_{\pm25}$ & $712_{\pm4}$ & $1308_{\pm20}$\\
\midrule
\rowcolor{green!10} $\VGASGR$ & $2.06_{\pm0.07}$ & $\mathbf{2.38}_{\pm0.07}$ & $0.921_{\pm0.014}$ & $795_{\pm32}$ & $406_{\pm4}$ & $586_{\pm11}$\\
\rowcolor{green!10} $\VGASRO$ & $\mathbf{1.22}_{\pm0.07}$ & $2.53_{\pm0.13}$ & $\mathbf{0.718}_{\pm0.011}$ & $\mathbf{790}_{\pm10}$ & $\mathbf{374}_{\pm5}$ & $\mathbf{555}_{\pm5}$\\
\bottomrule
\end{tabular}}
\end{table}

\subsection{Protein Sequence Design}
\label{sec:protein}
\noindent\textbf{Setup.} We evaluate on the protein inverse-folding benchmark of
\citet{wang2025drakes}, steering a discrete diffusion model over amino-acid sequences conditioned on a
target structure toward predicted stability ($\mathrm{ddG}$) from an oracle trained on
Megascale~\citep{tsuboyama2023}; each sequence is re-folded to check that it still folds to that
structure.

\smallskip\noindent\textbf{Metrics.} Reward is the median predicted stability (Pred-$\mathrm{ddG}$) and the
fraction of stabilizing designs ($\mathrm{ddG}{>}0$); fidelity is the self-consistency RMSD from re-folding
(scRMSD), the well-folded fraction (scRMSD${<}2$\,\AA), and the joint success rate
(Appendix~\ref{app:metrics}).

\smallskip\noindent\textbf{Results.} As shown in Table~\ref{tab:protein}, \old{both \VGAS{} paths
raise predicted stability above every training-free baseline and above the reward-fine-tuned DRAKES,
and $\VGASGR$ attains the largest fraction of stabilizing designs. Fidelity holds: the joint
structure-and-stability success rate is the highest at $82.1\%$, and the self-consistency RMSD of both
paths stays under the fine-tuned model's, with $\VGASRO$ matching the pretrained model's.}

\subsection{Small-Molecule Property Targeting}
\label{sec:qm9}
\noindent\textbf{Setup.} \old{We evaluate on QM9}~\citep{ramakrishnan2014qm9} under the conditional-generation protocol of TFG-Flow~\citep{tfgflow},
drawing a quantum-chemical property value as the condition and steering the sampler to realize it.

\smallskip\noindent\textbf{Metrics.} We measure reward by the mean absolute error (MAE) between the
requested property value and the value realized, \old{over the six properties of~\citet{tfgflow}; lower
is better. Following them and~\citet{dou2026gilc}, fidelity is chemical validity, and every method is
compared above a $75\%$ validity operating point so that a lower MAE cannot come from invalid molecules}.

\smallskip\noindent\textbf{Results.} As shown in Table~\ref{tab:qm9}, both \VGAS{} variants attain a
lower error than every baseline on all six properties. \old{The ordering across baseline groups matches the sequence
domains: search alone improves little, gradient guidance accounts for most of the gain, and the
combinations for the remainder. $\VGASRO$ is the stronger of the two here, leading on five of the six
properties at the larger reward budget of Table~\ref{tab:nfe}.}

\begin{table}[!tb]
\centering\footnotesize
\renewcommand{\arraystretch}{0.82}
\caption{\textbf{Component ablation and effect of $\SMCAT$.} (a) each proposal to its final \VGAS{} setting; (b) SMC-DDM's selector swapped for $\SMCAT$.}
\label{tab:ablation}
\begin{subtable}[b]{0.345\textwidth}\centering
\setlength{\tabcolsep}{2.5pt}
\resizebox{\linewidth}{!}{\begin{tabular}{@{}lrr@{}}
\toprule
Stage & DNA & Protein\\
& {\scriptsize Pred-med\,$\uparrow$} & {\scriptsize Pred-$\mathrm{ddG}$\,$\uparrow$}\\
\cmidrule(r){1-1}\cmidrule(l){2-3}
GILC-DB & $6.21$ & $0.64$\\
\quad$+$ Gumbel--Rao & $6.35$ & $0.70$\\
\quad$+$ SMC & $6.84$ & $0.85$\\
\rowcolor{green!10}\quad$+$ $\SMCAT$\ ($\VGASGR$) & $\mathbf{7.33}$ & $\mathbf{1.29}$\\
\addlinespace[3pt]
GILC-PG & $4.86$ & $0.16$\\
\quad$+$ SMC & $6.33$ & $0.41$\\
\quad$+$ $\SMCAT$ & $6.61$ & $1.04$\\
\rowcolor{green!10}\quad$+$ RO ($\VGASRO$) & $\mathbf{6.96}$ & $\mathbf{1.12}$\\
\bottomrule
\end{tabular}}
\caption{Component ablation.}
\label{tab:ladder-full}
\end{subtable}\hspace{0.025\textwidth}
\begin{subtable}[b]{0.49\textwidth}\centering
\setlength{\tabcolsep}{2.5pt}
\resizebox{\linewidth}{!}{\begin{tabular}{@{}llrrrr@{}}
\toprule
\rev{$n$} & selector & Pred-med\,$\uparrow$ & ATAC (\%)\,$\uparrow$ & 3-mer\,$\uparrow$ & JASPAR\,$\uparrow$\\
\cmidrule(r){1-2}\cmidrule(l){3-6}
$4$ & SMC-DDM & $5.12_{\pm0.06}$ & \rev{$35.1_{\pm0.6}$} & $0.38_{\pm0.03}$ & $0.70_{\pm0.01}$\\
$4$ & $\SMCAT$ & $\mathbf{5.37}_{\pm0.06}$ & \rev{$\mathbf{42.4}_{\pm1.8}$} & $\mathbf{0.74}_{\pm0.03}$ & $\mathbf{0.90}_{\pm0.01}$\\
$10$ & SMC-DDM & $5.71_{\pm0.04}$ & \rev{$41.8_{\pm0.5}$} & $0.53_{\pm0.02}$ & $0.77_{\pm0.01}$\\
$10$ & $\SMCAT$ & $\mathbf{5.85}_{\pm0.03}$ & \rev{$\mathbf{57.5}_{\pm1.0}$} & $\mathbf{0.90}_{\pm0.01}$ & $\mathbf{0.96}_{\pm0.01}$\\
\bottomrule
\end{tabular}}
\caption{\rev{$\SMCAT$ vs.\ SMC-DDM's selector at $n$ reward draws.}}
\label{tab:smcddm-sel}
\end{subtable}
\end{table}

\section{Ablation Study and Analysis}
\label{sec:analysis}
\noindent\textbf{Ablation study.} To separate the contribution of each component, we conduct an ablation
study that adds them one at a time to both proposals (Table~\ref{tab:ladder-full}; full metrics in Tables~\ref{tab:dna-full}
and~\ref{tab:protein-full}). The variance reduction lifts the proposal on its own, on both paths, and
selection carries the larger share: SMC lifts the weaker policy-gradient proposal the most on DNA, whereas
on protein the adaptive temperature contributes more than SMC on both paths. Every component contributes,
and the fidelity metrics do not degrade as the reward rises.

\noindent\textbf{Generalizability of \SMCAT.} \rev{To test whether the selector generalizes beyond its own
proposal, we hold SMC-DDM's fixed and replace only the selector}
(Table~\ref{tab:smcddm-sel}). Reward and every
fidelity metric rise at \rev{both budgets}, so an adaptive
temperature adds to a first-order approximation of the locally optimal
twist~\citep[\S3.2.1]{pani2025smcddm}; and since that sampler carries the proposal correction and ours
omits it, the comparison separates the selection rule from that choice.

\begin{wrapfigure}{r}{0.345\textwidth}
\vspace{-0.8\baselineskip}
\centering
\includegraphics[width=0.325\textwidth]{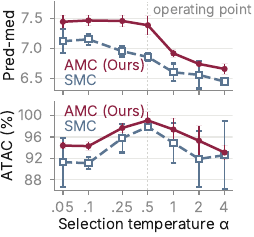}
\caption{Adaptive selection temperature beats every fixed $\alpha$ on the DNA benchmark.}
\label{fig:alpha}
\vspace{-1.2\baselineskip}
\end{wrapfigure}
\noindent\textbf{Effect of the adaptive temperature.} \rev{To show that the gain is not tied to one scale of
$\alpha$, we sweep it over two decades and run both selectors at each value on the same guided proposal,} so SMC is the rung of Table~\ref{tab:ladder-full} ($3$ seeds; Figure~\ref{fig:alpha}).
$\SMCAT$ is higher at every $\alpha$ and at each selector's own best $\alpha$, on the reward and the
held-out oracle alike, and is the flatter of the two; the oracle peaks at our default $\alpha{=}0.5$. No
prescribed temperature recovers the adaptive one, which confirms Eq.~\eqref{eq:std}.

\noindent\textbf{Selection and diversity.} \rev{To separate reward from diversity, we measure diversity three
ways} (Tables~\ref{tab:div} and~\ref{tab:frontier}). The benchmark's mean pairwise Hamming distance does not
register the collapse, while the nearest-neighbour distance and the uniqueness rate do. Sweeping the resampling frequency, $\SMCAT$ attains the higher reward \rev{at every setting
(Figure~\ref{fig:teaser}c), so the concentration is an operating point rather than a property of the
selector}.

\noindent\textbf{Limitations.} \VGAS{} costs a particle population rather than a single trajectory, and
selection concentrates the sample set, which the resampling frequency controls
(Appendix~\ref{app:more-analysis}). We evaluate on biological sequences and small molecules, and leave
text untested.

\section{Conclusion}
\rev{We proposed \VGAS, a training-free sampler for inference-time reward tilting of masked discrete
diffusion models. Guided proposals paired with selection leave three choices at their defaults, so we
first established what such a combination targets once the proposal correction is omitted, and settled
each against that target: a Rao--Blackwellized guidance estimate for differentiable rewards and a
leave-one-out baseline for non-differentiable ones, the tilting applied in the clean-token logits where
the pretrained schedule is preserved, and a selection temperature standardized by the spread of the
per-step rewards. Across regulatory DNA, protein and small-molecule benchmarks, \VGAS{} attains the
best training-free reward on all three and exceeds the reward-fine-tuned reference where one exists,
without any training and without degrading fidelity. We expect it to extend to further domains and
rewards through a selector that needs no retuning.}

\section*{Ethics statement}
\rev{Our benchmarks are regulatory DNA enhancer design, inverse protein folding and small-molecule
property targeting. Reward-guided sampling is indifferent to the sign of the objective: the same procedure
that raises a benign reward would raise a harmful one if such an oracle were supplied. The rewards and the
held-out evaluation oracles used here are published academic predictors and all results are in silico.} We
release the method rather than any designed sequence, and we regard oracle choice and biosecurity screening
as prerequisites for downstream use.

\section*{AI use statement}
We used generative AI tools to edit prose, including grammatical correction, and to write plotting
code. The method, the theory, the proofs and every experimental result are the authors' own. We take
responsibility for the final content of this paper.

\section*{Reproducibility statement}
Every result is reported as a mean over three seeds with the standard deviation alongside, except for
a few secondary tables in Appendix~\ref{app:more-analysis} that report point estimates. Our method is stated
as an algorithm (Algorithms~\ref{alg:corr-rao}--\ref{alg:smc}), with its theory and its experiments given
side by side, and every proof is in Appendix~\ref{app:proofs}. Backbones, oracles, data splits,
hyperparameters and per-step inference costs are in Appendix~\ref{app:impl} and metric definitions in
Appendix~\ref{app:metrics}. We will release the code.

\bibliographystyle{iclr2027_conference}
\bibliography{reference}

\begin{thebibliography}{50}
\providecommand{\natexlab}[1]{#1}
\providecommand{\url}[1]{\texttt{#1}}
\expandafter\ifx\csname urlstyle\endcsname\relax
  \providecommand{\doi}[1]{doi: #1}\else
  \providecommand{\doi}{doi: \begingroup \urlstyle{rm}\Url}\fi

\bibitem[Ahmadian et~al.(2024)Ahmadian, Cremer, Gall{\'e}, Fadaee, Kreutzer,
  Pietquin, {\"U}st{\"u}n, and Hooker]{ahmadian2024rloo}
Arash Ahmadian, Chris Cremer, Matthias Gall{\'e}, Marzieh Fadaee, Julia
  Kreutzer, Olivier Pietquin, Ahmet {\"U}st{\"u}n, and Sara Hooker.
\newblock {Back to Basics: Revisiting REINFORCE-Style Optimization for Learning
  from Human Feedback in LLMs}.
\newblock In \emph{Annual Meeting of the Association for Computational
  Linguistics (ACL)}, 2024.

\bibitem[Beirami et~al.(2025)Beirami, Agarwal, Berant, D'Amour, Eisenstein,
  Nagpal, and Suresh]{beirami2025bon}
Ahmad Beirami, Alekh Agarwal, Jonathan Berant, Alexander D'Amour, Jacob
  Eisenstein, Chirag Nagpal, and Ananda~Theertha Suresh.
\newblock {Theoretical Guarantees on the Best-of-N Alignment Policy}.
\newblock In \emph{International Conference on Machine Learning (ICML)}, 2025.

\bibitem[Bengio et~al.(2013)Bengio, L{\'e}onard, and Courville]{bengio2013ste}
Yoshua Bengio, Nicholas L{\'e}onard, and Aaron Courville.
\newblock {Estimating or Propagating Gradients Through Stochastic Neurons for
  Conditional Computation}.
\newblock \emph{arXiv preprint arXiv:1308.3432}, 2013.

\bibitem[Beskos et~al.(2016)Beskos, Jasra, Kantas, and
  Thiery]{beskos2016convergence}
Alexandros Beskos, Ajay Jasra, Nikolas Kantas, and Alexandre Thiery.
\newblock {On the convergence of adaptive sequential Monte Carlo methods}.
\newblock \emph{The Annals of Applied Probability}, 26\penalty0 (2):\penalty0
  1111--1146, 2016.

\bibitem[Blackwell(1947)]{blackwell1947}
David Blackwell.
\newblock Conditional expectation and unbiased sequential estimation.
\newblock \emph{The Annals of Mathematical Statistics}, 18\penalty0
  (1):\penalty0 105--110, 1947.

\bibitem[Bradley \& Nakkiran(2025)Bradley and Nakkiran]{bradley2025cfgpc}
Arwen Bradley and Preetum Nakkiran.
\newblock {Classifier-Free Guidance is a Predictor-Corrector}.
\newblock \emph{Transactions on Machine Learning Research (TMLR)}, 2025.

\bibitem[Campbell et~al.(2024)Campbell, Yim, Barzilay, Rainforth, and
  Jaakkola]{campbell2024generative}
Andrew Campbell, Jason Yim, Regina Barzilay, Tom Rainforth, and Tommi Jaakkola.
\newblock {Generative Flows on Discrete State-Spaces: Enabling Multimodal Flows
  with Applications to Protein Co-Design}.
\newblock In \emph{International Conference on Machine Learning (ICML)}, 2024.

\bibitem[Castro-Mondragon et~al.(2022)Castro-Mondragon, Riudavets-Puig,
  Rauluseviciute, Berhanu~Lemma, Turchi, Blanc-Mathieu, Lucas, Boddie, Khan,
  Manosalva~P{\'e}rez, Fornes, Leung, Aguirre, Hammal, Schmelter, Baranasic,
  Ballester, Sandelin, Lenhard, Vandepoele, Wasserman, Parcy, and
  Mathelier]{jaspar2022}
Jaime~A. Castro-Mondragon, Rafael Riudavets-Puig, Ieva Rauluseviciute, Roza
  Berhanu~Lemma, Laura Turchi, Romain Blanc-Mathieu, Jeremy Lucas, Paul Boddie,
  Aziz Khan, Nicol{\'a}s Manosalva~P{\'e}rez, Oriol Fornes, Tiffany~Y. Leung,
  Alejandro Aguirre, Fayrouz Hammal, Daniel Schmelter, Damir Baranasic, Benoit
  Ballester, Albin Sandelin, Boris Lenhard, Klaas Vandepoele, Wyeth~W.
  Wasserman, Fran{\c c}ois Parcy, and Anthony Mathelier.
\newblock {JASPAR 2022: the 9th release of the open-access database of
  transcription factor binding profiles}.
\newblock \emph{Nucleic Acids Research}, 50\penalty0 (D1):\penalty0 D165--D173,
  2022.

\bibitem[Chopin \& Papaspiliopoulos(2020)Chopin and
  Papaspiliopoulos]{chopin2020introduction}
Nicolas Chopin and Omiros Papaspiliopoulos.
\newblock \emph{{An Introduction to Sequential Monte Carlo}}.
\newblock Springer Series in Statistics. Springer, 2020.

\bibitem[Chu et~al.(2025)Chu, Wu, Chen, Song, and Yue]{chu2025sgdd}
Wenda Chu, Zihui Wu, Yifan Chen, Yang Song, and Yisong Yue.
\newblock {Split Gibbs Discrete Diffusion Posterior Sampling}.
\newblock In \emph{Advances in Neural Information Processing Systems
  (NeurIPS)}, 2025.

\bibitem[Chung et~al.(2023)Chung, Kim, McCann, Klasky, and Ye]{chung2023dps}
Hyungjin Chung, Jeongsol Kim, Michael~Thompson McCann, Marc~Louis Klasky, and
  Jong~Chul Ye.
\newblock {Diffusion Posterior Sampling for General Noisy Inverse Problems}.
\newblock In \emph{International Conference on Learning Representations
  (ICLR)}, 2023.

\bibitem[Dang et~al.(2026)Dang, Han, Xu, Xu, Srivastava, and
  Ermon]{dang2025pgdlm}
Meihua Dang, Jiaqi Han, Minkai Xu, Kai Xu, Akash Srivastava, and Stefano Ermon.
\newblock {Inference-Time Scaling of Diffusion Language Models via Trajectory
  Refinement}.
\newblock In \emph{Conference on Language Modeling (COLM)}, 2026.
\newblock arXiv:2507.08390.

\bibitem[Del~Moral(2004)]{delmoral2004feynman}
Pierre Del~Moral.
\newblock \emph{{Feynman--Kac Formulae: Genealogical and Interacting Particle
  Systems with Applications}}.
\newblock Probability and Its Applications. Springer, 2004.

\bibitem[Del~Moral et~al.(2012)Del~Moral, Doucet, and
  Jasra]{delmoral2012adaptive}
Pierre Del~Moral, Arnaud Doucet, and Ajay Jasra.
\newblock {An adaptive sequential Monte Carlo method for approximate Bayesian
  computation}.
\newblock \emph{Statistics and Computing}, 22\penalty0 (5):\penalty0
  1009--1020, 2012.

\bibitem[Dou et~al.(2026)Dou, Chen, Li, Li, and Deng]{dou2026gilc}
Hongkun Dou, Zike Chen, Fengji Li, Hongjue Li, and Yue Deng.
\newblock {Plug-and-Play Guidance for Discrete Diffusion Models via
  Gradient-Informed Logit Correction}.
\newblock In \emph{International Conference on Machine Learning (ICML)}, 2026.

\bibitem[Gosai et~al.(2024)Gosai, Castro, Fuentes, Butts, Mouri, Alasoadura,
  Kales, Nguyen, Noche, Rao, Joy, Sabeti, Reilly, and Tewhey]{gosai2023}
Sager~J. Gosai, Rodrigo~I. Castro, Natalia Fuentes, John~C. Butts, Kousuke
  Mouri, Michael Alasoadura, Susan Kales, Thanh Thanh~L. Nguyen, Ramil~R.
  Noche, Arya~S. Rao, Mary~T. Joy, Pardis~C. Sabeti, Steven~K. Reilly, and Ryan
  Tewhey.
\newblock {Machine-guided design of cell-type-targeting cis-regulatory
  elements}.
\newblock \emph{Nature}, 634:\penalty0 1211--1220, 2024.

\bibitem[Greensmith et~al.(2004)Greensmith, Bartlett, and
  Baxter]{greensmith2004variance}
Evan Greensmith, Peter~L. Bartlett, and Jonathan Baxter.
\newblock {Variance Reduction Techniques for Gradient Estimates in
  Reinforcement Learning}.
\newblock \emph{Journal of Machine Learning Research}, 5:\penalty0 1471--1530,
  2004.

\bibitem[Gruver et~al.(2023)Gruver, Stanton, Frey, Rudner, Hotzel,
  Lafrance-Vanasse, Rajpal, Cho, and Wilson]{gruver2023protein}
Nate Gruver, Samuel Stanton, Nathan~C. Frey, Tim G.~J. Rudner, Isidro Hotzel,
  Julien Lafrance-Vanasse, Arvind Rajpal, Kyunghyun Cho, and Andrew~Gordon
  Wilson.
\newblock {Protein Design with Guided Discrete Diffusion}.
\newblock In \emph{Advances in Neural Information Processing Systems
  (NeurIPS)}, 2023.

\bibitem[Guo et~al.(2025)Guo, Yang, Yuan, and Wang]{guo2025treeg}
Yingqing Guo, Yukang Yang, Hui Yuan, and Mengdi Wang.
\newblock {Training-Free Guidance Beyond Differentiability: Scalable Path
  Steering with Tree Search in Diffusion and Flow Models}.
\newblock In \emph{Advances in Neural Information Processing Systems
  (NeurIPS)}, 2025.

\bibitem[Han et~al.(2026)Han, Wang, Xu, Chu, Dang, Ye, Chen, Yue, and
  Ermon]{han2026sdpo}
Jiaqi Han, Austin Wang, Minkai Xu, Wenda Chu, Meihua Dang, Haotian Ye, Huayu
  Chen, Yisong Yue, and Stefano Ermon.
\newblock {Discrete Diffusion Trajectory Alignment via Stepwise Decomposition}.
\newblock In \emph{International Conference on Learning Representations
  (ICLR)}, 2026.

\bibitem[Hasan et~al.(2026)Hasan, Ohanesian, Gazizov, Bengio, Aspuru-Guzik,
  Bondesan, Skreta, and Neklyudov]{hasan2026dfkc}
Mohsin Hasan, Viktor Ohanesian, Artem Gazizov, Yoshua Bengio, Al{\'a}n
  Aspuru-Guzik, Roberto Bondesan, Marta Skreta, and Kirill Neklyudov.
\newblock {Discrete Feynman-Kac Correctors}.
\newblock \emph{arXiv preprint arXiv:2601.10403}, 2026.

\bibitem[He et~al.(2026)He, Rojas, and Tao]{he2026exactly}
Ye~He, Kevin Rojas, and Molei Tao.
\newblock {What Exactly Does Guidance Do in Masked Discrete Diffusion Models}.
\newblock In \emph{International Conference on Learning Representations
  (ICLR)}, 2026.

\bibitem[Holderrieth et~al.(2026)Holderrieth, Chen, Eyring, Shah, Anantharaman,
  He, Akata, Jaakkola, Boffi, and Simchowitz]{holderrieth2026diamond}
Peter Holderrieth, Douglas Chen, Luca Eyring, Ishin Shah, Giri Anantharaman,
  Yutong He, Zeynep Akata, Tommi Jaakkola, Nicholas~M. Boffi, and Max
  Simchowitz.
\newblock {Diamond Maps: Efficient Reward Alignment via Stochastic Flow Maps}.
\newblock In \emph{International Conference on Machine Learning (ICML)}, 2026.

\bibitem[Hoogeboom et~al.(2022)Hoogeboom, Satorras, Vignac, and
  Welling]{hoogeboom2022equivariant}
Emiel Hoogeboom, V{\'i}ctor~Garcia Satorras, Cl{\'e}ment Vignac, and Max
  Welling.
\newblock {Equivariant Diffusion for Molecule Generation in 3D}.
\newblock In \emph{International Conference on Machine Learning (ICML)}, 2022.

\bibitem[Jang et~al.(2017)Jang, Gu, and Poole]{jang2017gumbel}
Eric Jang, Shixiang Gu, and Ben Poole.
\newblock {Categorical Reparameterization with Gumbel-Softmax}.
\newblock In \emph{International Conference on Learning Representations
  (ICLR)}, 2017.

\bibitem[Kim et~al.(2026)Kim, Yoon, Phunyaphibarn, Kim, Mardani, and
  Sung]{kim2026cdm}
Jaihoon Kim, Taehoon Yoon, Prin Phunyaphibarn, Seungjun Kim, Morteza Mardani,
  and Minhyuk Sung.
\newblock {Contrastive Distribution Matching for Amortized Sequential Monte
  Carlo in Discrete Diffusion}.
\newblock \emph{arXiv preprint arXiv:2605.23346}, 2026.

\bibitem[Kool et~al.(2019{\natexlab{a}})Kool, van Hoof, and
  Welling]{kool2019rloo}
Wouter Kool, Herke van Hoof, and Max Welling.
\newblock {Buy 4 REINFORCE Samples, Get a Baseline for Free!}
\newblock In \emph{ICLR Deep Reinforcement Learning Meets Structured Prediction
  Workshop}, 2019{\natexlab{a}}.

\bibitem[Kool et~al.(2019{\natexlab{b}})Kool, van Hoof, and
  Welling]{kool2019stochastic}
Wouter Kool, Herke van Hoof, and Max Welling.
\newblock {Stochastic Beams and Where to Find Them: The Gumbel-Top-k Trick for
  Sampling Sequences Without Replacement}.
\newblock In \emph{International Conference on Machine Learning (ICML)},
  2019{\natexlab{b}}.

\bibitem[Lee et~al.(2025)Lee, Kim, Kim, Park, and Park]{lee2025iterref}
Sanghyun Lee, Sunwoo Kim, Seungryong Kim, Jongho Park, and Dongmin Park.
\newblock {Effective Test-Time Scaling of Discrete Diffusion through Iterative
  Refinement}.
\newblock \emph{arXiv preprint arXiv:2511.05562}, 2025.

\bibitem[Li et~al.(2025)Li, Zhao, Wang, Scalia, Eraslan, Nair, Biancalani, Ji,
  Regev, Levine, and Uehara]{li2024svdd}
Xiner Li, Yulai Zhao, Chenyu Wang, Gabriele Scalia, Gokcen Eraslan, Surag Nair,
  Tommaso Biancalani, Shuiwang Ji, Aviv Regev, Sergey Levine, and Masatoshi
  Uehara.
\newblock {Derivative-Free Guidance in Continuous and Discrete Diffusion Models
  with Soft Value-Based Decoding}.
\newblock In \emph{Advances in Neural Information Processing Systems
  (NeurIPS)}, 2025.
\newblock arXiv:2408.08252.

\bibitem[Lin et~al.(2025)Lin, Li, Ye, Yang, Ermon, Liang, and Ma]{tfgflow}
Haowei Lin, Shanda Li, Haotian Ye, Yiming Yang, Stefano Ermon, Yitao Liang, and
  Jianzhu Ma.
\newblock {TFG-Flow: Training-free Guidance in Multimodal Generative Flow}.
\newblock In \emph{International Conference on Learning Representations
  (ICLR)}, 2025.

\bibitem[Lin et~al.(2023)Lin, Akin, Rao, Hie, Zhu, Lu, Smetanin, Verkuil,
  Kabeli, Shmueli, dos Santos~Costa, Fazel-Zarandi, Sercu, Candido, and
  Rives]{lin2023evolutionary}
Zeming Lin, Halil Akin, Roshan Rao, Brian Hie, Zhongkai Zhu, Wenting Lu, Nikita
  Smetanin, Robert Verkuil, Ori Kabeli, Yaniv Shmueli, Allan dos Santos~Costa,
  Maryam Fazel-Zarandi, Tom Sercu, Salvatore Candido, and Alexander Rives.
\newblock {Evolutionary-scale prediction of atomic-level protein structure with
  a language model}.
\newblock \emph{Science}, 379\penalty0 (6637):\penalty0 1123--1130, 2023.

\bibitem[Maddison et~al.(2014)Maddison, Tarlow, and Minka]{maddison2014astar}
Chris~J Maddison, Daniel Tarlow, and Tom Minka.
\newblock {A* Sampling}.
\newblock In \emph{Advances in Neural Information Processing Systems
  (NeurIPS)}, 2014.

\bibitem[Nisonoff et~al.(2025)Nisonoff, Xiong, Allenspach, and
  Listgarten]{nisonoff2024}
Hunter Nisonoff, Junhao Xiong, Stephan Allenspach, and Jennifer Listgarten.
\newblock {Unlocking Guidance for Discrete State-Space Diffusion and Flow
  Models}.
\newblock In \emph{International Conference on Learning Representations
  (ICLR)}, 2025.

\bibitem[Ou et~al.(2026)Ou, Pani, and Li]{pani2025smcddm}
Zijing Ou, Chinmay Pani, and Yingzhen Li.
\newblock {Inference-Time Scaling of Discrete Diffusion Models via Importance
  Weighting and Optimal Proposal Design}.
\newblock In \emph{International Conference on Learning Representations
  (ICLR)}, 2026.

\bibitem[Paulus et~al.(2021)Paulus, Maddison, and
  Krause]{paulus2020raoblackwell}
Max~B Paulus, Chris~J Maddison, and Andreas Krause.
\newblock {Rao-Blackwellizing the Straight-Through Gumbel-Softmax Gradient
  Estimator}.
\newblock In \emph{International Conference on Learning Representations
  (ICLR)}, 2021.

\bibitem[Phunyaphibarn \& Sung(2026)Phunyaphibarn and
  Sung]{phunyaphibarn2026csmc}
Prin Phunyaphibarn and Minhyuk Sung.
\newblock {Reward-Guided Discrete Diffusion via Clean-Sample Markov Chain for
  Molecule and Biological Sequence Design}.
\newblock \emph{arXiv preprint arXiv:2602.09424}, 2026.

\bibitem[Ramakrishnan et~al.(2014)Ramakrishnan, Dral, Rupp, and von
  Lilienfeld]{ramakrishnan2014qm9}
Raghunathan Ramakrishnan, Pavlo~O. Dral, Matthias Rupp, and O.~Anatole von
  Lilienfeld.
\newblock {Quantum chemistry structures and properties of 134 kilo molecules}.
\newblock \emph{Scientific Data}, 1:\penalty0 140022, 2014.

\bibitem[Rao(1945)]{rao1945}
C.~Radhakrishna Rao.
\newblock Information and the accuracy attainable in the estimation of
  statistical parameters.
\newblock \emph{Bulletin of the Calcutta Mathematical Society}, 37:\penalty0
  81--91, 1945.

\bibitem[Rector-Brooks et~al.(2025)Rector-Brooks, Hasan, Peng, Liu, Mittal,
  Dziri, Bronstein, Chatterjee, Tong, and Bose]{rectorbrooks2025ddpp}
Jarrid Rector-Brooks, Mohsin Hasan, Zhangzhi Peng, Cheng-Hao Liu, Sarthak
  Mittal, Nouha Dziri, Michael Bronstein, Pranam Chatterjee, Alexander Tong,
  and Avishek~Joey Bose.
\newblock {Steering Masked Discrete Diffusion Models via Discrete Denoising
  Posterior Prediction}.
\newblock In \emph{International Conference on Learning Representations
  (ICLR)}, 2025.

\bibitem[Sahoo et~al.(2024)Sahoo, Arriola, Schiff, Gokaslan, Marroquin, Chiu,
  Rush, and Kuleshov]{sahoo2024mdlm}
Subham~Sekhar Sahoo, Marianne Arriola, Yair Schiff, Aaron Gokaslan, Edgar
  Marroquin, Justin~T Chiu, Alexander~M Rush, and Volodymyr Kuleshov.
\newblock {Simple and Effective Masked Diffusion Language Models}.
\newblock In \emph{Advances in Neural Information Processing Systems
  (NeurIPS)}, 2024.

\bibitem[Shao et~al.(2024)Shao, Wang, Zhu, Xu, Song, Bi, Zhang, Zhang, Li, Wu,
  and Guo]{shao2024grpo}
Zhihong Shao, Peiyi Wang, Qihao Zhu, Runxin Xu, Junxiao Song, Xiao Bi, Haowei
  Zhang, Mingchuan Zhang, Y.~K. Li, Y.~Wu, and Daya Guo.
\newblock {DeepSeekMath: Pushing the Limits of Mathematical Reasoning in Open
  Language Models}.
\newblock \emph{arXiv preprint arXiv:2402.03300}, 2024.

\bibitem[Shi et~al.(2024)Shi, Han, Wang, Doucet, and Titsias]{shi2024md4}
Jiaxin Shi, Kehang Han, Zhe Wang, Arnaud Doucet, and Michalis~K. Titsias.
\newblock {Simplified and Generalized Masked Diffusion for Discrete Data}.
\newblock In \emph{Advances in Neural Information Processing Systems
  (NeurIPS)}, 2024.

\bibitem[Singhal et~al.(2025)Singhal, Horvitz, Teehan, Ren, Yu, McKeown, and
  Ranganath]{singhal2025general}
Raghav Singhal, Zachary Horvitz, Ryan Teehan, Mengye Ren, Zhou Yu, Kathleen
  McKeown, and Rajesh Ranganath.
\newblock {A General Framework for Inference-time Scaling and Steering of
  Diffusion Models}.
\newblock In \emph{International Conference on Machine Learning (ICML)}, 2025.

\bibitem[Tejaswi et~al.(2026)Tejaswi, Rout, Caramanis, Shakkottai, and
  Sanghavi]{entrgi}
Atula Tejaswi, Litu Rout, Constantine Caramanis, Sanjay Shakkottai, and Sujay
  Sanghavi.
\newblock {Entropy Aware Reward Guidance for Diffusion Language Model
  Alignment}.
\newblock \emph{arXiv preprint arXiv:2602.05000}, 2026.

\bibitem[Tsuboyama et~al.(2023)Tsuboyama, Dauparas, Chen, Laine,
  Mohseni~Behbahani, Weinstein, Mangan, Ovchinnikov, and
  Rocklin]{tsuboyama2023}
Kotaro Tsuboyama, Justas Dauparas, Jonathan Chen, Elodie Laine, Yasser
  Mohseni~Behbahani, Jonathan~J Weinstein, Niall~M Mangan, Sergey Ovchinnikov,
  and Gabriel~J Rocklin.
\newblock {Mega-scale experimental analysis of protein folding stability in
  biology and design}.
\newblock \emph{Nature}, 620:\penalty0 434--444, 2023.

\bibitem[Wang et~al.(2025)Wang, Uehara, He, Wang, Biancalani, Lal, Jaakkola,
  Levine, Wang, and Regev]{wang2025drakes}
Chenyu Wang, Masatoshi Uehara, Yichun He, Amy Wang, Tommaso Biancalani,
  Avantika Lal, Tommi Jaakkola, Sergey Levine, Hanchen Wang, and Aviv Regev.
\newblock {Fine-Tuning Discrete Diffusion Models via Reward Optimization with
  Applications to DNA and Protein Design}.
\newblock In \emph{International Conference on Learning Representations
  (ICLR)}, 2025.

\bibitem[Williams(1992)]{williams1992}
Ronald~J. Williams.
\newblock Simple statistical gradient-following algorithms for connectionist
  reinforcement learning.
\newblock \emph{Machine Learning}, 8\penalty0 (3--4):\penalty0 229--256, 1992.
\newblock \doi{10.1007/BF00992696}.

\bibitem[Wu et~al.(2023)Wu, Trippe, Naesseth, Blei, and Cunningham]{wu2023smc}
Luhuan Wu, Brian~L Trippe, Christian~A Naesseth, David~M Blei, and John~P
  Cunningham.
\newblock {Practical and Asymptotically Exact Conditional Sampling in Diffusion
  Models}.
\newblock In \emph{Advances in Neural Information Processing Systems
  (NeurIPS)}, 2023.

\bibitem[Yadala~Chanchu et~al.(2026)Yadala~Chanchu, Abdulsamad, and
  Naesseth]{yadala2026nestedsmc}
Lohithsai Yadala~Chanchu, Hany Abdulsamad, and Christian~A. Naesseth.
\newblock {Discrete Diffusion Inference-Time Control with Nested Sequential
  Monte Carlo}.
\newblock \emph{arXiv preprint arXiv:2608.20123}, 2026.

\end{thebibliography}

\newpage
\appendix
\captionsetup[sub]{aboveskip=6pt,belowskip=2pt}

\section*{Appendix Contents}
\begin{description}[leftmargin=2.6em,labelsep=0.6em,itemsep=1.5pt,topsep=3pt,font=\normalfont\bfseries]
\item[\ref{app:proofs}] Deferred Proofs\dotfill\pageref{app:proofs}
\item[\ref{app:impl}] Implementation Details\dotfill\pageref{app:impl}
\item[\ref{app:metrics}] Metric Definitions\dotfill\pageref{app:metrics}
\item[\ref{app:fullres}] Full-Result Ablation Study\dotfill\pageref{app:fullres}
\item[\ref{app:more-analysis}] Additional Analysis\dotfill\pageref{app:more-analysis}
\begin{description}[leftmargin=1.6em,labelsep=0.6em,itemsep=0.5pt,topsep=1.5pt,font=\normalfont]
\item[\ref{app:h-checks}] Reward-Scale Invariance\dotfill\pageref{app:h-checks}
\item[\ref{app:h-isolate}] Isolating the Matched Variance Reducers\dotfill\pageref{app:h-isolate}
\item[\ref{app:h-protein}] Protein Structure versus Stability\dotfill\pageref{app:h-protein}
\item[\ref{app:h-div}] Diversity Metric Blindness\dotfill\pageref{app:h-div}
\old{\item[\ref{app:h-ess}] What a Fixed Temperature Costs\dotfill\pageref{app:h-ess}}
\end{description}
\item[\ref{app:related}] Related Work and Positioning\dotfill\pageref{app:related}
\begin{description}[leftmargin=1.6em,labelsep=0.6em,itemsep=0.5pt,topsep=1.5pt,font=\normalfont]
\item[\ref{app:adaptive-smc}] Relation to Adaptive-Tempering SMC\dotfill\pageref{app:adaptive-smc}
\item[\ref{app:anatomy}] Anatomy of Guidance and Selection in Prior Samplers\dotfill\pageref{app:anatomy}
\item[\ref{app:guidance-math}] Explicit Comparison of the Guidance Signals\dotfill\pageref{app:guidance-math}
\end{description}
\item[\ref{app:selectors}] Sampler and Selector Algorithms\dotfill\pageref{app:selectors}
\item[\ref{app:pg}] Base Proposal Estimators\dotfill\pageref{app:pg}
\end{description}

\old{The four results declared in the main text are restated and then proved, in the order they appear
there; two further results, declared here, follow.}

\section{Deferred Proofs}
\label{app:proofs}

\medskip
\noindent\textbf{Lemma~\ref{lem:target}} (Targets with and without the proposal correction, restated).
\emph{Assume $\|r\|_\infty<\infty$, $\|\gamma\bm g\|_\infty<\infty$, unbiased multinomial resampling,
the deterministic all-\textsc{mask} initial state, and the terminal identity $V(\bm z_0)=r(\bx_0)$. Run
the filter with proposal $q_\gamma$ and incremental weight $w_t^k=\exp(\Delta_t^k/\alpha)$,
$\Delta_t^k=V(\bm z_s^k)-V(\bm z_t^k)$, at fixed $\alpha$. Then for every bounded $\varphi$,
\rev{$\tfrac1N\sum_k\varphi(\bx_0^k)\to\E_{\pi_\gamma}[\varphi]$} almost surely as $N\to\infty$, where
$\pi_\gamma\propto q_\gamma\,e^{r/\alpha}$; and $\pi_0=p^\star$ of Eq.~\eqref{eq:target} with $\beta=\alpha$.
Moreover, adding $\log(p_\theta/q_\gamma)$ to the log-weights replaces $\pi_\gamma$ by $p^\star$ with
$\beta=\alpha$, for every $\gamma$.}

\begin{proof}
\textbf{Step 1 (accumulated weight).} Fix a particle $k$ with trajectory $\bm z_T\to\cdots\to\bm z_0$. Its
accumulated unnormalized weight is a product of successive ratios:
\begin{align}
\prod_{t}w_t^k
&=\exp\!\Big(\tfrac1\alpha\textstyle\sum_t\big(V(\bm z_s^k)-V(\bm z_t^k)\big)\Big)\\[2pt]
&=\exp\!\Big(\tfrac1\alpha\big(V(\bm z_0^k)-V(\bm z_T^k)\big)\Big)\\[2pt]
&=\exp\!\Big(\tfrac1\alpha\big(r(\bx_0^k)-c_0\big)\Big)\\[2pt]
&=e^{-c_0/\alpha}\;e^{r(\bx_0^k)/\alpha}.
\end{align}
The second equality follows from the cancellation of successive increments; the third uses the terminal
identity $V(\bm z_0^k)=r(\bx_0^k)$ and writes $c_0:=V(\bm z_T^k)$, a constant shared by all particles
because the chain starts from the fully masked state.

\old{Therefore, the accumulated weight is a function of $\bx_0^k$ alone, up to a factor common to every
particle.}

\textbf{\old{Step 2 (the potential is positive and bounded).}} Since $V(\bm z)$ is a value of $r$ at the
greedy sequence, $\|V\|_\infty\le\|r\|_\infty$. Therefore,
\begin{equation}
w_t^k\le e^{2\|r\|_\infty/\alpha},
\qquad
G(\bx_0):=e^{r(\bx_0)/\alpha}\in(0,\infty),
\end{equation}
so the per-step weights are positive and uniformly bounded, as the cited convergence result requires, and
the common factor $e^{-c_0/\alpha}$ cancels under self-normalization.

\textbf{\old{Step 3 (Feynman--Kac limit).}} The scheme is therefore a Feynman--Kac particle model with
mutation kernel $q_\gamma$ and per-step potentials $w_t$, whose successive factors cancel to leave the
terminal potential $G$. The intermediate resampling changes the genealogy but not the terminal measure,
which is
\begin{equation}
\eta(\mathrm d\bx_0)=\frac{q_\gamma(\mathrm d\bx_0)\,G(\bx_0)}{\int q_\gamma(\mathrm d\bx_0')\,G(\bx_0')},
\end{equation}
which is exactly
\begin{equation}
\eta(\bx_0)\;\propto\; q_\gamma(\bx_0)\,e^{r(\bx_0)/\alpha}\;=\;\pi_\gamma
\end{equation}
of Eq.~\eqref{eq:target-gamma}. \rev{The accumulated potential enters as the density of this Feynman--Kac
measure relative to the proposal, not as a weight carried past the selection: resampling at every step
leaves the surviving particles unweighted, which is what Algorithm~\ref{alg:smc} returns.} Since $G$ is
bounded and strictly positive and resampling is unbiased, the particle estimator is consistent
\citep[Ch.~7]{delmoral2004feynman}. Therefore, for every bounded measurable $\varphi$,
\begin{equation}
\rev{\tfrac1N\sum_{k}\varphi(\bx_0^k)}\;\xrightarrow[N\to\infty]{\text{a.s.}}\;\eta(\varphi)=\E_{\pi_\gamma}[\varphi] .
\end{equation}
Finally, at $\gamma=0$ the guided reveal reduces to the base kernel, $q_0=p_\theta$, so
\begin{equation}
\pi_0\;\propto\; p_\theta\,e^{r/\alpha}\;=\;p^\star
\end{equation}
of Eq.~\eqref{eq:target} with $\beta=\alpha$, \old{which completes Part~(i)}.

\textbf{Part (ii).} The guided reveal renormalizes a strictly positive tilt of $p_\theta$ inside the
vocabulary, so $q_\gamma(\cdot\mid\bm z)>0$ wherever $p_\theta(\cdot\mid\bm z)>0$, and the density ratio is
bounded at a step revealing $m_t$ positions by $e^{2m_t\|\gamma\bm g\|_\infty}$, hence along the whole path by
\begin{equation}
\frac{\mathrm dp_\theta}{\mathrm dq_\gamma}(\bm z_{T:0})\;\le\;e^{2L\|\gamma\bm g\|_\infty}\;<\;\infty .
\end{equation}
\old{Augmenting} the incremental weight by the proposal correction gives
\begin{equation}
\tilde w_t^k=w_t^k\,\frac{p_\theta(\bm z_s^k\mid\bm z_t^k)}{q_\gamma(\bm z_s^k\mid\bm z_t^k)},
\qquad
\prod_t\tilde w_t^k=e^{-c_0/\alpha}\,e^{r(\bx_0^k)/\alpha}\,
\frac{\mathrm dp_\theta}{\mathrm dq_\gamma}(\bm z_{T:0}^k),
\end{equation}
\old{so the Feynman--Kac model of Steps~1--3 is run with mutation kernel $q_\gamma$ and path potential
$G\cdot\mathrm dp_\theta/\mathrm dq_\gamma$. Its normalized terminal measure is} \rev{obtained by
integrating the path measure over the intermediate states, and the ratio is a path quantity rather than a
ratio of terminal marginals, so the integral is taken before it is cancelled:}
\begin{align}
\rev{\eta(\bx_0)}&\;\rev{\propto\;e^{r(\bx_0)/\alpha}\!\int\! q_\gamma(\bm z_{T:0})\,
\frac{\mathrm dp_\theta}{\mathrm dq_\gamma}(\bm z_{T:0})\,\mathrm d\bm z_{T:1}
\;=\;e^{r(\bx_0)/\alpha}\!\int\! p_\theta(\bm z_{T:0})\,\mathrm d\bm z_{T:1}}\notag\\
&\;\rev{=\;p_\theta(\bx_0)\,e^{r(\bx_0)/\alpha}\;=\;p^\star .}
\end{align}
The correction is given in closed form in Eq.~\eqref{eq:iscorr}.

\old{Therefore, the target is $p^\star$ of Eq.~\eqref{eq:target} with $\beta=\alpha$ for every $\gamma$,
which is Part~(ii).}
\end{proof}

\medskip
\noindent\textbf{Proposition~\ref{prop:efftemp}} (\old{Effective temperature of the guided proposal,
restated}).
\emph{\old{Under \textup{(A1)--(A3)} below, let $E(\gamma)$ be the bound of
Eq.~\eqref{eq:efftemp-bound}, which is $O(\gamma)$. Then the terminal log-twist satisfies
$\Phi_\gamma(\bx_0)=\kappa\gamma\,r(\bx_0)+c+\varepsilon(\bx_0)$ with $c$ independent of $\bx_0$,
$|\varepsilon|\le E(\gamma)$, and $\kappa>0$ any tangential scale for the estimator that carries the reward
gradient into the clean-token logits, against which $\delta_P$ is measured. Consequently
$\pi_\gamma(\bx_0)\propto p_\theta(\bx_0)\,e^{(\kappa\gamma+1/\alpha)r(\bx_0)+\varepsilon(\bx_0)}$, so
omitting the proposal correction leaves the target in the reward-tilted family of Eq.~\eqref{eq:target} at
inverse temperature $\kappa\gamma+1/\alpha$, exactly when $\varepsilon$ is constant in $\bx_0$ and up to the
factor $e^{\varepsilon}$ otherwise. At $\gamma=0$, $E(0)=0$ and $\beta=\alpha$ exactly.}}

\old{\noindent\textbf{Notation.} Let $\tilde r$ be the relaxation of $r$ to the product simplex that the guidance path
differentiates, so that $\tilde r(\mathrm{onehot}(\bx))=r(\bx)$. Write the soft representation of a state
$\bm z$ as $S(\bm z)$, with row $\ell$ equal to $\bm e_{z^\ell}$ at a revealed position and to
$\bm p^\ell=\softmax(\etab^\ell)$ at a masked one, and set $\tilde V(\bm z)=\tilde r(S(\bm z))$, so that
$\tilde V(\bm z_0)=r(\bx_0)$. It is distinct from the selector's greedy value $V$ of \S\ref{sec:filter}; the two agree at $\bm z_0$. Write $\bu^\ell$ for the simplex coordinate at position $\ell$, that is row $\ell$ of
$S(\bm z)$, and let $\bm h^\ell=\nabla_{\bu^\ell}\tilde r(S(\bm z))$ be the simplex gradient at position
$\ell$, with $\bm h$ the concatenation of the $\bm h^\ell$ over positions. The chain starts from the
deterministic all-\textsc{mask} state, so $\tilde V(\bm z_T)$ is common to every trajectory.}

\old{\noindent\textbf{Two log-twists.} The \emph{trajectory} twist
$\bar\Phi_\gamma(\bm z_{T:0})=\log\frac{q_\gamma(\bm z_{T:0})}{p_\theta(\bm z_{T:0})}$ is what the guided
reveal accumulates step by step; the \emph{terminal} twist $\Phi_\gamma(\bx_0)$ of Eq.~\eqref{eq:phi} is what
Eq.~\eqref{eq:target-gamma} needs. Steps~1--4 bound the first and Step~5 passes to the second.}

\old{\noindent\textbf{Assumptions.} We assume:}
\begin{itemize}[leftmargin=3.2em,labelsep=0.6em,align=left,itemsep=2pt,topsep=3pt]
\item[(A1)] \emph{Estimator form.} $\E[\tilde{\bm g}^\ell]=\bm g^\ell=\bm P^\ell\bm h^\ell$, the expectation
being over the estimator's own randomness. Let $\Pi^\ell$ be the orthogonal projector onto the tangent space
of the simplex at $\bm p^\ell$. \rev{Take $\kappa>0$ to be the scale that minimizes the anisotropy,
$\kappa=\arg\min_{\kappa'>0}\sup_\ell\|\bm P^\ell-\kappa'\Pi^\ell\|_{\mathrm{op}}$, so that $\kappa$ and
hence $\kappa\gamma+1/\alpha$ are determined rather than free}, and set the anisotropy
$\delta_P=\sup_\ell\|\bm P^\ell-\kappa\Pi^\ell\|_{\mathrm{op}}$, one pair for all positions and steps.
\item[(A2)] \emph{Smoothness.} $\nabla\tilde r$ is $L_2$-Lipschitz on the product simplex.
\item[(A3)] \emph{Boundedness.} $\|\tilde{\bm g}^\ell\|_\infty\le G$, and write
$\Delta_g=\max_\ell\|\tilde{\bm g}^\ell-\bm g^\ell\|_\infty$, so $\Delta_g\le 2G$.
\end{itemize}

\begin{proof}
\old{Steps~1 and~2 make the trajectory twist explicit and expand it in $\gamma$; Step~3 replaces the
realized estimator by $\kappa$ times the simplex gradient, which is where the reward enters linearly;
Step~4 sums over positions to produce $r(\bx_0)-\tilde V(\bm z_T)$; Step~5 passes to the terminal twist.}

\textbf{Step 1 (exact, along a trajectory).} \old{Eq.~\eqref{eq:phi-exact} below is an identity between
path densities, not between terminal marginals; the two are related in Step~5.} Every position is revealed
exactly once. Summing Eq.~\eqref{eq:iscorr} over the trajectory, negating, and writing $k_\ell=x_0^\ell$ for
the token committed at position $\ell$, we obtain
\begin{equation}
\label{eq:phi-exact}
\bar\Phi_\gamma(\bm z_{T:0})=\sum_{\ell=1}^{L}\Big[\gamma\tilde g^\ell_{k_\ell}
-\big(\lse(\etab^\ell+\gamma\tilde{\bm g}^\ell)-\lse(\etab^\ell)\big)\Big],
\end{equation}
with $\etab^\ell$ evaluated at the step at which $\ell$ was revealed. Eq.~\eqref{eq:iscorr} is written in
\rev{the gauge-shifted $\hat{\bm g}=\bm g-c\bm 1$}; since both terms above shift together under
$\bm g\mapsto\bm g+c\bm 1$, the realized estimator $\tilde{\bm g}$ may be substituted for it without
changing the value.

\old{Therefore, $\bar\Phi_\gamma$ is exactly the displayed sum over positions.}

\textbf{Step 2 (expansion).} Since
$\lse(\etab+\gamma\tilde{\bm g})-\lse(\etab)=\gamma\langle\bm p,\tilde{\bm g}\rangle
+\tfrac{\gamma^2}{2}\mathrm{Var}_{\bm p}(\tilde{\bm g})+O(\gamma^3)$, Eq.~\eqref{eq:phi-exact} gives
\begin{equation}
\label{eq:phi-exp}
\bar\Phi_\gamma=\gamma\sum_\ell\big\langle\tilde{\bm g}^\ell,\ \bm e_{k_\ell}-\bm p^\ell\big\rangle
-\frac{\gamma^2}{2}\sum_\ell\mathrm{Var}_{\bm p^\ell}(\tilde{\bm g}^\ell)+O(\gamma^3),
\end{equation}
still in the realized $\tilde{\bm g}$. Both terms are invariant under $\tilde{\bm g}\mapsto\tilde{\bm g}+c\bm 1$,
the gauge freedom of the softmax parameterization; the constant in $\tilde{\bm g}$ therefore does not reach
$\pi_\gamma$.

\old{Therefore, to second order, $\bar\Phi_\gamma$ is the displayed linear form in the realized
estimator.}

\textbf{Step 3 (identification).} Split the realized estimator into its mean and its deviation.
Since $\|\bm e_k-\bm p^\ell\|_1\le 2$, H\"older's inequality gives
\begin{equation}
\big|\langle\tilde{\bm g}^\ell-\bm g^\ell,\ \bm e_k-\bm p^\ell\rangle\big|
\;\le\;2\|\tilde{\bm g}^\ell-\bm g^\ell\|_\infty\;\le\;2\Delta_g .
\end{equation}
Accordingly, the deviation contributes at most $2\gamma L\Delta_g$ in total.
For the mean, $\Pi^\ell$ is an orthogonal projector, hence self-adjoint, and $\bm e_k-\bm p^\ell$ lies in its
range, so $\langle\kappa\Pi^\ell\bm h^\ell,\bm e_k-\bm p^\ell\rangle=\kappa\langle\bm h^\ell,\bm e_k-\bm
p^\ell\rangle$. Therefore, by (A1) and $\|\bm e_k-\bm p^\ell\|\le\sqrt2$,
\begin{equation}
\big|\langle\bm g^\ell,\bm e_k-\bm p^\ell\rangle-\kappa\langle\bm h^\ell,\bm e_k-\bm p^\ell\rangle\big|
\;\le\;\sqrt2\,\delta_P\|\bm h^\ell\| .
\end{equation}

\old{Therefore, the realized estimator may be replaced by $\kappa\bm h^\ell$, at a cost of $2\Delta_g$ per
position from the deviation and $\sqrt2\,\delta_P\|\bm h^\ell\|$ from the anisotropy. This is where the
reward gradient enters, and with it the linear dependence on $r$.}

\textbf{Step 4a (per-step bounds).} Let $m_t$ be the number of positions revealed at step $t$, so $\sum_t m_t=L$.
Let $\bm d_t$ collect the displacements $\bm e_{k_\ell}-\bm p^\ell$ of those positions, zero-padded to all
$L$ rows, and let $\bm D_t=S(\bm z_s)-S(\bm z_t)$ be the full displacement of the soft representation. The
two differ because revealing a token re-runs the denoiser, so the rows that stay masked move as well; write
\begin{equation}
\label{eq:deltaS}
\delta_S=\sum_t\big\|\bm D_t-\bm d_t\big\| ,
\end{equation}
a quantity determined by the trajectory.

The linear form is additive across positions, so $\sum_{\ell\in t}\langle\bm h^\ell,\bm e_{k_\ell}-\bm
p^\ell\rangle=\langle\bm h,\bm d_t\rangle$ exactly. By (A2),
\begin{equation}
\big|\langle\bm h,\bm D_t\rangle-\big(\tilde V(\bm z_s)-\tilde V(\bm z_t)\big)\big|\;\le\;\tfrac{L_2}{2}\|\bm D_t\|^2 ,
\end{equation}
and by Cauchy--Schwarz,
\begin{equation}
\big|\langle\bm h,\bm D_t\rangle-\langle\bm h,\bm d_t\rangle\big|
\;\le\;\|\bm h\|\,\|\bm D_t-\bm d_t\| ,
\qquad \|\bm h\|=\Big(\textstyle\sum_\ell\|\bm h^\ell\|^2\Big)^{1/2} .
\end{equation}

\old{Therefore, at each step, the linear form differs from the increment of $\tilde V$ by a curvature term
and a clean-prediction drift term.}

\textbf{\old{Step 4b (accumulation).}} The blocks of $\bm d_t$ are orthogonal and $\|\bm e_k-\bm p\|^2\le2$, so $\|\bm d_t\|^2\le 2m_t$.
Therefore, $\|\bm D_t\|^2\le 2\|\bm d_t\|^2+2\|\bm D_t-\bm d_t\|^2$, and summing over steps,
\begin{equation}
\label{eq:Dsum}
\sum_t\|\bm D_t\|^2\;\le\;4\sum_t m_t+2\sum_t\|\bm D_t-\bm d_t\|^2\;\le\;4L+2\delta_S^2 ,
\end{equation}
where the second inequality uses $\sum_t a_t^2\le(\sum_t a_t)^2$ for $a_t\ge0$.

The increments cancel along the trajectory, leaving
$\sum_t\big(\tilde V(\bm z_s)-\tilde V(\bm z_t)\big)=\tilde V(\bm z_0)-\tilde V(\bm z_T)=r(\bx_0)-\tilde
V(\bm z_T)$, with $\tilde V(\bm z_T)$ the value of the all-\textsc{mask} state, common to every
trajectory. Collecting the remainders of Steps~2--4 and taking the supremum over trajectories, we obtain
\begin{equation}
\label{eq:efftemp-bound}
\begin{aligned}
\sup\Big|\bar\Phi_\gamma-\kappa\gamma\big(r(\bx_0)-\tilde V(\bm z_T)\big)\Big|&\;\le\;E(\gamma),\\[3pt]
E(\gamma)&=\gamma\kappa\Big(L_2\big(2L+\delta_S^2\big)+\|\bm h\|\,\delta_S\Big)\\[2pt]
&\quad+\gamma\Big(\sqrt2\,\delta_P\,\textstyle\sum_\ell\|\bm h^\ell\|+2L\Delta_g\Big)
+\tfrac{\gamma^2}{2}LG^2+O\big(\gamma^3LG^3\big),
\end{aligned}
\end{equation}
where $\|\bm h\|$, $\sum_\ell\|\bm h^\ell\|$, $\delta_S$ and $\Delta_g$ are read at their trajectory suprema, each finite
because the product simplex is compact and $T$ is finite.

\old{Therefore, $\bar\Phi_\gamma$ equals $\kappa\gamma(r(\bx_0)-\tilde V(\bm z_T))$ up to $E(\gamma)$,
uniformly over trajectories.}

\textbf{Step 5 (trajectory to terminal).} Marginalizing the path measure over the trajectories that end
at $\bx_0$ gives
$e^{\Phi_\gamma(\bx_0)}=\E_{p_\theta(\bm z_{T:0}\mid\bx_0)}\big[e^{\bar\Phi_\gamma}\big]$. The leading term
of Eq.~\eqref{eq:efftemp-bound} is a function of $\bx_0$ alone, hence factors out of the expectation, and the
remainder is bounded uniformly over trajectories. Therefore,
\begin{equation}
\big|\Phi_\gamma(\bx_0)-\kappa\gamma\big(r(\bx_0)-\tilde V(\bm z_T)\big)\big|
\;\le\;\log\E\big[e^{|\bar\Phi_\gamma-\kappa\gamma(r-\tilde V(\bm z_T))|}\big]\;\le\;E(\gamma).
\end{equation}

\old{Therefore, $\Phi_\gamma(\bx_0)=\kappa\gamma\,r(\bx_0)+c+\varepsilon(\bx_0)$ with
$c=-\kappa\gamma\tilde V(\bm z_T)$ and $|\varepsilon|\le E(\gamma)$, which is
Proposition~\ref{prop:efftemp}.}
\end{proof}

\begin{remark}[What the bound does and does not give]
\label{rem:efftemp}
\old{$E(\gamma)$ is $O(\gamma)$ as $\gamma\to0$ and grows with the sequence length $L$, while the leading term
does not. The bound therefore constrains the \emph{form} of the relation, namely that $\Phi_\gamma$ is linear in
$r$ with slope $\kappa\gamma$, rather than certifying a small relative error at large $L$, and it does not
predict $\kappa$. Of the four
controlling quantities, the reductions of \S\ref{sec:proposal} act on $\Delta_g$, so the variance
reduction and the omitted correction bear on the same quantity.}
\end{remark}

\medskip
\noindent\textbf{Lemma~\ref{lem:baseline}} (Baseline invariance of the resampling law, restated with the
leave-one-out case).
\emph{Fix a scale $c>0$, and for increments $\{\Delta^k\}_{k=1}^{N}\subset\R$ and baselines
$\{b^k\}_{k=1}^N$ let $\hat w^k\propto\exp\!\big((\Delta^k-b^k)/c\big)$, with $\mu=\tfrac1N\sum_j\Delta^j$.
Then \textup{(i)} for a constant baseline $b^k\equiv b$, $\hat w^k=\mathrm{softmax}_k(\Delta^k/c)$,
independent of $b$ and a function of the centered increments $\Delta^k-\mu$ only; and \textup{(ii)} for the
leave-one-out baseline $b^k=\tfrac1{N-1}\sum_{j\neq k}\Delta^j$ with bounded increments,
$\log\hat w^k$ equals its value in \textup{(i)} up to an additive $O\!\big(1/(N{-}1)\big)$.}

\begin{proof}
\textbf{Part (i): constant baseline.} With $b^k\equiv b$, the factor $e^{-b/c}$ appears in both numerator and
denominator and cancels:
\begin{equation}
\hat w^k=\frac{e^{(\Delta^k-b)/c}}{\sum_{j}e^{(\Delta^j-b)/c}}
=\frac{e^{\Delta^k/c}}{\sum_{j}e^{\Delta^j/c}}
=\operatorname{softmax}_k\!\big(\Delta^k/c\big).
\end{equation}
Thus $\hat w^k$ is independent of $b$. Since $\operatorname{softmax}$ is invariant under a common additive
shift of its arguments,
\begin{equation}
\operatorname{softmax}_k\!\big(\Delta^k/c\big)=\operatorname{softmax}_k\!\big((\Delta^k-\mu)/c\big),
\end{equation}
so the law depends on $\{\Delta^k\}$ only through the centered increments $\{\Delta^k-\mu\}$.

\textbf{Part (ii): leave-one-out baseline.} Rewrite the baseline in terms of the group mean $\mu$:
\begin{align}
b^k&=\frac{1}{N-1}\sum_{j\neq k}\Delta^j\\[2pt]
&=\frac{1}{N-1}\big(N\mu-\Delta^k\big)\\[2pt]
&=\mu+\frac{\mu-\Delta^k}{N-1}.
\end{align}
Define $\rho^k:=\dfrac{\mu-\Delta^k}{(N-1)c}$. Then the exponent differs from part (i) by exactly
$\rho^k$,
\begin{equation}
\frac{\Delta^k-b^k}{c}=\frac{\Delta^k-\mu}{c}-\rho^k,
\end{equation}
and if $|\Delta^k|\le B$ for all $k$,
\begin{equation}
|\rho^k|\le\frac{|\mu|+|\Delta^k|}{(N-1)c}\le\frac{2B}{(N-1)c}=O\!\big(1/(N{-}1)\big).
\end{equation}
Let $\ell^k:=\log\hat w^k$ under the leave-one-out baseline and $\ell^k_0$ its value from part (i).
Subtracting the two normalized log-weights gives
\begin{equation}
\ell^k-\ell^k_0
=-\rho^k+\log\frac{\sum_j e^{(\Delta^j-\mu)/c}}{\sum_j e^{(\Delta^j-\mu)/c}\,e^{-\rho^j}}.
\end{equation}
Because each summand in the denominator is the corresponding numerator summand scaled by
$e^{-\rho^j}\in\big[e^{-\max_j|\rho^j|},\,e^{\max_j|\rho^j|}\big]$, the log-ratio lies in
$\big[-\max_j|\rho^j|,\,\max_j|\rho^j|\big]$, so
\begin{equation}
\big|\ell^k-\ell^k_0\big|\le 2\max_j|\rho^j|=O\!\big(1/(N{-}1)\big).
\end{equation}
Hence, the resampling law is exactly invariant to any constant baseline, and invariant to the leave-one-out
baseline up to $O(1/(N{-}1))$.
\end{proof}

\begin{remark}[Leave-one-out baseline in the selector]
\label{rem:loo}
Replacing the group mean $\mu$ by the leave-one-out baseline changes $\log\hat w^k$ by
$O\!\big(1/(N{-}1)\big)$, so the selector is insensitive to which of the two is used. This is the prediction
tested in Table~\ref{tab:rloo}. Since the group-mean shift cancels exactly, the only operative effect of
standardization on the resampling law is the replacement $\alpha\mapsto\alpha\sigma$.
\end{remark}

\medskip
\noindent\textbf{Proposition~\ref{prop:at}} (Target of $\SMCAT$, restated).
\emph{Assume $\|r\|_\infty<\infty$, $\alpha>0$, $\sigma_t\ge\epsilon>0$, and multinomial resampling at every
step, and assume as in Lemma~\ref{lem:target} that the chain starts from the deterministic
all-\textsc{mask} state and that $V(\bm z_0)=r(\bx_0)$. The accumulated weight of a particle is proportional to $e^{R_\sigma}$ with $R_\sigma=\sum_t\Delta_t/\alpha_t$ and
$\alpha_t=\alpha\sigma_t$.
\textup{(i)} Let $\varphi$ be bounded and measurable, and assume the empirical $\mu_t,\sigma_t$ admit
deterministic limits $\mu_t^\star,\sigma_t^\star$ and that the map from the empirical measure to
$(\mu_t,\sigma_t)$ satisfies the regularity conditions of \citet{beskos2016convergence}. Let $\pi^\star_{\SMCAT}\propto q_\gamma
e^{R_{\sigma^\star}}$ be the path measure built from those limits. Then the self-normalized $\SMCAT$
estimator converges \emph{in probability} to $\E_{\pi^\star_{\SMCAT}}[\varphi(\bx_0)]$.
\textup{(ii)} If $\sigma^\star_t\equiv\sigma$, then the $\bx_0$-marginal of $\pi^\star_{\SMCAT}$ is
$\pi_\gamma$ at temperature $\alpha\sigma$.}

\noindent\textbf{Setup.} At each step $\SMCAT$ alternates a mutation (one guided reveal $\bm z_s^k\sim
q_\gamma(\cdot\mid\bm z_t^k)$ of Lemma~\ref{lem:target}) with a selection that resamples the $N$
particles with probabilities proportional to the potential
\begin{equation}
G_t(\bm z_s^k)=\exp\!\Big(\tfrac{\Delta_t^k-\mu_t}{\alpha\sigma_t}\Big),\qquad
\Delta_t^k=V(\bm z_s^k)-V(\bm z_t^k),
\end{equation}
with $\mu_t,\sigma_t$ the mean and standard deviation of $\{\Delta_t^k\}_{k=1}^{N}$, so the scheme is a
Feynman--Kac model with per-step potentials $G_t$. \rev{The accumulated potential
$G_{0:T}^k=\prod_t w_t^k$ is the density of the target path measure relative to the proposal; with
resampling at every step the surviving particles are unweighted.}

\begin{proof}

\textbf{\old{Step 1 (accumulated weight).}} \old{The population constants $e^{-\mu_t/\alpha_t}$ are common
to the $N$ particles at step $t$, so they cancel in \rev{$G_{0:T}^k$} by Lemma~\ref{lem:baseline}(i). Therefore,}
\begin{equation}
G_{0:T}^k\;\propto\;\exp\big(R_\sigma^k\big),
\qquad
R_\sigma^k=\sum_t\frac{\Delta_t^k}{\alpha_t},
\qquad \alpha_t=\alpha\sigma_t,
\end{equation}
\old{which is the first claim of the proposition. Because $\alpha_t$ varies with $t$, the increments do not
cancel and $R_\sigma$ is a functional of the whole trajectory, so the reweighted measure lives on paths;
the estimator averages $\varphi(\bx_0)$, a bounded function of the terminal state, so its limit is an
expectation under the $\bx_0$-marginal of that path measure.}

\textbf{Step \old{2}: the potentials are bounded above and below.} Since $V$ is bounded by $\|r\|_\infty$, both
$\Delta_t^k$ and $\mu_t$ lie in $[-2\|r\|_\infty,2\|r\|_\infty\,]$, so
\begin{equation}
\big|\Delta_t^k-\mu_t\big|\;\le\;4\|r\|_\infty .
\end{equation}
The floor $\sigma_t\ge\epsilon$ gives $\alpha\sigma_t\ge\alpha\epsilon>0$. Therefore,
\begin{equation}
0\;<\;e^{-4\|r\|_\infty/(\alpha\epsilon)}\;\le\;G_t\;\le\;e^{4\|r\|_\infty/(\alpha\epsilon)}\;<\;\infty .
\end{equation}

\textbf{Step \old{3}: placing the scheme in the framework of \citet{beskos2016convergence}.}
Their framework covers a Feynman--Kac flow whose potential at step $t$ is parameterized by a bounded summary
statistic of the particle population that is about to be reweighted, resampled multinomially at every step.
\old{It asks for three things.}
\begin{itemize}[leftmargin=2.4em,labelsep=0.6em,align=left,itemsep=2pt,topsep=3pt]
\item[\old{(a)}] \old{\emph{A bounded summary statistic.} Ours is $(\mu_t,\sigma_t)$, the first two
empirical moments of $\{\Delta_t^k\}$, bounded by Step~2 through $\|r\|_\infty<\infty$ and the floor
$\epsilon$. \emph{Verified.}}
\item[\old{(b)}] \old{\emph{Multinomial resampling at every step.} A hypothesis of the proposition.
\emph{Assumed as a hypothesis.}}
\item[\old{(c)}] \old{\emph{Deterministic limits and regularity.} That the empirical moments admit
deterministic limits, and that the map from the empirical measure to $(\mu_t,\sigma_t)$, which is non-linear
in $\sigma_t$, satisfies their regularity conditions. \emph{Assumed, and carried in the hypotheses of
part~(i).}}
\end{itemize}
Therefore, by their Theorems~3.1 and~3.2, for every bounded $\varphi$,
\begin{equation}
\rev{\tfrac1N\sum_{k}\varphi(\bx_0^k)}\;\xrightarrow[N\to\infty]{\ \mathbb P\ }\;\E_{\pi^\star_{\SMCAT}}[\varphi],
\end{equation}

\old{Therefore, the self-normalized estimator converges in probability to
$\E_{\pi^\star_{\SMCAT}}[\varphi]$ under the $\bx_0$-marginal of $\pi^\star_{\SMCAT}$, which is
part~(i).}

\textbf{Step \old{4}: where our statistic sits.} Their adaptation uses the population that is about to be
reweighted. The same holds here:
\begin{equation}
(\mu_t,\sigma_t)=\text{moments of }\{\Delta_t^k\}_{k=1}^{N},
\qquad
\Delta_t^k=V(\bm z_s^k)-V(\bm z_t^k),
\qquad
\hat w^k\propto G_t(\bm z_s^k),
\end{equation}
so the statistic is read off the mutated states $\{\bm z_s^k\}$ and the potential is applied to those same
$N$ particles. Accordingly, the statistic and the reweighting act on one population, as their structure
requires it.

\textbf{Step \old{5}: the constant-$\sigma$ case.} Let $\sigma^\star_t\equiv\sigma$, and evaluate the limiting
potentials at that value. Then
\begin{equation}
G_t(\bm z)=\exp\!\Big(\tfrac{\Delta_t(\bm z)-\mu_t}{\alpha\sigma}\Big)
=e^{-\mu_t/(\alpha\sigma)}\,\exp\!\Big(\tfrac{\Delta_t(\bm z)}{\alpha\sigma}\Big)
\;\propto\;\exp\!\Big(\tfrac{\Delta_t(\bm z)}{\alpha\sigma}\Big),
\end{equation}
the last step because the population constant $e^{-\mu_t/(\alpha\sigma)}$ cancels in the resampling softmax
by Lemma~\ref{lem:baseline}(i). This is the fixed-temperature potential of Lemma~\ref{lem:target} with
$\alpha$ replaced by $\alpha\sigma$, whose increments cancel along the trajectory. Therefore,
\begin{equation}
R_\sigma=\sum_t\frac{\Delta_t}{\alpha\sigma}=\frac{r(\bx_0)-V(\bm z_T)}{\alpha\sigma},
\end{equation}
which is $r/(\alpha\sigma)$ up to an additive constant.

\old{Therefore, the $\bx_0$-marginal of $\pi^\star_{\SMCAT}$ is $\pi_\gamma$ at temperature
$\alpha\sigma$, which is part~(ii).}
\end{proof}

\begin{lemma}[Optimal constant baseline for the score-function correction]
\label{lem:rloo}
Let $x\sim\bm p$ on a finite vocabulary with $\bm p$ non-degenerate, $\hat\bx=\mathrm{onehot}(x)$ and
$R=r(x)$. For $b\in\R$ define the single-sample correction
$\bm u(b)=(R-b)(\hat\bx-\bm p)$ and write $Q=\|\hat\bx-\bm p\|^2$, so $\E[Q]>0$. Then \textup{(i)}
$\E[\bm u(b)]=\nabla_{\etab}\E[r]$ for every $b$; and \textup{(ii)} the variance of $\bm u(b)$ is
minimized at $b^\star=\E[RQ]/\E[Q]$, and if $b^\star>0$ it is strictly smaller than that of $\bm u(0)$ for
every $b\in(0,2b^\star)$.
\end{lemma}

\noindent Part~\textup{(i)} is the Reinforce unbiasedness of \citet[Thm.~1]{williams1992} and
part~\textup{(ii)} is \citet[Thm.~11]{greensmith2004variance} specialized to a categorical distribution,
where $\nabla_{\etab}\log p(x)=\hat\bx-\bm p$; the interval $(0,2b^\star)$ is a one-line consequence of the
excess-variance identity in that theorem rather than a separate result of theirs. We give the short proof
in our notation.

\begin{proof}
\textbf{Part (i).} Since $\E[\hat\bx-\bm p]=\bm 0$, the baseline term vanishes. Therefore,
\begin{equation}
\nabla_{\eta_j}\E[r]=\sum_k r_k\,p_k(\delta_{kj}-p_j)=\E\big[R(\hat x_j-p_j)\big]=\E\big[u_j(b)\big] ,
\end{equation}
the second equality being the softmax Jacobian $\partial p_k/\partial\eta_j=p_k(\delta_{kj}-p_j)$.

\textbf{Part (ii).} By part~(i) the mean of $\bm u(b)$ does not depend on $b$, so comparing variances is
comparing second moments. Accordingly,
\begin{equation}
\E\|\bm u(b)\|^2=\E\big[(R-b)^2Q\big]=\E[R^2Q]-2b\,\E[RQ]+b^2\,\E[Q] ,
\end{equation}
a quadratic in $b$, strictly convex because $\E[Q]>0$, whose minimizer is $b^\star=\E[RQ]/\E[Q]$.
Subtracting the value at $b=0$ gives
\begin{equation}
\E\|\bm u(b)\|^2-\E\|\bm u(0)\|^2=b\,\E[Q]\,\big(b-2b^\star\big) ,
\end{equation}
which is negative exactly when $0<b<2b^\star$.
\end{proof}

\begin{remark}[Leave-one-out as a plug-in]
\label{rem:rloo-plugin}
The leave-one-out baseline estimates $\E[R]$ from the remaining samples of the group. When $R$ and $Q$ are
uncorrelated, $b^\star=\E[R]$, so it is then an unbiased plug-in for the variance-optimal constant baseline;
in general it removes the reward's mean offset without matching $b^\star$ exactly, and part~(ii) requires
$b^\star>0$. Two further gaps separate the lemma from $\bm g_{\mathrm{RO}}$ of Eq.~\eqref{eq:rloo}: its
baseline is random rather than constant, and the $n$ leave-one-out baselines are coupled, so the variance of
the averaged estimator carries cross terms the single-sample comparison does not see. Unbiasedness survives
both, since $b^i$ is independent of sample $i$; the variance reduction is the motivation for the estimator,
and Table~\ref{tab:rloo} is what tests it.
\end{remark}

\medskip
\noindent The last result concerns the placement of the tilt and is used in
Appendix~\ref{app:guidance-math}, where the two proposal forms it compares, Eq.~\eqref{eq:posthoc} and
Eq.~\eqref{eq:inside}, are derived.

\begin{proposition}[Schedule leakage]
\label{prop:leak}
Fix a masked position $\ell$ and a step with $a_t\in(0,1)$, let $\gamma>0$, and write
$Z^\ell=\E_{p_\theta^\ell}[e^{\gamma g}]$ for the tilt normalizer of Eq.~\eqref{eq:posthoc}. Let
$\tilde a_t^\ell=1-q^\ell(\textsc{mask})$ be the reveal probability actually induced by the guided proposal.
Then
\begin{equation}
\label{eq:leak}
\tilde a_t^{\ell,\,\mathrm{post}}=\frac{a_t Z^\ell}{1-a_t+a_t Z^\ell},
\qquad
\tilde a_t^{\ell,\,\mathrm{ours}}=a_t .
\end{equation}
Consequently:
\begin{itemize}[leftmargin=2.4em,labelsep=0.6em,align=left,itemsep=2pt,topsep=3pt]
\item[(a)] $\tilde a_t^{\ell,\,\mathrm{post}}$ is strictly increasing in $Z^\ell$, with
$\tilde a_t^{\ell,\,\mathrm{post}}\gtrless a_t$ according as $Z^\ell\gtrless1$.
\item[(b)] Under the gauge $\E_{p_\theta^\ell}[g]=0$, Jensen's inequality gives $Z^\ell\ge1$, with equality
iff $g$ is $p_\theta^\ell$-almost surely constant, so unless $g$ is degenerate guidance strictly accelerates
unmasking.
\item[\old{(b$'$)}] \old{Under that same gauge, $Z^\ell$ is nondecreasing in $\gamma$, so by part~(a) the
acceleration grows with $\gamma$ by a state- and step-dependent amount.}
\item[(c)] $q^\ell$ of Eq.~\eqref{eq:inside} is exactly invariant under
$\bm g^\ell\mapsto\bm g^\ell+c\bm 1$, whereas $\tilde a_t^{\ell,\,\mathrm{post}}\to1$ as $c\to+\infty$ and
$\to0$ as $c\to-\infty$.
\end{itemize}
\old{\noindent Therefore, the post-hoc reveal probability is not determined by the guidance signal itself,
whereas the logit-space one is.}
\end{proposition}
\begin{proof}
Both lines of Eq.~\eqref{eq:leak} are read off Eq.~\eqref{eq:posthoc} and Eq.~\eqref{eq:inside}.

\textbf{Part (a).} Write $f(z)=a_tz/(1-a_t+a_tz)$. Then
\begin{equation}
f'(z)=\frac{a_t(1-a_t)}{(1-a_t+a_tz)^2}\;>\;0
\qquad\text{for }a_t\in(0,1),
\end{equation}
and $f(1)=a_t$. Therefore, $f$ is strictly increasing and crosses $a_t$ at $z=1$, which is the claim.

\textbf{Part (b).} By Jensen's inequality and the gauge $\E_{p_\theta^\ell}[g]=0$,
\begin{equation}
Z^\ell=\E\big[e^{\gamma g}\big]\;\ge\;e^{\gamma\E[g]}=1,
\end{equation}
with equality if and only if $\gamma g$ is degenerate.

\textbf{\old{Part (b$'$).}} For the dependence on $\gamma$, let
$\Lambda(\gamma)=\log\E[e^{\gamma g}]$ be the cumulant generating function. Accordingly,
\begin{equation}
\Lambda(0)=0,\qquad \Lambda'(0)=\E[g]=0,\qquad \Lambda''(\gamma)=\mathrm{Var}_{\text{tilted}}(g)\ge0 .
\end{equation}
The third identity makes $\Lambda'$ nondecreasing, so $\Lambda'(\gamma)\ge\Lambda'(0)=0$ for $\gamma>0$, and $\Lambda$ is
nondecreasing there, strictly so unless $g$ is degenerate.

\textbf{Part (c).} The shift $\bm g^\ell\mapsto\bm g^\ell+c\bm 1$ multiplies every $e^{\gamma g_k}$, and
hence $Z^\ell$, by $e^{\gamma c}$. \old{In Eq.~\eqref{eq:inside} the two factors cancel,}
\begin{equation}
\old{q^\ell(k)\big|_{\bm g^\ell+c\bm 1}
=\frac{a_t\,p_\theta^\ell(k)\,e^{\gamma(g_k+c)}}{e^{\gamma c}Z^\ell}
=\frac{a_t\,p_\theta^\ell(k)\,e^{\gamma g_k}}{Z^\ell}
=q^\ell(k),}
\end{equation}
\old{so $\tilde a_t^{\ell,\,\mathrm{ours}}=a_t$ is unmoved. In Eq.~\eqref{eq:leak} the factor survives,}
\begin{equation}
\old{\begin{aligned}
\tilde a_t^{\ell,\,\mathrm{post}}\big|_{\bm g^\ell+c\bm 1}
&=\frac{a_t\,e^{\gamma c}Z^\ell}{1-a_t+a_t\,e^{\gamma c}Z^\ell},\\[4pt]
&\longrightarrow\;1\ \ (c\to+\infty),
\qquad
\longrightarrow\;0\ \ (c\to-\infty).
\end{aligned}}
\end{equation}

\old{Therefore, a constant that the guidance signal does not fix moves the post-hoc reveal schedule across
its whole range, while leaving the logit-space one at $a_t$.}
\end{proof}

\section{Implementation Details}
\label{app:impl}
All methods are reproduced from their official repositories under a common setup, following the protocol
of~\citet{wang2025drakes}. \old{We carry \rev{$N{=}20$} particles or candidates for every search method, which is the particle count
$N$ of \S\ref{sec:filter} for ours.} \old{Table~\ref{tab:hparams} lists every hyperparameter of \VGAS{} and its value on each domain.}
Gumbel--Rao lifts the protein \emph{proposal} at no extra denoiser cost ($0.64\!\to\!0.70$)\old{, and
we read $M$ off the proposal, where the effect is resolved, adopting $M{=}8$ on both sequence benchmarks
(Table~\ref{tab:grm})}.

\old{\begin{table}[h]
\centering\small
\setlength{\tabcolsep}{6pt}
\caption{\textbf{Hyperparameters of \VGAS.} The guidance configuration follows the one
GILC~\citep{dou2026gilc} releases for these benchmarks, with $\gamma$ selected on the held-out validation
split of \S\ref{app:impl}; GILC is run at the same values wherever it is compared against. The selector's
$\alpha$ and $r$-freq are shared across all three domains.}
\label{tab:hparams}
\begin{tabular}{@{}llrrr@{}}
\toprule
Symbol & Controls & DNA & Protein & QM9\\
\cmidrule(r){1-2}\cmidrule(l){3-5}
$T$ & denoising steps & $128$ & $50$ & \rev{$100$}\\
\rev{$N$} & particles carried by the selector & $20$ & $20$ & $20$\\
$\gamma$ & guidance scale on the logit correction & $11000$ & $1000$ & $10$\\
$n$ & Monte Carlo samples, $\VGASGR$ & $10$ & $5$ & $5$\\
$n$ & Monte Carlo samples, $\VGASRO$ & $20$ & $20$ & $20$\\
$M$ & Gumbel--Rao samples ($\VGASGR$) & $8$ & $8$ & $4$\\
$\tau$ & Gumbel--Softmax straight-through temperature & $1.0$ & $1.0$ & $1.0$\\
$\alpha$ & base selection temperature, $\alpha_t=\alpha\sigma_t$ & $0.5$ & $0.5$ & $0.5$\\
$r$-freq & steps between resampling events & $1$ & $1$ & $1$\\
\bottomrule
\end{tabular}
\end{table}}

\begin{table}[h]
\centering\footnotesize
\setlength{\tabcolsep}{5pt}
\caption{\rev{\textbf{Inference cost per denoising step}, in the convention
of~\citet[Tab.~2]{dou2026gilc}.}}
\label{tab:nfe}
\begin{tabular}{@{}llrc@{}}
\toprule
Method & Denoiser calls & Reward calls & Denoiser backward\\
\cmidrule(r){1-1}\cmidrule(lr){2-3}\cmidrule(l){4-4}
Pretrained (MDLM)~\citep{sahoo2024mdlm} & $1$ & $0$ & \xmark\\
DRAKES~\citep{wang2025drakes} & $1$ & $0$ & \xmark\\
\addlinespace[2pt]
Best-of-$N$~\citep{beirami2025bon} & $N$ & $0$ & \xmark\\
SMC~\citep{delmoral2004feynman} & $3N$ & $2N$ & \xmark\\
SVDD~\citep{li2024svdd} & $N{+}1$ & $N$ & \xmark\\
\addlinespace[2pt]
GILC-DB~\citep{dou2026gilc} & $1$ & $n$ & \xmark\\
GILC-PG~\citep{dou2026gilc} & $1$ & $n$ & \xmark\\
\addlinespace[2pt]
TDS~\citep{wu2023smc} & $4N$ & $3N$ & \cmark\\
SMC-DDM~\citep{pani2025smcddm} & $N$ & $Nn$ & \cmark\\
TreeG-G~\citep{guo2025treeg} & $N$ & $Nn$ & \cmark\\
\cmidrule(r){1-1}\cmidrule(lr){2-3}\cmidrule(l){4-4}
\rowcolor{green!10} $\VGASGR$ (ours) & $N$ & $N(n{+}1)$ & \xmark\\
\rowcolor{green!10} $\VGASRO$ (ours) & $N$ & $N(n{+}1)$ & \xmark\\
\bottomrule
\end{tabular}
\end{table}

\paragraph{Tuning protocol.} All \VGAS{} hyperparameters ($\gamma$, $\alpha$, $M$, and the resampling
frequency) were selected on a held-out validation split of generations
that is disjoint from the evaluation set whose numbers we report; the reported metrics are never used for
selection. Baselines are run at their official default configurations. This keeps the comparison free of
test-set tuning and avoids re-tuning any baseline downward.

Results are means over three seeds ($640$ generated sequences per seed on DNA). The reward oracles for
guidance and the held-out oracles for evaluation follow~\citet{wang2025drakes} \rev{on DNA and protein,
and~\citet{tfgflow} on QM9}.

\rev{TreeG releases no official protein-design pipeline, so on that benchmark the combination
comparison is against SMC-DDM only.}

\rev{The protein benchmark is the Megascale corpus of~\citet{tsuboyama2023} as curated
by~\citet{wang2025drakes}, whose wild-type structures are clustered with Foldseek and split by cluster,
leaving a test set of $12$ proteins. On QM9 the generator is the multimodal flow model
of~\citet{tfgflow}, which adapts Multiflow~\citep{campbell2024generative} to small molecules with the
equivariant architecture of~\citet{hoogeboom2022equivariant}. A molecule there is a three-dimensional
point cloud of heavy atoms carrying continuous coordinates and discrete atom types. The coordinates are
guided as in~\citet{tfgflow} and our logit correction is added on the atom-type logits.}

\old{\paragraph{DG.} DG~\citep{nisonoff2024} is run at its official default configuration like every other
baseline. On DNA it behaves as a guidance method: the reward rises from $0.18$ to $1.10$ while every
fidelity metric falls (Table~\ref{tab:dna}). On protein neither axis moves beyond one standard deviation
of the pretrained model (Table~\ref{tab:protein}).}

\paragraph{SMC-DDM and TreeG-G.} We run SMC-DDM~\citep{pani2025smcddm} and TreeG's gradient variant
TreeG-G~\citep{guo2025treeg} from their official repositories on the shared backbone, matching the
denoiser-call budget of the other search methods and otherwise following the setup above. TreeG-G releases no
official protein-design pipeline, so we omit it from the protein comparison in Table~\ref{tab:protein}.
No schedule shape is imposed on top of $\alpha_t=\alpha\sigma_t$: the selector's only free parameter is
the scalar $\alpha$, and \rev{the same value is used on all three domains}. \old{Table~\ref{tab:nfe} reports the per-step inference cost of every method underlying
Tables~\ref{tab:dna}--\ref{tab:protein}:} \rev{calls to the denoising network and to the reward oracle,
per step and per generated sequence, with a third column for whether the method backpropagates through the
\emph{denoiser}, which the call counts do not capture. There $n$ is each method's own Monte Carlo sample
count per guidance estimate. Best-of-$N$ scores only completed sequences, so its reward calls fall outside
the per-step count, and the extra reward call in our rows is the selector's value $V(\bm z_s)$, the previous
step's being cached. We report this budget rather than claim per-call optimality. The budgets are not equal
across rows. On DNA, $\VGASGR$ and SMC-DDM both run at $n{=}10$ while $\VGASRO$ runs at $n{=}20$, so
$\VGASRO$ issues about twice the reward calls of either; on QM9 the two paths run at $n{=}5$ and $n{=}20$, a
factor of three and a half. Where the non-differentiable path leads, it leads at the larger budget.}

\section{Metric Definitions}
\label{app:metrics}
This appendix gives the precise definition of every reward and fidelity metric used in
\S\ref{sec:experiments} and in the full ablation of \S\ref{app:fullres}.
Throughout, ``reward'' metrics measure how strongly generation is pushed toward the guidance objective,
while ``fidelity'' metrics measure whether the resulting designs stay realistic rather than becoming
adversarial artifacts that merely fool the reward oracle.

\smallskip\noindent\textbf{DNA enhancer design.}
\begin{itemize}[leftmargin=1.4em,itemsep=1pt,topsep=2pt]
\item \emph{Pred-med}, the median, over generated sequences, of the predicted HepG2 enhancer
activity from the reward oracle. Higher is better.
\item \emph{ATAC\%}, the fraction of designs judged accessible by a held-out
chromatin-accessibility (ATAC) oracle that is never used for guidance. Because it is independent of the
guidance reward, a rise in ATAC\% together with reward indicates genuine generalization rather than
reward hacking.
\item \emph{$3$-mer, JASPAR}, Pearson correlations between the designs and natural enhancers, in
$3$-mer nucleotide frequency and in JASPAR~\citep{jaspar2022} transcription-factor motif occurrence,
respectively; both measure how closely the sequence statistics match natural DNA.
\item \emph{App-LL}, the approximate log-likelihood of the designs under the frozen pretrained
generator, a measure of how natural the samples remain to the base model (less negative is better).
\rev{It is the evidence lower bound of the masked diffusion model of~\citet{sahoo2024mdlm}, computed with
the released implementation of~\citet{wang2025drakes}: the bound is discretized over $128$ timesteps with a
fresh masking pattern drawn at each, and a single Monte Carlo replicate. Following~\citet{wang2025drakes}
and~\citet{dou2026gilc}, it is reported on DNA only.}
\end{itemize}

\smallskip\noindent\textbf{Small-molecule property targeting (QM9).}
\begin{itemize}[leftmargin=1.4em,itemsep=1pt,topsep=2pt]
\item \emph{MAE}, the mean absolute error between the property value requested as the condition and the
value realized by the generated molecule, evaluated with the property predictor of~\citet{tfgflow}; reported per property, lower is better.
\rev{Heat capacity $C_v$ is in $\mathrm{cal}\,\mathrm{mol}^{-1}\mathrm{K}^{-1}$, the isotropic
polarizability $\alpha$ in $\mathrm{Bohr}^{3}$, the dipole moment $\mu$ in Debye, and the HOMO--LUMO gap
and the two orbital energies in meV.}
\item \emph{Validity}, the fraction of generated graphs that are chemically valid. It is a fidelity
constraint rather than a reported column: every method is compared at an operating point with validity above
$75\%$, so that MAE gains cannot come from emitting invalid molecules. \rev{Table~\ref{tab:validity} gives
the realized values.}
\end{itemize}

\smallskip\noindent\textbf{Protein sequence design.}
\begin{itemize}[leftmargin=1.4em,itemsep=1pt,topsep=2pt]
\item \emph{Pred-$\mathrm{ddG}$}, the median, over generated sequences, of the predicted folding
stability from the Megascale-trained oracle~\citep{tsuboyama2023}; higher is more stable.
\item \emph{\%($\mathrm{ddG}{>}0$)}, the fraction of stabilizing designs (positive predicted
$\mathrm{ddG}$).
\item \emph{scRMSD}, self-consistency RMSD: \rev{each generated sequence is re-folded with
ESMFold~\citep{lin2023evolutionary} and its backbone compared to the wild-type backbone,
as in~\citet{wang2025drakes}}; lower is better.
\item \emph{\%(scRMSD${<}2$)}, the fraction of designs that re-fold well (scRMSD $<2$\,\AA).
\item \emph{Success Rate}, the joint structure-and-stability success rate: the fraction of designs that are
simultaneously stable ($\mathrm{ddG}{>}0$) and well-folded (scRMSD $<2$\,\AA).
\end{itemize}

\old{\paragraph{Diversity.} Three further metrics measure how concentrated a generated set is, and are used
in Appendix~\ref{app:h-div}.}
\old{\begin{itemize}[leftmargin=1.2em,itemsep=1pt,topsep=2pt]
\item \emph{mean-Ham}, the mean pairwise Hamming distance over all pairs of generated sequences, normalized
by sequence length. It is the diversity metric reported by the benchmark of~\citet{wang2025drakes}.
\item \emph{NN}, the mean nearest-neighbour distance: for each generated sequence, the normalized Hamming
distance to its closest other generated sequence, averaged over the set. Unlike mean-Ham it is sensitive to
duplicates and near-duplicates.
\item \emph{uniq}, the uniqueness rate: the fraction of generated sequences that are distinct.
\end{itemize}}

\section{Full-Result Ablation Study}
\label{app:fullres}
The main-text ablation (Table~\old{\ref{tab:ladder-full}}) reports one metric per stage. Here we give the complete
ablation of our method: every stage of the build-up for both guided proposals, over all metrics. Shaded rows are our two final settings.

Two details are referenced from \S\ref{sec:analysis}.
\emph{(i) The selector's share is largest where the proposal is weakest:} on protein the $\SMCAT$
increment exceeds the SMC increment on both paths, and by the wider margin on the non-differentiable one. \emph{(ii) Recovery.} With recovery defined as
$1-(\text{final gap})/(\text{proposal gap})$, where the proposal gap is the difference between the
variance-reduced differentiable proposal and the plain policy-gradient one, and the final gap is the same
difference at the $+\SMCAT$ rung of each ladder, selection closes $51.4\%$ of the gap on DNA and \old{$54.9\%$} on protein,
matching despite the non-differentiable proposal being $4.4\times$ weaker on protein.

\begin{table}[h]
\centering\footnotesize
\setlength{\tabcolsep}{3.2pt}
\caption{\textbf{Full DNA ablation of \VGAS} (mean${\pm}$std, $3$ seeds, \rev{$N{=}20$}). Every stage of the build-up for both guided
proposals, all metrics; shaded are our two finals.}
\label{tab:dna-full}
\resizebox{\textwidth}{!}{%
\begin{tabular}{@{}lrrrrr@{}}
\toprule
Method & Pred-med\,$\uparrow$ & ATAC (\%)\,$\uparrow$ & 3-mer\,$\uparrow$ & JASPAR\,$\uparrow$ & App-LL\,$\uparrow$\\
\cmidrule(r){1-1}\cmidrule(l){2-6}
GILC-DB proposal & $6.214_{\pm0.046}$ & $83.80_{\pm1.66}$ & $0.788_{\pm0.024}$ & $0.898_{\pm0.002}$ & $-279.1_{\pm0.5}$\\
\quad$+$ Gumbel--Rao & $6.345_{\pm0.008}$ & $87.33_{\pm1.37}$ & $0.797_{\pm0.011}$ & $0.887_{\pm0.002}$ & $-278.8_{\pm0.3}$\\
\quad$+$ SMC & $6.845_{\pm0.149}$ & $97.87_{\pm0.59}$ & $0.871_{\pm0.023}$ & $0.892_{\pm0.010}$ & $-279.3_{\pm0.9}$\\
\rowcolor{green!10}\quad$+$ $\SMCAT$ ($\VGASGR$) & $7.332_{\pm0.177}$ & $99.10_{\pm1.56}$ & $0.875_{\pm0.025}$ & $0.882_{\pm0.010}$ & $-276.3_{\pm0.7}$\\
\midrule
GILC-PG proposal & $4.859_{\pm0.074}$ & $47.20_{\pm0.70}$ & $0.373_{\pm0.020}$ & $0.865_{\pm0.004}$ & $-276.3_{\pm0.4}$\\
\quad$+$ SMC & $6.326_{\pm0.096}$ & $86.00_{\pm1.04}$ & $0.908_{\pm0.055}$ & $0.936_{\pm0.008}$ & $-275.8_{\pm1.0}$\\
\quad$+$ $\SMCAT$ & $6.610_{\pm0.037}$ & $89.50_{\pm3.84}$ & $0.926_{\pm0.032}$ & $0.900_{\pm0.009}$ & \rev{$-274.7_{\pm0.2}$}\\
\rowcolor{green!10}\quad$+$ RO ($\VGASRO$) & $6.956_{\pm0.188}$ & $93.93_{\pm2.44}$ & $0.943_{\pm0.006}$ & $0.900_{\pm0.009}$ & $-275.0_{\pm0.3}$\\
\bottomrule
\end{tabular}}
\end{table}

\begin{table}[h]
\centering\footnotesize
\setlength{\tabcolsep}{3.2pt}
\caption{\textbf{Full protein ablation of \VGAS} (mean${\pm}$std, $3$ seeds, \rev{$N{=}20$}). Every stage of the build-up for both guided
proposals, all metrics; shaded are our two finals. $-$: metric not run for that row.}
\label{tab:protein-full}
\resizebox{\textwidth}{!}{%
\begin{tabular}{@{}lrrrrr@{}}
\toprule
Method & Pred-$\mathrm{ddG}$\,$\uparrow$ & \%($\mathrm{ddG}{>}0$)\,$\uparrow$ & scRMSD\,$\downarrow$ & \rev{\%(scRMSD${<}2$\,\AA)}\,$\uparrow$ & Success Rate (\%)\,$\uparrow$\\
\cmidrule(r){1-1}\cmidrule(l){2-6}
GILC-DB proposal & $0.644_{\pm0.018}$ & $77.00_{\pm1.29}$ & $0.893_{\pm0.002}$ & $89.17_{\pm0.36}$ & $67.51_{\pm1.35}$\\
\quad$+$ Gumbel--Rao & $0.703_{\pm0.022}$ & $79.77_{\pm1.18}$ & $0.902_{\pm0.012}$ & $89.30_{\pm0.26}$ & $70.05_{\pm1.41}$\\
\quad$+$ SMC & $0.846_{\pm0.108}$ & $83.36_{\pm2.23}$ & $0.901_{\pm0.024}$ & $90.36_{\pm1.18}$ & $74.50_{\pm1.33}$\\
\rowcolor{green!10}\quad$+$ $\SMCAT$ ($\VGASGR$) & $1.285_{\pm0.163}$ & $90.23_{\pm4.30}$ & $0.881_{\pm0.035}$ & $91.64_{\pm0.87}$ & $82.07_{\pm3.19}$\\
\midrule
GILC-PG proposal & $0.160_{\pm0.022}$ & $54.34_{\pm0.38}$ & $0.853_{\pm0.012}$ & $90.80_{\pm0.54}$ & $49.65_{\pm0.63}$\\
\quad$+$ SMC & $0.405_{\pm0.175}$ & $58.77_{\pm3.68}$ & $0.869_{\pm0.010}$ & $90.13_{\pm0.75}$ & $53.10_{\pm3.08}$\\
\quad$+$ $\SMCAT$ & $1.040_{\pm0.056}$ & $72.14_{\pm2.13}$ & $0.857_{\pm0.012}$ & $92.17_{\pm2.45}$ & $67.43_{\pm2.13}$\\
\rowcolor{green!10}\quad$+$ RO ($\VGASRO$) & $1.116_{\pm0.088}$ & $76.04_{\pm1.04}$ & $0.852_{\pm0.016}$ & $92.30_{\pm1.74}$ & $69.73_{\pm0.34}$\\
\bottomrule
\end{tabular}}
\end{table}

\rev{\begin{table}[h]
\centering
\setlength{\tabcolsep}{3.2pt}
\caption{\textbf{Realized chemical validity on QM9} (\%, mean${\pm}$std, $3$ seeds), one column per
targeted property. The pretrained model generates unconditionally, so its value does not vary with the
target. Every method clears the $75\%$ operating point of~\citet{tfgflow}.}
\label{tab:validity}
\resizebox{\textwidth}{!}{%
\begin{tabular}{@{}lrrrrrr@{}}
\toprule
Method & $C_v$ & $\alpha$ & $\mu$ & gap & HOMO & LUMO\\
\cmidrule(r){1-1}\cmidrule(l){2-7}
Pretrained (MDLM)~\citep{sahoo2024mdlm} & $88.61_{\pm1.37}$ & $88.61_{\pm1.37}$ & $88.61_{\pm1.37}$ & $88.61_{\pm1.37}$ & $88.61_{\pm1.37}$ & $88.61_{\pm1.37}$\\
SMC~\citep{delmoral2004feynman} & $88.90_{\pm1.97}$ & $90.53_{\pm1.95}$ & $84.24_{\pm0.75}$ & $86.98_{\pm0.85}$ & $87.21_{\pm0.68}$ & $87.34_{\pm0.65}$\\
SVDD~\citep{li2024svdd} & $88.93_{\pm1.05}$ & $88.67_{\pm1.06}$ & $85.32_{\pm1.13}$ & $87.27_{\pm1.03}$ & $87.24_{\pm0.94}$ & $87.27_{\pm0.76}$\\
\midrule
GILC-DB~\citep{dou2026gilc} & $75.12_{\pm0.41}$ & $79.59_{\pm1.70}$ & $75.46_{\pm1.77}$ & $87.76_{\pm1.09}$ & $88.51_{\pm0.84}$ & $89.06_{\pm0.80}$\\
GILC-PG~\citep{dou2026gilc} & $75.00_{\pm0.98}$ & $78.87_{\pm1.96}$ & $76.81_{\pm1.96}$ & $86.59_{\pm0.65}$ & $87.92_{\pm0.73}$ & $85.12_{\pm0.66}$\\
SMC-DDM~\citep{pani2025smcddm} & $88.41_{\pm1.10}$ & $92.06_{\pm0.30}$ & $76.66_{\pm0.71}$ & $96.22_{\pm0.41}$ & $95.83_{\pm0.65}$ & $95.96_{\pm0.54}$\\
TreeG-G~\citep{guo2025treeg} & $88.83_{\pm1.63}$ & $92.12_{\pm0.25}$ & $75.84_{\pm1.32}$ & $95.73_{\pm0.62}$ & $95.35_{\pm0.65}$ & $95.71_{\pm0.26}$\\
\midrule
\rowcolor{green!10} $\VGASGR$ & $78.17_{\pm1.09}$ & $79.33_{\pm2.50}$ & $78.23_{\pm0.48}$ & $87.21_{\pm0.20}$ & $88.05_{\pm1.39}$ & $85.29_{\pm1.33}$\\
\rowcolor{green!10} $\VGASRO$ & $77.18_{\pm1.19}$ & $78.22_{\pm0.26}$ & $78.33_{\pm0.83}$ & $82.98_{\pm1.07}$ & $87.53_{\pm2.69}$ & $83.37_{\pm0.83}$\\
\bottomrule
\end{tabular}}
\end{table}}

\rev{Every method clears the floor, so the errors of Table~\ref{tab:qm9} are read at comparable validity.
$\VGAS$ is not the most valid of them: on the gap and the two orbital energies SMC-DDM and TreeG-G stay
above $95\%$ where $\VGAS$ sits between $83\%$ and $88\%$, and their error on those same properties is the
larger by a wide margin. The guidance that lowers the error does cost validity relative to leaving the
kernel alone, and the floor is what keeps that cost bounded.}

\section{Additional Analysis}
\label{app:more-analysis}

\rih{Synthetic verification.} We first reproduce the $64{\times}64$ grid of~\citet{pani2025smcddm},
on which the reward-tilted target is available in closed form, and add our components one at a time
(Figure~\ref{fig:toy}): guidance alone recovers the target region but leaves mass on off-target modes, the
Gumbel--Rao proposal sharpens it, and the adaptive-temperature selector brings the samples closest to the
target, whereas SMC matches the reward only by collapsing onto a few points.

\begin{figure}[t]
\centering
\includegraphics[width=\textwidth]{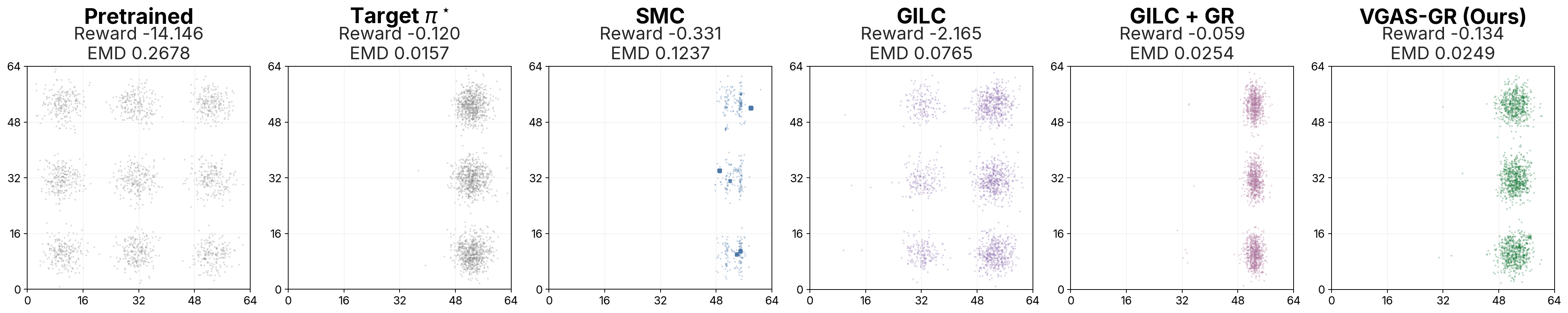}
\caption{\textbf{Steering on a synthetic grid.} Samples against the tilted target $p^\star$, with
reward and earth-mover distance to $p^\star$ above each panel.}
\label{fig:toy}
\end{figure}

This appendix reports the ablation-study details, the shape of the effective-temperature schedule and a
check of reward-scale invariance, the isolation of the two matched variance reducers, and the
reward--diversity frontier.

\subsection{Reward-Scale Invariance}
\label{app:h-checks}
The per-step spread $\sigma_t$ that $\SMCAT$ divides by \old{drifts by $3.7\times$ to $27.7\times$ along
the trajectory (Figure~\ref{fig:mechanism}a, Appendix~\ref{app:h-ess})}. One further check corroborates the
mechanism.
Because standardization divides by $\sigma_t$, $\SMCAT$ is invariant to
a reward rescaling $r\mapsto ar$ (the shift $b$ cancels in the softmax by Lemma~\ref{lem:baseline}).
Table~\ref{tab:scale} confirms the scale part empirically: a ninefold change in reward scale moves
Pred-med by $0.078$ under $\SMCAT$ versus $0.471$ under SMC, so $\alpha$ needs no retuning when the reward
is rescaled.

\begin{table}[h]
\centering\small
\caption{\textbf{Reward-scale invariance.} DNA Pred-med under a reward rescaling $r\mapsto ar$.
$\SMCAT$ barely moves ($|\Delta|{=}0.078$) whereas SMC shifts by $0.471$, so the base temperature
$\alpha$ needs no retuning when the reward is rescaled.}
\label{tab:scale}
\begin{tabular}{@{}lrr@{}}
\toprule
& $r{\times}\tfrac13$ & $r{\times}3$\\
\cmidrule(r){1-1}\cmidrule(l){2-3}
$\SMCAT$ & $7.268$ & $7.346$\\
SMC & $6.678$ & $7.149$\\
$|\Delta|$ & \multicolumn{2}{c}{$0.078$ vs $0.471$}\\
\bottomrule
\end{tabular}
\end{table}

\subsection{Isolating the Matched Variance Reducers}
\label{app:h-isolate}
\paragraph{RLOO placement (non-differentiable path).} \old{Table~\ref{tab:rloo} varies the placement of the
leave-one-out baseline over the two-by-two of proposal and selector. On the proposal, at the group-mean
selector, it gives a clear gain ($6.610\to6.956$, $+0.346$, $t{=}3.13$). On the selector it moves the result
by $+0.211$ in one row and $-0.296$ in the other, of comparable size and opposite sign, so the table
identifies no selector effect. This is the shape the leave-one-out case of
Lemma~\ref{lem:baseline} (restated in Appendix~\ref{app:proofs}) leads one to expect, the selector baseline
being inert up to $O(1/(N{-}1))$ in the log-weights, although the lemma's constant is in
log-weight rather than reward units and we do not convert it. We therefore place RLOO on the proposal
alone.}

\paragraph{Gumbel--Rao sample count (differentiable path).} \old{On the proposal the estimator improves
over the base row on both domains, by $0.161$ on DNA and $0.064$ on protein against seed spreads of
$0.008$--$0.046$ and $0.010$--$0.022$, and is saturated by $M{=}8$. Under $\SMCAT$ the five rows span less
than the spread of any one of them, so the sweep resolves no effect of $M$ on the final method
(Table~\ref{tab:grm}); we therefore set $M{=}8$ from the proposal and use it on both sequence benchmarks.}

\begin{table}[h]
\centering\small
\setlength{\tabcolsep}{6pt}
\caption{\textbf{Placement of the RLOO baseline} (DNA, PG path). DNA Pred-med with the baseline on the
proposal, the selector, both, or neither. The proposal placement carries a consistent gain; the selector
placement moves the result in opposite directions in the two rows.}
\label{tab:rloo}
\begin{tabular}{@{}lrr@{}}
\toprule
(DNA, PG) & sel $=$ group mean & sel $=$ leave-one-out\\
\cmidrule(r){1-1}\cmidrule(l){2-3}
prop $=$ group mean & $6.610$ & $6.821$\\
prop $=$ leave-one-out & $\mathbf{6.956}$ & $6.660$\\
\bottomrule
\end{tabular}
\end{table}

\begin{table}[h]
\centering\small
\setlength{\tabcolsep}{6pt}
\caption{\textbf{Gumbel--Rao sample count $M$} (backward-only; NFE fixed; mean${\pm}$std, $3$ seeds).
Pred-med (DNA) and $\mathrm{ddG}$ (protein), proposal alone and under $\SMCAT$. \old{On the proposal the
gain over the base estimator is resolved on both domains and saturates; under $\SMCAT$ the columns lie
within the seed spread, so $M$ is set from the proposal and carried over.}}
\label{tab:grm}
\begin{tabular}{@{}lrrrr@{}}
\toprule
$M$ & DNA prop & DNA $+\SMCAT$ & Protein prop & Protein $+\SMCAT$\\
\cmidrule(r){1-1}\cmidrule(l){2-5}
base & $6.214_{\pm0.046}$ & $7.227_{\pm0.108}$ & $0.644_{\pm0.018}$ & $1.243_{\pm0.070}$\\
$2$ & $6.328_{\pm0.016}$ & $7.265_{\pm0.062}$ & $0.664_{\pm0.017}$ & $1.235_{\pm0.050}$\\
$4$ & $6.335_{\pm0.016}$ & $7.132_{\pm0.165}$ & $0.684_{\pm0.020}$ & $1.264_{\pm0.128}$\\
$8$ & $6.345_{\pm0.008}$ & $7.332_{\pm0.177}$ & $0.703_{\pm0.022}$ & $1.285_{\pm0.163}$\\
$16$ & $\mathbf{6.375}_{\pm0.027}$ & $7.224_{\pm0.028}$ & $\mathbf{0.708}_{\pm0.010}$ & $1.278_{\pm0.048}$\\
\bottomrule
\end{tabular}
\end{table}

\subsection{Protein Structure versus Stability}
\label{app:h-protein}
On protein the adaptive selector lifts stability while holding structure. $\VGASGR$ attains the best
training-free $\mathrm{ddG}$ and surpasses the trained DRAKES on it (Table~\ref{tab:protein}), and also exceeds
DRAKES on the joint structure-and-stability success rate ($82.1\%$ vs.\ $78.6\%$). In the protein ladder
(Table~\ref{tab:protein-full}), the selector raises $\mathrm{ddG}$ from $0.85$ to $1.29$ while the joint
success rate rises ($74.5\%\!\to\!82.1\%$) and the well-folded rate holds ($90.4\%\!\to\!91.6\%$); on DNA the
same pattern holds (ATAC and $3$-mer both improve). \VGAS{} therefore attains the best training-free stability at no cost in structural success.

\subsection{Diversity Metric Blindness}
\label{app:h-div}
Table~\ref{tab:div} is the diversity evidence for \S\ref{sec:analysis}: the standard mean-Hamming metric
stays flat once selection is switched on, while the nearest-neighbour distance over the same runs falls by
more than two orders of magnitude, under either selector. Mean-Hamming therefore does not detect the
concentration that the frontier of Table~\ref{tab:frontier} then quantifies.

\begin{table}[h]
\centering\small
\caption{\textbf{Metric blindness.} DNA mean-Hamming distance is unchanged by selection, although the
nearest-neighbour (NN) distance over the same runs collapses.}
\label{tab:div}
\begin{tabular}{@{}lrr@{}}
\toprule
& mean-Ham\,$\uparrow$ & NN\,$\uparrow$\\
\cmidrule(r){1-1}\cmidrule(l){2-3}
Guided proposal (no selection) & $0.73$ & $0.62$\\
SMC & $0.72$ & $0.0044$\\
$\SMCAT$ & $0.72$ & $0.0043$\\
\bottomrule
\end{tabular}
\end{table}

\begin{table}[h]
\centering\small
\setlength{\tabcolsep}{7pt}
\caption{\textbf{Reward--diversity frontier at matched resampling frequency} (DNA, $3$ seeds; the
standard deviation is given for the reward column only). $\SMCAT$ against SMC on the same guided proposal;
lower $r$-freq resamples more often.
$\SMCAT$ attains the higher reward at every $r$-freq.}
\label{tab:frontier}
\begin{tabular}{@{}llrrr@{}}
\toprule
$r$-freq & selector & Pred-med\,$\uparrow$ & uniq\,$\uparrow$ & NN\,$\uparrow$\\
\cmidrule(r){1-2}\cmidrule(l){3-5}
$1$ & $\SMCAT$ & $\mathbf{7.332}_{\pm0.177}$ & $0.492$ & $0.0043$\\
 & SMC & $6.845_{\pm0.149}$ & $0.489$ & $0.0044$\\
\addlinespace[2pt]
$2$ & $\SMCAT$ & $\mathbf{6.742}_{\pm0.052}$ & $0.598$ & $0.0113$\\
 & SMC & $6.513_{\pm0.145}$ & $0.623$ & $0.0101$\\
\addlinespace[2pt]
$4$ & $\SMCAT$ & $\mathbf{6.582}_{\pm0.036}$ & $0.894$ & $0.0274$\\
 & SMC & $6.487_{\pm0.063}$ & $0.854$ & $0.0276$\\
\addlinespace[2pt]
$8$ & $\SMCAT$ & $\mathbf{6.544}_{\pm0.069}$ & $0.990$ & $\mathbf{0.0575}$\\
 & SMC & $6.349_{\pm0.058}$ & $\mathbf{0.999}$ & $0.0558$\\
\bottomrule
\end{tabular}
\end{table}

\subsection{\old{What a Fixed Temperature Costs}}
\label{app:h-ess}

\old{\paragraph{Protocol.} At every denoising step we record the whole array $\{\Delta_t^k\}_{k=1}^N$ as it
enters the resampling softmax, which needs no additional denoiser call. We run four selectors on each
benchmark, $\SMCAT$ and the fixed temperatures $\alpha\in\{0.1,0.5,2.0\}$, at three seeds. The three
benchmarks parameterize time differently, so we report everything against denoising progress, from $0$ at
the all-\textsc{mask} state to $1$ at $\bx_0$; without that normalization two of the domains drift in
opposite directions. Drift is the ratio of the mean of $\sigma_t$ over the first fifth of the trajectory to
its mean over the last fifth, which is insensitive to the noise in any single step. Curves are drawn with a
five-step moving average. Table~\ref{tab:drift} reports the drift and Figure~\ref{fig:ess} the per-step
effective sample size.}

\begin{table}[h]
\centering\small
\caption{\old{\textbf{Drift of the per-step spread.} Mean $\sigma_t$ over the first and the last fifth of
the trajectory, and their ratio. The main text quotes the fixed-$\alpha$ row, the selector being replaced.}}
\label{tab:drift}
\old{\begin{tabular}{@{}llrrr@{}}
\toprule
Domain & Selector & first fifth & last fifth & drift\\
\cmidrule(r){1-2}\cmidrule(l){3-5}
DNA & fixed $\alpha{=}0.5$ & $0.792$ & $0.213$ & $3.7\times$\\
 & $\SMCAT$ & $0.809$ & $0.187$ & $4.3\times$\\
\addlinespace[2pt]
Protein & fixed $\alpha{=}0.5$ & $0.156$ & $0.0173$ & $9.0\times$\\
 & $\SMCAT$ & $0.121$ & $0.0091$ & $13.3\times$\\
\addlinespace[2pt]
QM9 & fixed $\alpha{=}0.5$ & $12.58$ & $0.454$ & $27.7\times$\\
 & $\SMCAT$ & $12.58$ & $0.439$ & $28.6\times$\\
\bottomrule
\end{tabular}}
\end{table}

\begin{figure}[h]
\centering
\includegraphics[width=\textwidth]{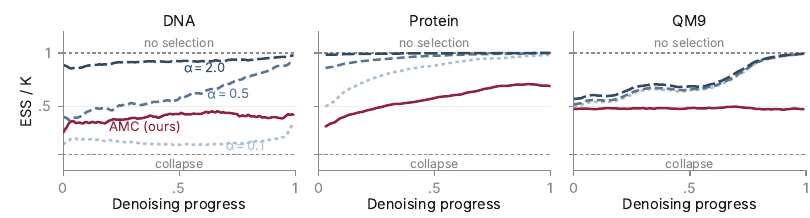}
\caption{\old{\textbf{Per-step effective sample size.} Four selectors on each benchmark. On QM9 the
fixed-$\alpha$ curves are recomputed from the $\SMCAT$ increments rather than run separately.}}
\label{fig:ess}
\end{figure}

\old{\paragraph{Reading the figure.} The effective sample size
$\mathrm{ESS}=(\sum_k\hat w^k)^2/\sum_k(\hat w^k)^2$ counts the particles that survive the resampling at
that step. At $N$ the weights are uniform and selection does nothing; at $1$ a single particle carries the
population. DNA shows the two failure modes at once. \rev{With $\alpha{=}0.1$ the effective sample size sits near
$0.16N$ for most of the trajectory, so the population is collapsed from the first steps, while with
$\alpha{=}2.0$ it stays above $0.92N$, so selection barely acts anywhere. Every fixed value rises toward the
end. On DNA the two larger ones finish above $0.93N$ and $\alpha{=}0.1$ finishes at $0.43N$, on protein all
three finish above $0.97N$ and on QM9 above $0.99N$, whereas $\SMCAT$ finishes at $0.37N$, $0.66N$ and
$0.47N$.} The temperature that is
correct at one end of the trajectory is therefore wrong at the other, and no single value is correct
throughout.}

\old{\paragraph{The shape of the increments moves too.} Standardization is not a fixed shape rescaled. The
skewness and excess kurtosis of $(\Delta_t^k-\mu_t)/\sigma_t$ drift along the trajectory as well, and on
protein and QM9 the excess kurtosis sits near $+4$, far from Gaussian. The selector holds its selection
pressure while dividing by $\sigma_t$ alone, and neither Lemma~\ref{lem:baseline} nor
Proposition~\ref{prop:at} assumes anything about the shape: the first uses only that the shift cancels, the
second only that $\sigma_t$ admits a deterministic limit.}

\old{\paragraph{What the late trajectory means.} A fixed $\alpha$ could be read as behaving correctly at the
end: once $\sigma_t$ is small there is little left to choose between, uniform weights would be the right
response, and dividing by $\sigma_t$ would amplify noise. The samples say otherwise. Were the late
selection acting on noise, $\SMCAT$ would not beat every fixed $\alpha$ in Figure~\ref{fig:alpha} nor lead
Table~\ref{tab:dna}.}

\old{\paragraph{On QM9.} The QM9 diagnostic runs at $\gamma{=}100$ rather than the $\gamma{=}10$ of
Table~\ref{tab:hparams}, and its fixed-$\alpha$ curves are recomputed from the $\SMCAT$ increments under the
fixed weighting rule rather than run separately, so those curves share one trajectory. The drift is a
property of $\{\Delta_t^k\}$ itself and the four selectors agree on it to within $27.6$--$28.6\times$, so
Table~\ref{tab:drift} is unaffected.}

\section{Related Work and Positioning}
\label{app:related}

\rev{\paragraph{Reading Table~\ref{tab:positioning}.} Prior work guides, selects, or combines the two.
\emph{Training-free} means the generative model is not retrained. The \emph{Gradient} column marks use of
the reward's gradient: through the denoiser for the guidance rows, through the sampling trajectory for
DRAKES; the gradient-free MCMC methods (SGDD, CSMC, IterRef) and the pure-search methods use none.
\emph{Adaptive sel.} marks a selection rule whose temperature is measured from the particle population at
run time; an ESS-triggered resampling schedule, which adapts when to resample rather than how sharply to
select, is not counted (Appendix~\ref{app:adaptive-smc}). Unlike the combined methods (TDS, TreeG,
SMC-DDM), our selection uses an adaptive rather than a pre-set rule, TreeG's hard top-$A$ beam or the SMC
family's pre-set temperature. Four entries carry a qualification. $^\dagger$DG does not retrain the
generator but fits a noise-conditional predictor, and we tabulate its gradient-based DG-TAG variant.
$^\ddagger$TFG-Flow uses the gradient on the continuous part only; its discrete part is guided by
importance sampling. $^\S$SMC-DDM is our shorthand for~\citet{pani2025smcddm}, after their earlier
preprint title. $^\P$TDS is formulated for continuous diffusion, and the discrete rows report the
adaptation obtained by post-hoc kernel tilting (Eq.~\eqref{eq:posthoc}).}

\begin{table}[h]
\centering
\caption{\textbf{Positioning of reward-steering methods.} \rev{Columns and markers are defined above.}}
\label{tab:positioning}
\resizebox{\textwidth}{!}{%
\begin{tabular}{@{}lcccccc@{}}
\toprule
Method & Training-free & Gradient & Search & Discrete & Combination & Adaptive sel.\\
\cmidrule(r){1-1}\cmidrule(l){2-7}
\multicolumn{7}{@{}l}{\emph{Fine-tuning / learning-based}}\\
DRAKES~\citep{wang2025drakes} & \xmark & \cmark & \xmark & \cmark & \xmark & \xmark \\
SDPO~\citep{han2026sdpo} & \xmark & \xmark & \xmark & \cmark & \xmark & \xmark \\
DDPP~\citep{rectorbrooks2025ddpp} & \xmark & \xmark & \xmark & \cmark & \xmark & \xmark \\
\addlinespace[2pt]
\multicolumn{7}{@{}l}{\emph{Gradient guidance}}\\
DPS~\citep{chung2023dps} & \cmark & \cmark & \xmark & \xmark & \xmark & \xmark \\
\old{DG-TAG}~\citep{nisonoff2024} & \old{\cmark$^\dagger$} & \cmark & \xmark & \cmark & \xmark & \xmark \\
TFG-Flow~\citep{tfgflow} & \cmark & \old{\cmark$^\ddagger$} & \xmark & \cmark & \xmark & \xmark \\
EntRGi~\citep{entrgi} & \cmark & \cmark & \xmark & \cmark & \xmark & \xmark \\
GILC~\citep{dou2026gilc} & \cmark & \cmark & \xmark & \cmark & \xmark & \xmark \\
\addlinespace[2pt]
\multicolumn{7}{@{}l}{\emph{Sampling-based guidance (gradient-free MCMC)}}\\
SGDD~\citep{chu2025sgdd} & \cmark & \xmark & \xmark & \cmark & \xmark & \xmark \\
CSMC~\citep{phunyaphibarn2026csmc} & \cmark & \xmark & \xmark & \cmark & \xmark & \xmark \\
IterRef~\citep{lee2025iterref} & \cmark & \xmark & \xmark & \cmark & \xmark & \xmark \\
\addlinespace[2pt]
\multicolumn{7}{@{}l}{\emph{Search / selection}}\\
Best-of-$N$~\citep{beirami2025bon} & \cmark & \xmark & \cmark & \cmark & \xmark & \xmark \\
SMC~\old{\citep{delmoral2004feynman,chopin2020introduction}} & \cmark & \xmark & \cmark & \cmark & \xmark & \xmark \\
SVDD~\citep{li2024svdd} & \cmark & \xmark & \cmark & \cmark & \xmark & \xmark \\
\addlinespace[2pt]
\multicolumn{7}{@{}l}{\emph{Guidance $+$ search}}\\
TDS~\citep{wu2023smc}\rev{$^\P$} & \cmark & \cmark & \cmark & \xmark & \cmark & \xmark \\
TreeG~\citep{guo2025treeg} & \cmark & \cmark & \cmark & \cmark & \cmark & \xmark \\
\old{SMC-DDM$^\S$}~\citep{pani2025smcddm} & \cmark & \cmark & \cmark & \cmark & \cmark & \xmark \\
\cmidrule(r){1-1}\cmidrule(l){2-7}
\textbf{Ours} & \cmark & \cmark & \cmark & \cmark & \cmark & \cmark \\
\bottomrule
\end{tabular}}
\end{table}

\paragraph{Gradient guidance.} The families differ in where the reward gradient is applied: at the
denoiser input~\citep{chung2023dps}, \rev{in the hidden states of the denoiser, with a discriminative head
fitted for the purpose~\citep{gruver2023protein}}, on the rate matrix~\citep{nisonoff2024,tfgflow}, or on
the clean-token logits~\citep{entrgi,dou2026gilc}. We adopt GILC's placement, which \old{bypasses the model Jacobian, stated
here as approximating it by the identity (\S\ref{app:guidance-math})}, and equip it with Gumbel--Rao. A
separate, gradient-free line steers by MCMC over clean or intermediate
states~\citep{chu2025sgdd,phunyaphibarn2026csmc,lee2025iterref}, admitting the reward through
accept/reject rather than a gradient. \old{Against GILC specifically, the difference is where the
standardization acts: GILC-PG already standardizes by a group standard deviation, but on the \emph{proposal},
across the Monte Carlo samples drawn at one state, whereas ours is on the \emph{selector}, across particles
and per step, and it is the drift of that spread that \S\ref{sec:search} acts on.}

\paragraph{Inference-time search.} The families differ in the selection rule.
SMC and the Twisted Diffusion Sampler~\citep{wu2023smc} resample a particle population by a value twist,
value-based decoding~\citep{li2024svdd} selects per step, TreeG~\citep{guo2025treeg} unifies candidate
proposal, value evaluation and selection under a tree search\rev{, and a general Feynman--Kac framework
covers the family across modalities, discrete text diffusion among
them~\citep{singhal2025general}}. TreeG is the closest prior work: its gradient-based variant TreeG-G establishes the
gradient-plus-search combination for discrete diffusion, and we run it from the official repository.
\old{What separates us from all of them is the selection rule itself.} The prior combinations select by a
rule \emph{set in advance}, hard top-$A$ for TreeG's beam search and a prescribed resampling temperature for
the SMC family, whereas we measure the temperature from the particles; our contribution is that rule, not a
new proposal.

\paragraph{Learning-based steering.} These differ from us on the training-free axis rather than on
\old{the gradient or the search axis}, amortizing the reward-tilted posterior into a trained
sampler~\citep{rectorbrooks2025ddpp} or retraining the generator~\citep{wang2025drakes,han2026sdpo}. We
compare against DRAKES, the standard fine-tuned reference on these benchmarks. The value-function view of
Diamond Maps~\citep{holderrieth2026diamond} connects these directions through the soft value.

\paragraph{Concurrent training-free SMC steerers.} These share our stance but not our selection rule,
steering through pre-set potentials~\citep{hasan2026dfkc}, trajectory-level
resampling~\citep{dang2025pgdlm}, or a corrected weighting of prior bootstrap
formulations~\citep{yadala2026nestedsmc}. \old{None of them measures the temperature from the particles,
which is the $\sigma_t$-standardized selection of \S\ref{sec:search} (Appendix~\ref{app:adaptive-smc}).}
Orthogonally, CDM~\citep{kim2026cdm} \emph{learns} an amortized twist, the same axis on which we position
against SMC-DDM's amortized proposal.

\subsection{Relation to Adaptive-Tempering SMC}
\label{app:adaptive-smc}
Temperature adaptation within SMC is well established: when a sampler anneals to a static target through a
ladder $\pi_t\propto\pi_0^{1-\lambda_t}\pi_{\mathrm{target}}^{\lambda_t}$, the exponents are commonly chosen
\emph{adaptively} to hold a target effective sample size~\citep{beskos2016convergence}. $\SMCAT$ differs on
three axes.
\emph{(i) Object.} We adapt the per-step \emph{resampling} temperature inside a reverse process we do not
design. \emph{(ii) Statistic.} Adaptive tempering solves iteratively for the exponent that hits a
user-chosen ESS level; $\SMCAT$ uses a closed-form, parameter-free statistic, the per-step spread $\sigma_t$
of the group advantage. \emph{(iii) Motivation.} Our diagnosis is specific to reward-guided diffusion:
$\sigma_t$ \old{drifts by $3.7\times$ to $27.7\times$} along the trajectory (Figure~\ref{fig:mechanism}a), a phenomenon absent
from static-target annealing. Table~\ref{tab:adaptive} makes the distinction concrete. \old{The
ESS-adaptive tolerance ladder of \citet{delmoral2012adaptive} is in the same spirit; we cite
\citet{beskos2016convergence} for the convergence of $\SMCAT$ (Proposition~\ref{prop:at}).}

\begin{table}[h]
\centering\footnotesize
\setlength{\tabcolsep}{4pt}
\renewcommand{\arraystretch}{1.15}
\caption{\rev{\textbf{$\SMCAT$ versus classical adaptive-tempering SMC.}}}
\label{tab:adaptive}
\begin{tabular}{@{}l p{0.31\textwidth} p{0.34\textwidth}@{}}
\toprule
& Adaptive-tempering SMC & $\SMCAT$ (ours)\\
\cmidrule(r){1-1}\cmidrule(l){2-3}
Adapted quantity & annealing exponent $\lambda_t$, on a static target & resampling temperature, inside the diffusion reverse process\\
Selection rule & solve $\lambda_t$ for a target \old{ESS}, iteratively & $\alpha_{\mathrm{eff}}=\alpha\sigma_t$ from the advantage spread, closed-form\\
\rev{Tuning} & \rev{a target ESS level, retuned per setting} & \rev{one scalar $\alpha$, the same value on all three domains}\\
Motivation & sample-health control & $\sigma_t$ drifts \old{$3.7$--$27.7\times$} along the trajectory\\
\bottomrule
\end{tabular}
\end{table}

\subsection{Anatomy of Guidance and Selection in Prior Samplers}
\label{app:anatomy}
Table~\ref{tab:anatomy} resolves what Table~\ref{tab:positioning} leaves open: which gradient and which
search. The closest baselines are often described as sharing one guidance and differing only in selection,
\old{and neither half of that description holds}. The guidance signals are distinct objects, and all three
baselines differentiate through the denoiser at one backward pass per step, whereas
$\partial\etab/\partial\bm z_t\approx\bm I$ \old{removes the denoiser from ours}. The selection rules differ too,
and TreeG departs furthest, keeping the top-$A$ candidates by value with no importance weight, a beam search
rather than a particle filter. The two substantive differences, which derivative is taken and where the tilt
is inserted, are made precise in \S\ref{app:guidance-math}.

\begin{table}[t]
\centering
\setlength{\tabcolsep}{3.5pt}
\renewcommand{\arraystretch}{1.2}
\scriptsize
\caption{\textbf{Guidance versus selection in the closest samplers.} Every gradient signal is written in the
common form $\nabla_{(\cdot)}\hat r$ ($\hat r=\frac1n\sum_i r(\hat\bx_0^{(i)})$); methods differ only in the
variable of differentiation ($\bx_t$, the noisy state $\bm z_t$, or the clean logits $\etab$) and in where
the tilt is inserted. Ours is the only Jacobian-free entry and the only one with an adaptive resampling
temperature. \rev{The backward row counts passes through the denoiser only; every differentiable-reward
method, ours included, still backpropagates through the reward oracle.}}
\label{tab:anatomy}
\begin{tabular}{@{}>{\raggedright\arraybackslash}p{2.0cm}>{\raggedright\arraybackslash}p{2.55cm}>{\raggedright\arraybackslash}p{2.75cm}>{\raggedright\arraybackslash}p{2.75cm}>{\raggedright\arraybackslash}p{2.85cm}@{}}
\toprule
\textbf{Axis} & \textbf{TDS} & \textbf{SMC-DDM} & \textbf{TreeG-G} & \textbf{Ours} \\
& \citep{wu2023smc} & \citep{pani2025smcddm} & \citep{guo2025treeg} & (\S\ref{sec:proposal}--\ref{sec:search}) \\
\midrule
State space & continuous (Gaussian) & discrete (masked) & discrete (flow / masked) & discrete (masked) \\
\cmidrule{1-5}
Guidance signal & $\nabla_{\bx_t}\hat r$ (continuous score at $\hat\bx_0$) & $\nabla_{\bm z_t}\hat r$ (ST-Gumbel relaxation) & $\nabla_{\bm z_t}\hat r$ (ST-Gumbel relaxation) & $\nabla_{\etab}\hat r$ (clean logits; $\partial\etab/\partial\bm z_t\!\approx\!\bm I$) \\
\cmidrule{1-5}
Model Jacobian & \cmark & \cmark & \cmark & \xmark \\
\cmidrule{1-5}
\rev{Denoiser backward passes / step} & 1 & 1 & 1 & 0 \\
\cmidrule{1-5}
Reveal factorized over positions & \cmark & \cmark & \cmark & \cmark \\
\cmidrule{1-5}
Selection & SMC (systematic), pre-set twist & SMC (multinomial, every step), pre-set $\lambda_t$ schedule & top-$A$ by value (their default; a value-proportional option exists), no temp. & SMC (multinomial), \textbf{adaptive temp.}\ ($\SMCAT$) \\
\cmidrule{1-5}
Proposal correction $\log(p_\theta/q)$ & \cmark\ (post-hoc $e^{g}$; drops $Z$) & \cmark & \xmark & \xmark\ by design (Lemma~\ref{lem:target}) \\
\bottomrule
\end{tabular}
\end{table}

\subsection{Explicit Comparison of the Guidance Signals: \old{TDS,} SMC-DDM, TreeG-G, and Ours}
\label{app:guidance-math}
Here we state the guidance signals as formulas, so that the two substantive differences, which derivative
is taken and where the tilt is inserted, can be verified line by line. Throughout, $\etab=\etab(\bm z_t)$ are
the clean logits emitted by the denoiser at $\bm z_t$, $p_\theta^\ell=\softmax(\etab^\ell)$ is the
per-position clean prediction, $a_t$ is the schedule's reveal probability from Eq.~\eqref{eq:revkernel}, and
we work at a single masked position $\ell$.

\paragraph{(i) Which derivative.}
All three gradient methods target the sensitivity of the multi-sample reward
$\hat r(\bm z_t)=\frac1n\sum_i r(\bx_0^{(i)})$, $\bx_0^{(i)}\sim p_\theta(\cdot\mid \bm z_t)$, routing the
categorical draw through the same Gumbel--Softmax relaxation, and differ only in where the chain rule is
stopped. The decomposition is GILC's~\citep[\S3.3]{dou2026gilc}, which names the last factor the model
Jacobian; SMC-DDM keeps it, tilting the kernel from the state~\citep[Eq.~8 and App.~C.2]{pani2025smcddm},
\begin{equation}
\label{eq:smcddm-chain}
\nabla_{\bm z_t}\hat r\;\approx\;\frac1n\sum_{i=1}^{n}
\underbrace{\frac{\partial r(\bx_0^{(i)})}{\partial \bx_0^{(i)}}}_{\text{reward backward}}\,
\underbrace{\frac{\partial \bx_0^{(i)}}{\partial \etab}}_{\text{ST-Gumbel}}\,
\underbrace{\frac{\partial \etab}{\partial \bm z_t}}_{\textbf{denoiser backward}},
\end{equation}
and TreeG-G \citep[Modules~3--4]{guo2025treeg} forms exactly the same three-factor product, then contracts it
against $(\bm z_t^{\backslash\ell}(k)-\bm z_t)$ to obtain a per-token score
$g^{(\ell)}_k=(\bm z_t^{\backslash\ell}(k)-\bm z_t)^\top\nabla_{\bm z_t}\hat r$. Ours keeps the first two
factors and replaces the third by the identity,
\begin{equation}
\label{eq:ours-chain}
\bm g^\ell=\frac1n\sum_{i=1}^{n}
\frac{\partial r(\bx_0^{(i)})}{\partial \bx_0^{(i)}}\frac{\partial \bx_0^{(i)}}{\partial \etab^\ell}
\;=\;\Big[\nabla_{\bm z_t}\hat r\,\Big]_{\ \partial\etab/\partial\bm z_t\;\to\;\bm I},
\end{equation}
which is Eq.~\eqref{eq:direction}. Ours is therefore Eq.~\eqref{eq:smcddm-chain} truncated after two
factors, the ``Jacobian-free'' row of Table~\ref{tab:anatomy}, and \old{two things follow from the
truncation.

First, the discarded factor is the only one requiring a backward pass through the denoiser. Second, the
surviving product lives in logit space, $\bm g^\ell\in\R^{|\V|}$, whereas Eq.~\eqref{eq:smcddm-chain} lives
in state space. Since the reveal is parameterized by $\etab$, a logit-space correction can be added to
$\etab$ and reused verbatim, while a state-space one must be applied externally to the transition kernel.

The truncation is an approximation, exact only for an identity denoiser, and it changes the quantity
estimated: it gives the direction in which the clean logits should move rather than the direction in which
the noisy state should move, and the former is the one the reveal can use as it stands. Its cost in reward
and fidelity is measured in \S\ref{sec:experiments}.}

\paragraph{(ii) Where the tilt is inserted, and its effect on the schedule.}
At a masked position the base kernel Eq.~\eqref{eq:revkernel} puts mass $a_t\,p_\theta^\ell(k)$ on each
token $k\in\V$ and mass $1-a_t$ on staying masked. Every method tilts this by $e^{\gamma g_k}$ and they
differ only in whether the $\textsc{mask}$ branch is tilted with it. Write
$Z^\ell=\E_{p_\theta^\ell}[e^{\gamma g}]$. The post-hoc form of TDS \citep{wu2023smc} and SMC-DDM
\citep[Eq.~8]{pani2025smcddm} multiplies the tilt onto the kernel after the reveal probability is fixed, so
it renormalizes over $\V\cup\{\textsc{mask}\}$:
\begin{equation}
\label{eq:posthoc}
q_{\mathrm{post}}^\ell(k)=\frac{a_t\,p_\theta^\ell(k)\,e^{\gamma g_k}}{1-a_t+a_t Z^\ell},
\qquad
q_{\mathrm{post}}^\ell(\textsc{mask})=\frac{1-a_t}{1-a_t+a_t Z^\ell}.
\end{equation}
Ours Eq.~\eqref{eq:direction} inserts the tilt inside the softmax that defines the clean prediction and
leaves the $\textsc{mask}$ branch untouched:
\begin{equation}
\label{eq:inside}
q^\ell(k)=a_t\,\softmax(\etab^\ell+\gamma\bm g^\ell)_k=\frac{a_t\,p_\theta^\ell(k)\,e^{\gamma g_k}}{Z^\ell},
\qquad
q^\ell(\textsc{mask})=1-a_t,
\end{equation}
where the second equality is the elementary identity
$\softmax(\etab^\ell+\gamma\bm g^\ell)_k=p_\theta^\ell(k)e^{\gamma g_k}/Z^\ell$. Comparing
Eq.~\eqref{eq:posthoc} and \eqref{eq:inside}: the numerators are identical. Both are the same exponential
tilt of the same clean prediction, so both belong to a single twist family. The entire
difference is the denominator, whether the normalizer $Z^\ell$ is retained on the revealed branch alone, or
distributed across the revealed and masked branches together. Proposition~\ref{prop:leak}, stated and
proved in Appendix~\ref{app:proofs}, turns that difference into a statement about the reveal probability
$\tilde a_t$ each form actually induces.

\old{Part~(a) of that proposition makes $\tilde a_t^{\ell,\,\mathrm{post}}$ strictly increasing in
$Z^\ell$, crossing $a_t$ at $Z^\ell=1$, so the direction of the leakage is decided by whether $Z^\ell$
exceeds one. Part~(b), that $Z^\ell\ge1$ under the gauge $\E_{p_\theta^\ell}[g]=0$, settles that question:
the leakage is systematic and one-signed, and the post-hoc form reveals earlier than the schedule
intends. This matters because the masked-diffusion ELBO and
the $T\to\infty$ consistency of the reverse process are stated for $a_t$, so replacing it by
$\tilde a_t>a_t$ commits tokens earlier than intended, exactly when guidance is strongest and the clean
prediction least reliable.

Part~(c), the invariance of Eq.~\eqref{eq:inside} to $\bm g^\ell\mapsto\bm g^\ell+c\bm1$, is the sharpest,
because that constant is arbitrary: the softmax is unchanged by $\etab\mapsto\etab+c\bm1$, so a correction
written in logit coordinates carries no canonical choice of it, and \rev{the gauge shift
$\hat{\bm g}=\bm g-c\bm 1$ fixes it only by fiat}. Under the post-hoc form that convention
moves the unmasking schedule, all the way to ``reveal everything'' or ``reveal nothing''; under
Eq.~\eqref{eq:inside} it is invisible.}

\noindent \old{Both parts describe a discrepancy that can be repaired downstream rather than avoided.}
TDS and SMC-DDM carry the $\log(p_\theta/q)$ correction that repairs \old{the leakage}, but the repair is
only as good as the particle count and is unavailable under rules that do not reweight, including TreeG's
beam search and vanilla SVDD. Eq.~\eqref{eq:inside} removes the discrepancy at the proposal level instead.

\paragraph{(iii) Is putting the correction in the logits ours alone?}
Not as an idea, and we do not claim it: the logit-space placement is GILC's \citep{dou2026gilc}, restated
in Eq.~\eqref{eq:direction}. The narrower, checkable claim is that among the training-free
gradient-plus-search samplers of Table~\ref{tab:anatomy}, none places the correction there. TDS tilts a
Gaussian transition; SMC-DDM tilts $p_\theta(x_{t-1}\mid x_t)$ from outside
\citep[Eq.~8]{pani2025smcddm}; TreeG-G adds its per-token score to candidate generation and selects by beam
search \citep[Modules~3--4]{guo2025treeg}. Proposition~\ref{prop:leak} is the consequence of that placement,
not of the tilt, which is shared. Our own contribution sits elsewhere: $\bm g^\ell$ is a per-position object,
so Eq.~\eqref{eq:inside} remains a product over positions, and we act through selection instead
(\S\ref{sec:search}).

\section{Sampler and Selector Algorithms}
\label{app:selectors}
This appendix gives the single-trajectory \textsc{Guided sampler} loop referenced in \S\ref{sec:proposal}
(Algorithm~\ref{alg:gilc}) and our selector $\SMCAT$ (Algorithm~\ref{alg:smc}); both call the guided reveal of
\S\ref{sec:proposal} in place of the base model's unmasking, with the correction $\bm g$ supplied by either
reduced-variance estimator (Algorithms~\ref{alg:corr-rao}--\ref{alg:corr-pgro}). For $\SMCAT$ the
single line that departs from vanilla SMC (the adaptive-temperature weight) is flagged; resampling is
multinomial.

\begin{algorithm}[H]
\caption{\textsc{Guided sampler} (single trajectory)}
\label{alg:gilc}
\KwIn{denoiser $p_\theta$; reward $r$; guidance scale $\gamma$; MC size $n$; steps $T$}
$\bm z_T = $ all-\textsc{mask}\;
\For{$t=T,\dots,1$}{
  $\etab = \mathrm{denoiser}(\bm z_t)$,\ \ $\bm g = \textsc{Guidance}(\etab, r, n)$ (Alg.~\ref{alg:corr-rao}/\ref{alg:corr-pgro})\;
  reveal masked $\ell$ from $(1{-}a_t)\delta_{\textsc{mask}} + a_t\softmax(\etab^\ell{+}\gamma\bm g^\ell)$ by Eq.~\eqref{eq:direction}\;
  keep revealed frozen;\ \ $\bm z_t = \bm z_s$\;
}
\KwOut{$\bm z_0$}
\end{algorithm}

\begin{algorithm}[H]
\caption{$\SMCAT$}
\label{alg:smc}
\KwIn{particles $N$; reward $r$; guidance scale $\gamma$; MC size $n$; selection temperature $\alpha$; variance floor $\epsilon$; steps $T$}
\rev{$\bm z_T^k\gets$ all-\textsc{mask};\ $\etab^k\gets\mathrm{denoiser}(\bm z_T^k)$;\ $V^k\gets r(\argmax\etab^k)$, $k=1,\dots,N$}\;
\For{$t=T,\dots,1$}{
  \For{$k=1,\dots,N$}{
    \rev{$\bm g\gets\textsc{Guidance-GR}(\etab^k,\bm z_t^k,r,n)$}\tcp*{\rev{GR proposal; $\etab^k$ cached}}
    reveal $\bm z_s^k$ from $(1{-}a_t)\delta_{\textsc{mask}}+a_t\softmax(\rev{\etab^k}+\gamma\bm g)$ by Eq.~\eqref{eq:direction}\;
    \rev{$\etab_s^k\gets\mathrm{denoiser}(\bm z_s^k)$;\ $V_s^k\gets r(\argmax\etab_s^k)$;\ $\Delta^k\gets V_s^k-V^k$}\tcp*{Feynman--Kac increment}
  }
  $\mu\gets\mathrm{mean}(\{\Delta^k\})$,\ $\sigma\gets\max\!\big(\mathrm{std}(\{\Delta^k\}),\epsilon\big)$;\ $\hat w^k\propto\exp\!\big((\Delta^k-\mu)/(\alpha\sigma)\big)$;\ resample \rev{$\{(\bm z_s^k,\etab_s^k,V_s^k)\}\sim\hat w$}\tcp*{adaptive temperature; vanilla SMC uses a prescribed $\alpha$}
  $\bm z_t^k\gets\bm z_s^k$;\ \rev{$\etab^k\gets\etab_s^k$;\ $V^k\gets V_s^k$}\tcp*{\rev{one denoiser call per particle per step}}
}
\KwOut{$\bm z_0^{1},\dots,\bm z_0^{N}$}
\end{algorithm}

\subsection{Closed Form of the Importance Correction}
\label{app:correction}
Lemma~\ref{lem:target} states that adding $\log(p_\theta/q_\gamma)$ to the log-weights restores
$p^\star$ for every $\gamma$. Both kernels leave the masked/unmasked split untouched, the guided reveal
renormalizes inside the vocabulary and multiplies by the same schedule factor, so the ratio is
supported on the positions revealed at that step. For the product-form proposal, writing
\rev{$\hat{\bm g}=\bm g-c\bm 1$ for an arbitrary gauge constant $c$} and $z_s^\ell=k$,
\begin{equation}
\label{eq:iscorr}
\log\frac{p_\theta(\bm z_s\mid\bm z_t)}{q_\gamma(\bm z_s\mid\bm z_t)}
=\sum_{\ell\,\in\,\mathrm{revealed}}\Big(-\gamma\hat g^\ell_k
+\mathrm{lse}\big(\etab^\ell+\gamma\hat{\bm g}^\ell\big)-\mathrm{lse}\big(\etab^\ell\big)\Big).
\end{equation}
Summing Eq.~\eqref{eq:iscorr} over the trajectory and negating gives the trajectory log-twist
$\bar\Phi_\gamma$ of Appendix~\ref{app:proofs}, whose terminal counterpart is $\Phi_\gamma$ of
Eq.~\eqref{eq:phi}. Every quantity in Eq.~\eqref{eq:iscorr} is already materialized by the guided reveal, so
the correction costs no additional forward or backward pass. Under $\SMCAT$ it must be added to the log-weights after the
standardization of $\{\Delta^k\}$, since it carries an absolute scale that the standardization would
otherwise destroy.

\section{Base Proposal Estimators}
\label{app:pg}
\S\ref{sec:proposal} presents the two reduced-variance proposals we run, \textsc{Guidance-GR} and
\textsc{Guidance-RO}. This appendix states the two base estimators they reduce: the differentiable
\textsc{Guidance-DB} and the non-differentiable \textsc{Guidance-PG}. Both share the mean
$\nabla_{\etab}\E[r]$ and are interchangeable in Algorithm~\ref{alg:gilc}.

\paragraph{Differentiable base (\textsc{Guidance-DB}).}
\textsc{Guidance-DB} (Alg.~\ref{alg:corr-db}), the estimator GILC uses in practice, differentiates the
reward through the ST-GS path Eq.~\eqref{eq:st}: draw $n$ Gumbel--Softmax samples, evaluate the reward on the
hard one-hot, and average the straight-through gradients. Evaluating its ST-GS Jacobian at a single Gumbel
draw per sample makes it high-variance; the Gumbel--Rao reduction of \S\ref{sec:proposal} replaces that Jacobian
by an $M$-sample conditional average without an extra pass. \rev{Substituting the softmax Jacobian for the
derivative of the argmax leaves the ST-GS path biased, and conditioning on $i^\star$ preserves that mean:
Gumbel--Rao removes the variance of the draw, not the bias of the relaxation. Two implementation points
carry into every estimator here. The conditional perturbations $\bm\zeta^{(m)}$ of
Eq.~\eqref{eq:condgumbel} depend on $\etab$ through $T$ and $v_k$, so they are detached from the
autograd graph before the reward is differentiated; leaving them attached backpropagates through the
sampler itself and corrupts the Jacobian. And the candidate $\hat\bx$ is assembled from the current state:
positions already revealed in $\bm z_t$ are held at their committed tokens and only the masked positions
receive a sample, so the absorbing structure of Eq.~\eqref{eq:revkernel} is preserved.} The base estimator is the ``base'' row
of Table~\ref{tab:grm}. The
conditional $p(\bm\zeta\mid\argmax{=}i^\star)$ used there is drawn by the truncated-Gumbel
reparameterization of Eq.~\eqref{eq:condgumbel}, where $T$ is the maximum $v_{i^\star}$.

\begin{algorithm}[t]
\caption{\textsc{Guidance-DB} (differentiable base estimator)}
\label{alg:corr-db}
\KwIn{clean logits $\etab$; \rev{state $\bm z_t$}; reward $r$; sample size $n$; Gumbel temp $\tau$}
\For{$i=1,\dots,n$}{
  draw $\bm\zeta^{(i)}\!\sim\!\mathrm{Gumbel}(0,1)$ i.i.d.\;
  $\hat\bx_{\mathrm{soft}}^{(i)}\gets\softmax\!\big((\etab+\bm\zeta^{(i)})/\tau\big)$\;
  $\hat\bx^{(i)}\gets$ ST estimator of $\hat\bx_{\mathrm{soft}}^{(i)}$ by Eq.~\eqref{eq:st}\;
  $R_i\gets r(\hat\bx^{(i)})$\;
}
$\bm g \gets \tfrac1n\sum_{i=1}^n \partial R_i/\partial\etab$\tcp*{$=\nabla_{\etab}\E[r]$}
\KwOut{logit correction $\bm g$}
\end{algorithm}

\paragraph{Non-differentiable base (\textsc{Guidance-PG}).}
When the reward is non-differentiable (a black box, or expensive to backprop through), the correction
can instead be formed by the score-function / policy-gradient (Reinforce) estimator. Sampling
hard tokens $\bx^{(i)}$ from the proxy clean distribution $\bm p=\softmax(\etab)$ and using
$\nabla_{\etab}\log p(\bx^{(i)})=\mathrm{onehot}(\bx^{(i)})-\bm p$, the reward-weighted score is an
unbiased estimate of $\nabla_{\etab}\E[r]$. Following the group-relative formulation of
GILC~\citep{dou2026gilc}, itself after group-relative policy optimization~\citep{shao2024grpo}, we
standardize the rewards into a \emph{group-relative advantage} before weighting (\textsc{Guidance-PG},
Alg.~\ref{alg:corr-pg}). This needs no relaxation, no Gumbel temperature and no model Jacobian, at the
price of higher variance than the differentiable path. The leave-one-out baseline
$A_i=R_i-\tfrac{1}{n-1}\sum_{j\neq i}R_j$ is the variance-reduction counterpart, for this proposal, of
Gumbel--Rao for the differentiable one: Table~\ref{tab:rloo} shows it outperforms the full-group baseline
on the proposal, while on the selector it moves the result in opposite directions in the two rows, as the
inertness of \old{Remark~\ref{rem:loo}} would lead one to expect, which is why our non-differentiable final
method places it on the proposal and standardization only on the selector.

\begin{algorithm}[t]
\caption{\textsc{Guidance-PG}}
\label{alg:corr-pg}
\KwIn{clean logits $\etab$; \rev{state $\bm z_t$}; reward $r$; sample size $n$}
$\bm p \gets \softmax(\etab)$\tcp*{proxy clean-token dist.\ $\hat\bx_\theta$}
\For{$i=1,\dots,n$}{
  draw $\bx^{(i)}\sim\bm p$;\ \ $\hat\bx^{(i)}\gets\mathrm{onehot}(\bx^{(i)})$\tcp*{hard; no relaxation}
  $R_i\gets r(\hat\bx^{(i)})$\;
}
$A_i \gets \big(R_i-\mean(\{R_j\})\big)\big/\std(\{R_j\})$\tcp*{group-relative advantage}
$\bm g \gets \tfrac1n\sum_{i=1}^n A_i\,\dfrac{\partial\langle\log\bm p,\,\hat\bx^{(i)}\rangle}{\partial\etab}$\tcp*{$=\tfrac1n\sum_i A_i(\hat\bx^{(i)}-\bm p)$; same as GILC}
\KwOut{logit correction $\bm g$}
\end{algorithm}

\end{document}